\pdfoutput=1
\documentclass[12pt]{article}
\usepackage{amsmath}
\usepackage{graphicx,psfrag,epsf}
\usepackage{enumerate}

\newcommand{\blind}{0}

\usepackage{amsthm}
\usepackage{graphicx}
\usepackage{amsmath}
\usepackage{amssymb}
\usepackage{floatflt}
\usepackage{subfigure}
\usepackage{algorithm}
\usepackage{algorithmic}
\usepackage{color}
\usepackage{hyperref}
\usepackage{pifont}
\usepackage{amssymb}
\usepackage{soulutf8}

\DeclareGraphicsExtensions{.jpg,.pdf,.mps,.png}

\def\argmin{\mathop{\rm argmin}}

\def\RR{\mathbb R}

\def\N{{\cal N}}

\def\N{{\cal N}}
\newtheorem{proposition}{Proposition}
\newtheorem{theorem}{Theorem}
\newtheorem{corollary}{Corollary}
\newtheorem{lemma}{Lemma}
\newcommand{\bmu}{{\boldsymbol{\mu}}}
\newcommand{\beps}{{\boldsymbol{\epsilon}}}

\newcommand{\bx}{\mathbf{x}}

\newcommand\Tstrut{\rule{0pt}{2.5ex}}         

\begin{document}



\def\spacingset#1{\renewcommand{\baselinestretch}%
{#1}\small\normalsize} \spacingset{1}


\if0\blind
{
  \title{\b Unsupervised Learning of GMM with a Uniform Background Component}
  \author{Sida Liu\thanks{
    The authors gratefully acknowledge DARPA ARO W911NG-16-1-0579 for support}\hspace{.2cm}\\
    Department of Statistics, Florida State University\\
    and \\
    Adrian Barbu \\
    Department of Statistics, Florida State University}
  \maketitle
} \fi

\if1\blind
{
  \bigskip
  \bigskip
  \bigskip
  \begin{center}
    {\LARGE\bf Title}
\end{center}
  \medskip
} \fi

\bigskip
\begin{abstract}
Gaussian Mixture Models are one of the most studied and mature models in unsupervised learning. However, outliers are often present in the data and could influence the cluster estimation. In this paper, we study a new model that assumes that data comes from a mixture of a number of Gaussians as well as a uniform ``background'' component assumed to contain outliers and other non-interesting observations.  We develop a novel method based on robust loss minimization that performs well in clustering such GMM with a uniform background. We give theoretical guarantees  for our clustering algorithm to obtain best clustering results with high probability. Besides, we show that the result of our algorithm does not depend on initialization or local optima, and the parameter tuning is an easy task. By numeric simulations, we demonstrate that our algorithm enjoys high accuracy and achieves the best clustering results given a large enough sample size.  Finally, experimental comparisons  with typical clustering methods on  real datasets witness the  potential  of our algorithm in real applications.
\end{abstract}





\section{Introduction}

Over several past decades, mixture models have become the center of many clustering problems. 
Among various mixture models, Gaussian Mixture Models (GMM) are the  most well-known and studied. 
As a fundamental model in describing numerous natural and artificial phenomena, GMMs are being studied with different types of methods over the past few decades. In these applications, the data samples are always assumed to originate from various sources where each source can approximately fit a Gaussian model.

However, in many real clustering problems such as object recognition, observations from desired categories are always a minority while a majority of the observations are highly variable and cannot be clustered in any particular way. 

\begin{figure}[t]
\centering

\includegraphics[height=0.25\textwidth]{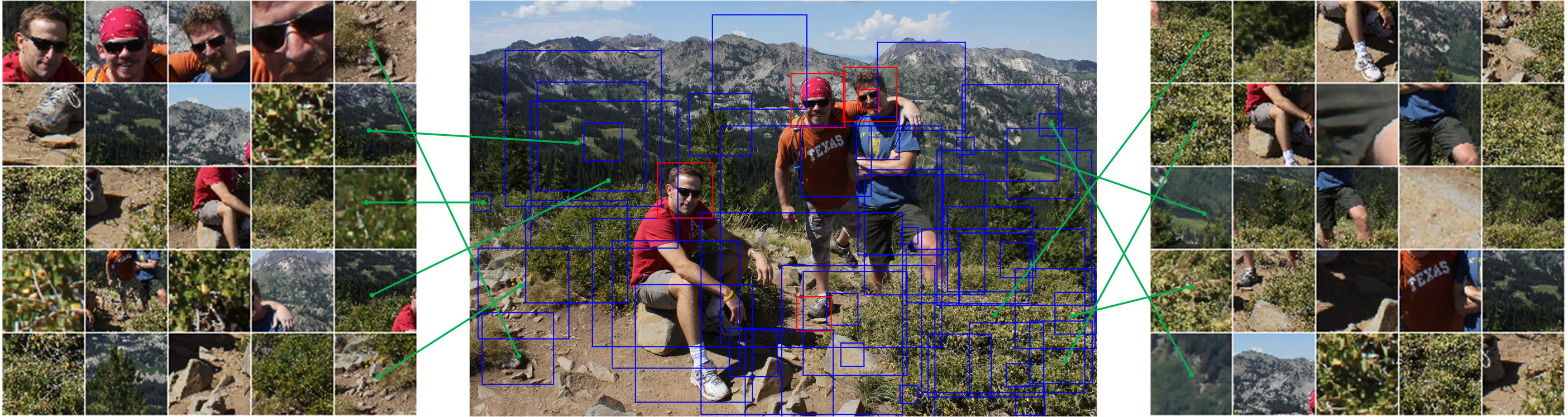}

\caption{Random patches cropped for an image. The red boxes contain some meaningful objects while the blue boxes are just irrelevant noise (background).}\label{fig:motivation}
\vspace{-3mm}
\end{figure} Furthermore, when designing algorithms for GMM, the prior knowledge or a reasonable estimate of the number of clusters is of great significance. However, in real image problems the total number of object clusters is very large, on the order of thousands and we are often interested in only a few of these clusters. This issue can be addressed by semi-supervised learning since assigning a label for a single example from a cluster makes it clear that the cluster of importance to us.

\begin{figure}[htb]
\vspace{-1mm}
\centering
\includegraphics[height=0.27\textwidth]{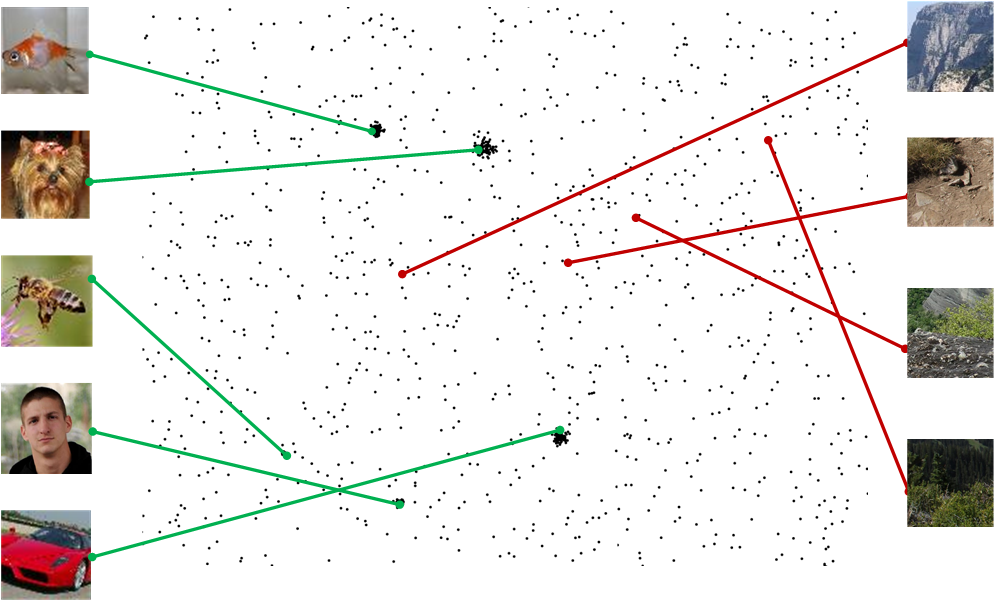}
\caption{Data from a mixture of five Gaussians plus a uniform background.}\label{fig:example2d}
\end{figure} 

In this paper we are interested in GMM corrupted by outliers. 
The motivation for this type of data comes from common problem of object detection of images. 
In this case, a sliding window extracts thousands or even millions of patches of different sizes from grid locations in the image. 
The patches are resized to the same size, for example $64\times 64$ pixels and a classifier is used to predict whether each patch contains an object of interest centered in it or not.

 As one could imagine, in the thousands of patches, only a small number correspond to objects of interest and the rest are just meaningless noise (background). 
Figure \ref {fig:motivation} illustrates this case. The red boxes are some clear objects like faces and shoes while the blue boxes are the random crops which don't form any meaningful object of interest. 

In this paper we will study one of the simplest models that exhibits this phenomenon, which we call a Gaussian Mixture Model with Uniform Background (GMMUB). The model consists of a Gaussian Mixture Model to represent the objects of interest (which we call {\em positives}), together with an extra mixture component that contains uniformly distributed observations in a large domain of the entire space for the background patches (called {\em negatives}). 
Usually, the negatives dominate the data with a large mixture proportion, as illustrated in Figure \ref{fig:example2d}. 

Figure \ref{fig:example2d} illustrated the distribution of the GMMUB in the observation space. There are five clusters  (e.g. human faces, dogs, fishes, bees and cars) and lots of background patches that don't cluster together.

Our model is somehow similar to \cite{browne2011model}, which is a mixture of GMM and uniform distributions, so it contains our GMMUB as a special case. 
However, our algorithm is different from the EM algorithm used in \cite{browne2011model}, as  it does not depend on an initialization and it enjoys better clustering performance results on different datasets .  

The contributions of this paper are the following:
\begin{itemize}
\item We introduce the Gaussian mixture model with uniform background (GMMUB) as a simple framework for handling the type of data existent in object detection problems in computer vision.
\item We introduce a novel clustering algorithm for this kind of data, which finds the positive clusters as local minima of a robust loss function, this way extracting them out of the uniform background. 
   This robust loss function is non-zero only within a certain distance from the cluster centers and zero everywhere else. 
 In this respect the robust loss function is similar to the negative of a kernel density function, where the kernel is a truncated quadratic.  A feature of the algorithm is that it does not rely on a well-chosen initialization. 
  Moreover, the process of loss minimization in our algorithm is quite simple and computationally efficient and avoids the problem of being trapped in local optima specific to gradient descent and EM based methods.
 \item We prove that our algorithm is able to correctly cluster all the positives with high probability under certain assumptions of separation and concentration, even when the majority of data comes from the uniform background.

\item We conduct experiments on simulated data and real data. 
The simulation results indicate that when the assumptions are met, our algorithm performs better than other clustering methods such as K-means, Spectral Clustering, Tensor Decomposition, etc. 
Furthermore,  experiments on real data indicate that our algorithm remains applicable and powerful on  real data applications when most of the assumptions are met.
\end{itemize}

\section{Related Work}
Research of GMM has advanced swiftly and vigorously with the advent of the information era.  
In 1977, the Expectation Maximization (EM) algorithm was formalized by \cite{dempster1977maximum}, marking the beginning of modern clustering algorithms regarding GMM. 
In 2000, Dasgupta and Schulman \cite{dasgupta2000two} built a framework for a two-step EM variant that has theoretical convergence guarantees. 
Since then, multiple algorithms have been proposed to improve the theoretical bounds and loosen the separation conditions. 
The spectral projection method from \cite{vempala2004spectral} improved the theoretical guarantees while the PAC learning method of \cite{feldman2006pac} makes no assumptions about the separation between the means of the Gaussians. 
Later on, \cite{kannan2008spectral} found another spectral method that can be applicable not only to GMM but also to mixtures of log-concave distributions. A polynomial-time algorithm for the case of two Gaussians was proposed in \cite{kalai2010efficiently} with provably minimal assumptions on the Gaussians and with polynomial data size requirements.

Tensor Decomposition (\cite {hsu2013learning}) is a spectral decomposition method based on low-order observable moments that has theoretical guarantees without additional separation conditions.
However, as shown experimentally in  this paper, this method is very sensitive to outliers, thus it does not work well on our uniform background setting. Furthermore, the method is not computationally efficient and has prohibitive computation cost for high dimensional data.

The above algorithms are designed for GMM or other distribution family models, and are known as distributions models. 
Aside from them,  there are other clustering methods that do not require specific distribution assumptions for the data. 
They actually measure the similarity between observations in different ways and perform clustering based on that similarity measure.
However, there is no universally accepted definition of the term "Clustering". 
From different points of view, different clustering algorithms can be divided into different categories. 
K-means clustering (\cite{hartigan1979algorithm,kanungo2002efficient,lloyd1982least}) and its variants are probably one kind of the most popular and widely-used clustering algorithms. 
Hierarchical Clustering (\cite{day1984efficient,johnson1967hierarchical}) builds a hierarchy of clusters with different distance metrics. 
They are typical distance-based clustering algorithms. 

DBSCAN (\cite{ester1996density}) is a representative of density models. 
Given a set of points, it clusters the points that have many nearby neighbors. 
It also marks the points that are not reachable from any other point as outliers.
Based on its properties, DBSCAN can obtain clusters with arbitrary shapes. 
Major variants for DBSCAN are l-DBSCAN(\cite{viswanath2006dbscan}), ST-DBSCAN (\cite{birant2007st}), C-DBSCAN(\cite{ruiz2007c}) and P-DBSCAN (\cite{kisilevich2010p}).

Another kind of density based clustering algorithm is KNN kernel density clustering. 
The most influential work is KNNCLUST(\cite {tran2006knn}). 
Some related work and variants are KNN-DBSCAN (\cite {sharma2017knn}), cKNN+ (\cite{gallego2018clustering}) and RECOME (\cite{geng2018recome}). 
KNNCLUST is based on a combination of KNN and Kernel Density Estimation. 
It is  similar to Partitional Clustering (\cite{tran2005clustering})with some modifications to make it more robust  and less sensitive to initialization. 
However,  KNNCLUST is not a perfect choice to estimate the true density for GMMUB, while our method does much better experimentally and has theoretical guarantees.

Spectral Clustering (\cite{shi2000normalized,ng2002spectral})  uses the eigenvectors of a similarity matrix for dimension reduction of the data before clustering. 

T-SNE (\cite {maaten2008visualizing}) is a typical data visualization and dimension reduction technique. It can be used for clustering, employing K-means or other clustering algorithms on the low-dimensional feature space obtained by T-SNE.

Though these methods may also be applied to GMM and other mixture models, as shown in our paper, their performance may be no better than the clustering algorithms that are specialized in clustering GMM data.

 The study of the convergence of  most GMM clustering algorithms is related to the initial value of the GMM parameters. 
 Many methods including EM can get stuck in local optima when the initialization is not close enough to the true values. This is why a good initialization is of great significance for many clustering algorithms.
There are some recent methods that try to overcome this drawback and provide good initialization methods.
K-means$++$ (\cite{arthur2007k}) chooses the initial centers in a fast and simple way and achieves certain theoretical guarantees that k-means cannot.

In \cite{karami2014choosing} is presented a hybrid clustering method based on DBSCAN that automatically specifies the appropriate parameter values.

\begin{table}[t]
\centering
\caption{Comparison between different clustering algorithms}\label{tab:alg}
\scalebox{0.73}{
\begin{tabular}{l|ccccc}
\hline
 &Convergence                    &Running                  &Theoretical                                  & Compatible w/ &Assumptions and   \Tstrut\\
Algorithm  &Rate &Time &Guarantee &Uniform Bgd. &Conditions\\
\hline
K-means with cluster shifting(\cite{pakhira2014linear})   &$-$ &$\mathcal{O}(n) $  & $\times$  & $\times$ &$-$                   \Tstrut\\

\hline
KNNCLUST(\cite{tran2006knn}) &$-$ &$\mathcal{O}(n\log(n) )$  & $\times$& $\times$  &$-$ \Tstrut\\

\hline
Batch K-means (\cite{bottou1995convergence})        &$-$      &$\mathcal{O}(n^2)$                             & $\checkmark$  & $\times$                &$-$                     \Tstrut\\ 
\hline
K-means$++$(\cite{arthur2007k})  &$-$     & $\mathcal{O}(n^{2}) $                               & $\checkmark$  & $\times$          &$-$ \Tstrut\\ 
\hline

Hierarchical Clustering (\cite{carlsson2010characterization})      &$-$       &$\mathcal{O}(n^3)$       & $\checkmark$  & $\times$ &Finite Metric Space\Tstrut    \\
\hline
Stochastic K-means (\cite{tang2016convergence}) &$-$ &$\mathcal{O}(n^{2}) $  &$\checkmark$   & $\times$ & Geometric Assumptions
\Tstrut\\ 
\hline
Spectral Clustering (\cite{von2008consistency}) &$\mathcal{O}(n^{-1/2})$   &$\mathcal{O}(n^3)$            & $\checkmark$      & $\times$                 &General Assumptions                          
\Tstrut\\
\hline

Tensor Decomposition(\cite{hsu2013learning} &$\mathcal{O}(n^{-1/2})$&$\mathcal{O}(n^3)$& $\checkmark$  & $\times$ &No outliers\\
\hline
EM (\cite{balakrishnan2017statistical})    &$\mathcal{O}(n^{-1/2})$  &$\mathcal{O}(nt) $         &  $\checkmark$  & $\times$  & Initialization close enough to MLE                          \Tstrut\\
\hline
EM  for GMM (\cite{balakrishnan2017statistical}) & $\mathcal{O}(n^{-1/2})$&$\mathcal{O}(nt) $  & $\checkmark$ & $\times$  &GMM with init. close enough to MLE\Tstrut\\
\hline

EM  for GMMUB (\cite{browne2011model},\cite{melchior2016filling}) &$ \mathcal{O}(n^{-1/2})$ & $\mathcal{O}(nt) $  & $\checkmark$ &$\checkmark$ &GMMUB\Tstrut\\
\hline
DBSCAN (\cite{sriperumbudur2012consistency})       &$\mathcal{O}(n^{-1/2})$   &$\mathcal{O}(n^2)$         & $\checkmark$  & $\checkmark$   &Holder Continuous Assumption \Tstrut\\
\hline
CRLM  (ours)  &$\mathcal{O}(n^{-1/2})$   &$\mathcal{O}(n^2)$            &$\checkmark$  & $\checkmark$                                                     &\hspace{-5mm}GMM+uniform, separation and coverage
\Tstrut\\
\hline
\end{tabular}}
\vspace{-6mm}
\end{table}

Some  papers have focused on EM algorithms for GMM with a small proportion of noise or outliers. 
For example, in \cite{browne2011model} and \cite{melchior2016filling} presented EM versions that can deal with noisy and incomplete GMM data samples.

In Table \ref{tab:alg} is shown  an overview of the various clustering methods described in this section as well as some of their variations. 
For each method is  shown the computation time in $\mathcal{O}$ notation, whether it has theoretical guarantees of convergence to the true parameters, the convergence rate to the true parameters, whether it is compatible to adding lots of uniform background points (outliers), and the assumptions made by the algorithm about the data.
For some of the algorithms, $t$ represents the number of iteration steps.
 Our algorithm is called CRLM, and we will see that it enjoys a fast convergence rate and an acceptable computational complexity. 

The K-means$++$ method enjoys certain theoretical guarantees, since it finds an optimum of the potential function $\displaystyle \phi=\sum_{x\in \mathcal{X}}\min_{c \in \mathcal{C}}||x-c||^2$ which is bounded by a factor of $\mathcal{O}(\log(k)) $ from the local optimum.
However, the actual  rate of convergence of the estimated parameters to the true model parameters is not clear.
 As for Hierarchical Clustering (\cite{carlsson2010characterization}), the stability and convergence of Hierarchical Clustering are established by measuring the Gromov-Hausdorff distance. Still, the actual  rate of convergence remains unclear. 
Hence, batch K-means, Hierarchical Clustering and K-means$++$, EM  for GMMUB are labeled as clustering methods with theoretical guarantees but without a convergence rate.

\section{Formulation and Algorithm} \label{sec:form}

The problem we are addressing is to cluster a set of unlabeled training examples 
$S=\{\bx_i\in \RR^d, i=1,...n\}$ coming from a mixture of $k$ isotropic Gaussians $\N(\bmu_j, \sigma_j^2I_d)$ with mixture weights $\pi_j$ plus a ``negative'' mixture component containing uniform samples from inside a large ball with radius $D\sqrt{d}$. An example for $d=2$ is shown in Figure \ref{fig:example2d}.

\subsection{Robust Loss Functions}

We will use the following robust loss function
\[
\ell(\bx,\sigma)=\min(\frac{\Vert \bx\Vert^{2}}{d\sigma^{2}}-G,0)
\]
where we fix $G=4$.
Observe that the loss function  is zero outside a ball of radius $R_\sigma=\sigma\sqrt{dG}$.
A graph of the loss for different values of $\sigma$ is given in Figure \ref{fig:loss_d}, left.
\begin{figure*}[htb]
\vspace{-3mm}
\centering
\includegraphics[width=5cm]{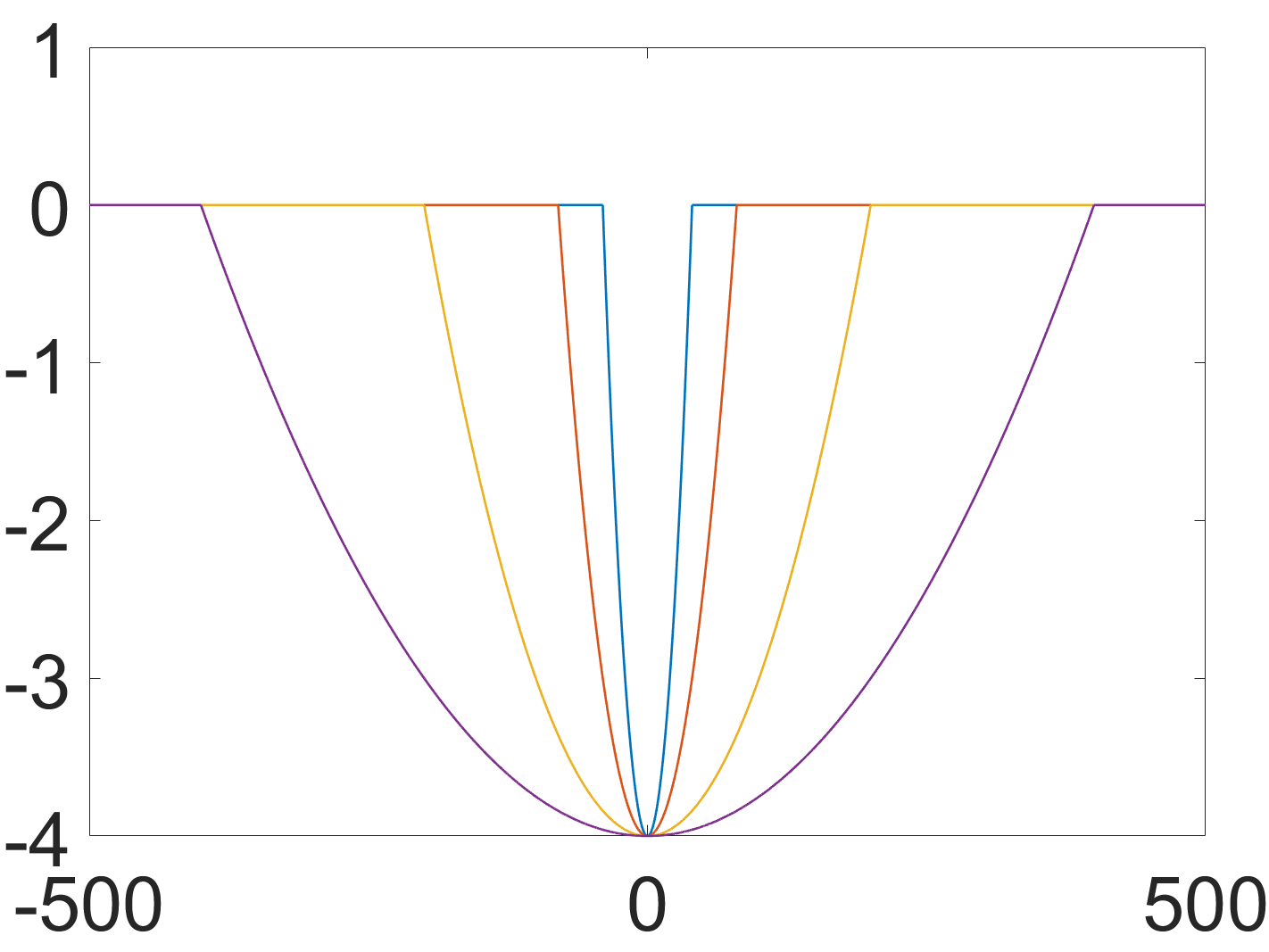}
\hspace{3mm}
\includegraphics[width=5cm]{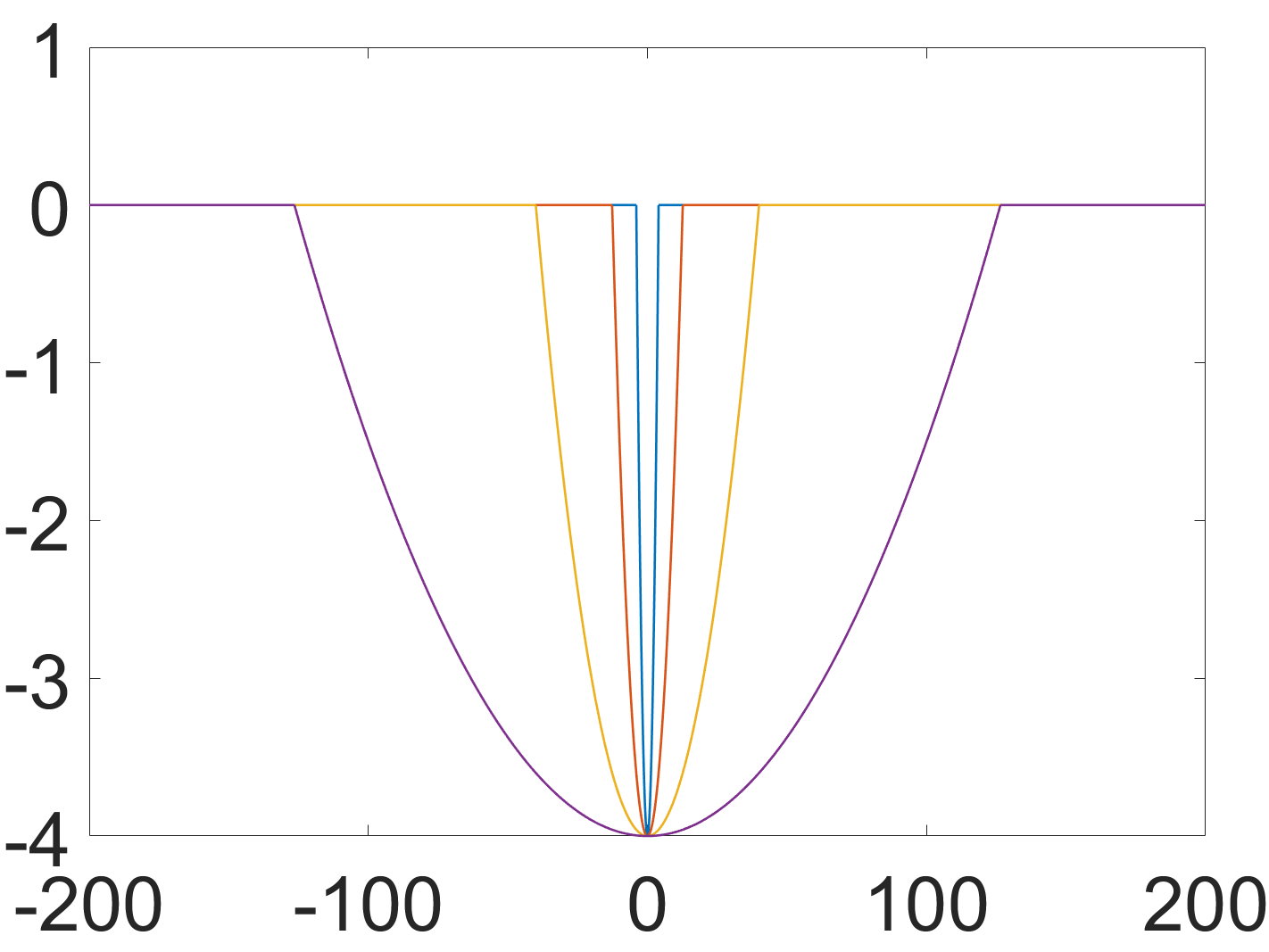}
\vskip -3mm
\caption{Left: The robust loss function $\ell(\bx,\sigma)$ for different values of $\sigma$. Right:
The robust loss function for different values of $d$.}\label{fig:loss_d}
\vspace{-6mm}
\end{figure*}

\subsection{Finding one Cluster By Loss Minimization and One Step Mean Shift}

The goal is to find the cluster parameters $({\bmu},\sigma)$ by minimizing the cost function:
\vspace{-2mm}
\begin{equation}
L(\bmu,\sigma)=\sum_{i=1}^N \ell(\bx_i -\bmu,\sigma) \label{eq:loss}
\vspace{-2mm}\end{equation}

For that, the cost function $L(\bmu,\sigma)$ is computed with center $\bmu$ at each training example and $\sigma=\sigma_{\max}$, a fixed value. The pair $(\bx_i,\sigma_{max})$ of minimum loss is then used as the initialization for one step of the mean shift algorithm. 
The algorithm is described in detail in Algorithm \ref{alg:closs}. 
\vspace{-2mm}
 \begin{algorithm}[htb]
   \caption{{\bf Finding One Cluster by Robust Loss Minimization (OCRLM)}}
   \label{alg:closs}
\begin{algorithmic}
   \STATE {\bfseries Input:} Training examples $S=\{\bx_i\in \RR^d, i=1,...,n\}$, initial standard deviation $\sigma_{\max}$.
   \STATE {\bfseries Output:} Cluster points C, cluster center $\hat \bmu$ and standard deviation $\hat \sigma$.
\end{algorithmic}
\begin{algorithmic} 
\STATE Find $i=\argmin_{i}L(\bx_{i},\sigma_{\max})$.
\STATE Obtain the positive cluster as 
\vspace{-2mm}
\[
                C=\{\bx\in S, \|\bx-\bx_i\|<\sqrt{dG}\sigma_{\max}\}
\vspace{-6mm}
 \]
 \IF {$|C|=1$}
                \STATE 
\vspace{-5mm}
                \[
                \hat \bmu=\bx_i, \hat \sigma=\sigma_{\max}
  \vspace{-2mm}
               \]
  \ELSE 
                \STATE 
\vspace{-5mm}
\[
\hat \bmu=\frac{1}{|C|}\sum_{\bx\in C}\bx, \;\;
               \hat  \sigma^2=\frac{1}{d(|C|-1)}\sum_{\bx\in C}\|\bx-\hat \bmu\|^2
 \vspace{-2mm}
               \]
  \ENDIF
\end{algorithmic}
\end{algorithm}

\subsection{Finding Multiple Clusters}

To find multiple clusters, the one cluster finding algorithm is called repeatedly, after each call eliminating the detected cluster points.

The first cluster by CRLM is regarded as the cluster with the largest clusterability in terms of minimization of robust loss function.  It is similar to  the cluster with minimal  distances within the points of the cluster.  
Unlike other methods that update all cluster means at the same time, CRLM finds the means of different clusters at different iterations. 
Another notable feature for CRLM is that it doesn't cluster all the point, but leaves the points that cannot be clustered to the background cluster. That is a key point why it can cluster data from a GMMUB model with high probability.
 \begin{algorithm}[htb]
   \caption{{\bf Clustering by  Robust Loss Minimization (CRLM)}}
   \label{alg:multclust}
\begin{algorithmic}
   \STATE {\bfseries Input:} Training examples $\{\bx_i\in \RR^d\}_{i=1}^{n}$, maximum number of clusters $k$,  initial standard deviation $\sigma_{\max}$.
   \STATE {\bfseries Output:} Cluster centers $\hat \bmu_j$ with standard deviation $\hat \sigma_j, j=1,...,k$.
\end{algorithmic}
\begin{algorithmic} 
        \FOR {j = 1 to $k$}
                \STATE Find cluster $(C_j,\hat \bmu_j,\hat \sigma_j)$ using OCRLM.
               \IF {$|C_j|=1$}
                \STATE break
               \ENDIF                
               \STATE Remove all observations $\bx_i\in C_j$.
      \ENDFOR
\end{algorithmic}
\end{algorithm}

\section{Main Results}
First, we will set up the notation used in this paper and the main assumptions used in the derivation of our main theorems.
\subsection{Notations}
In the rest of the paper we will use the following terms:
\begin{itemize}
\item $n$ - the number of observations
\item $k$ - the number of positive clusters 
\item $d$ - the dimension of the observations, $\bx_i\in \RR^d$.
\item $D\sqrt{d}$ - a bound for the norm of the observations to be clustered in $\RR^d$
\item $\pi_j$ - the true mixture weight of positive cluster $j$
\item $\bmu_j,\sigma_j$ - the true mean and  standard deviation of positive cluster $j$
\item $\hat\bmu_j, \hat \sigma_j$ - the estimated  mean and standard deviation of positive cluster $j$.
\item $S_{j}$ -  the points contained in the positive cluster $j$
\item $\sigma_{\max}$ - a large initial standard deviation for clustering
\item $G$ - a constant in the loss function, usually $G>1$. In this paper, for the experiments, we use $G = 4$
\item $R_\sigma= \sigma\sqrt{dG}$
\end{itemize}

\begin{figure*}[ht]
\centering
\includegraphics[height=4.cm]{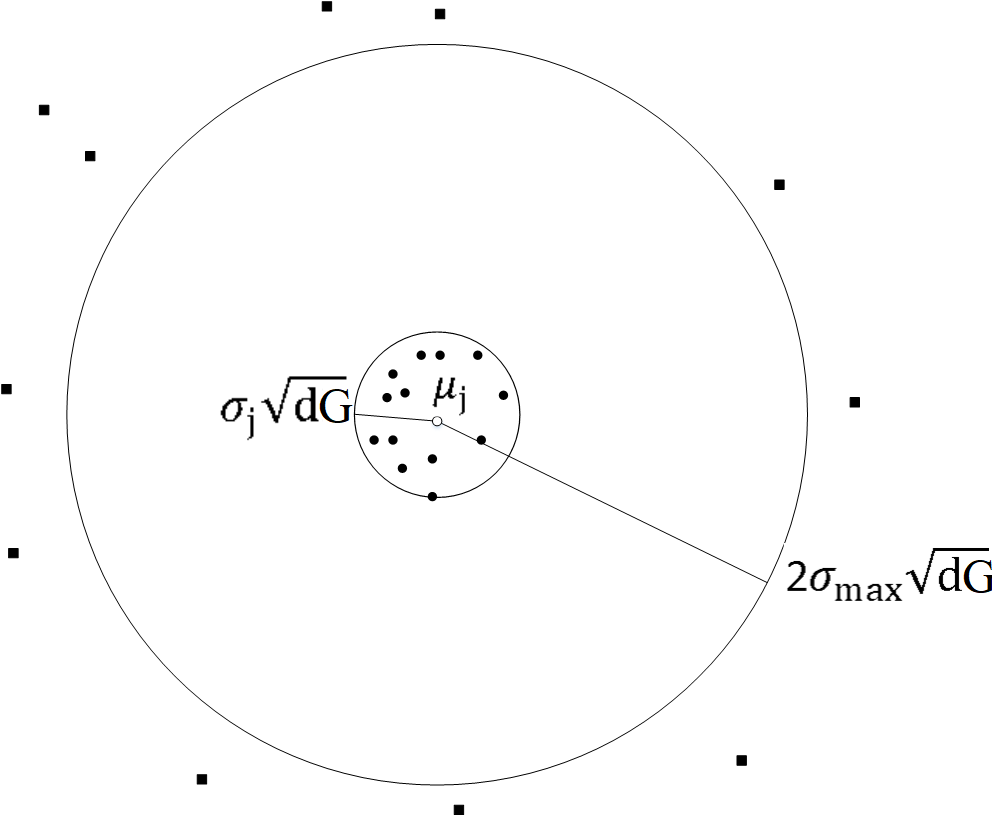}
\hspace{10mm}
\includegraphics[height=4.cm]{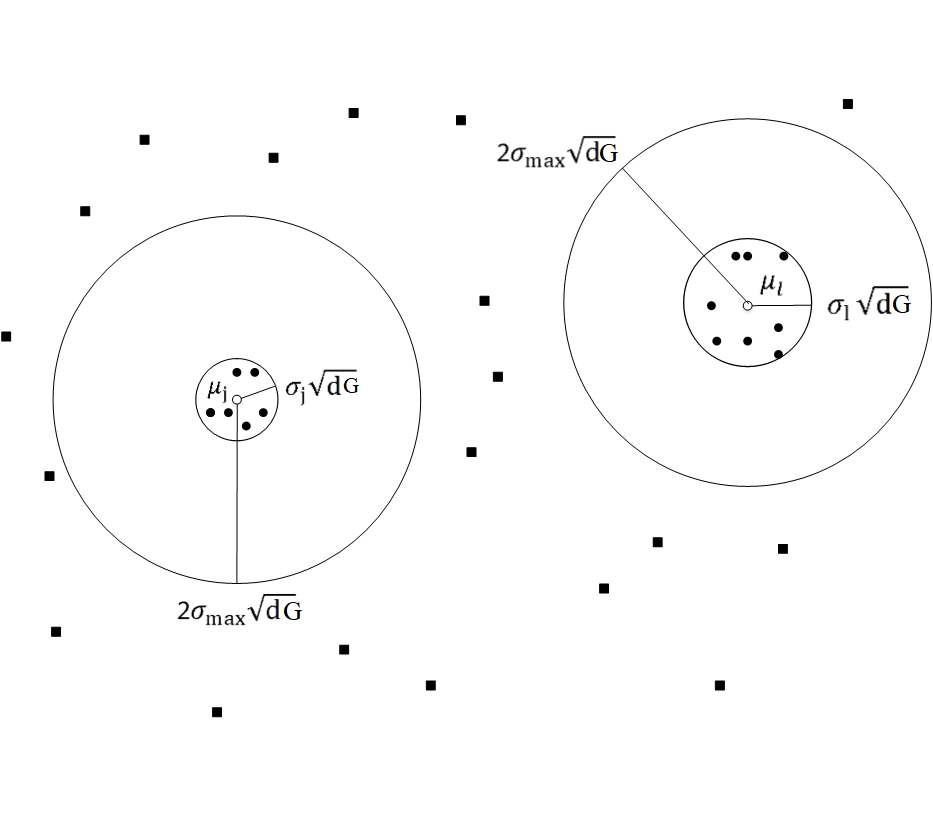}
\vskip -3mm
\caption{Diagram illustrating  C1 (left) and C2 (right). The positives are shown as circles and the negatives as squares.}\label{fig:Conditions}
\vspace{-3mm}
\end{figure*} 

\subsection{Assumptions} \label{sec:assump}
The following Separation and Concentration Conditions will be used in the proof of our main theorem. 
These conditions are illustrated in Figure \ref{fig:Conditions}. We will later show that these conditions happen with high probability.

\noindent{\bf C1: Separation Condition Between Positives and Negatives}: There are no negative points at a distance less than $R=\sigma_{\max}\sqrt{dG}$ from any positive point.

\noindent{\bf C2: Concentration  Condition for Positives}:
For any positive cluster $S_j$ with true mean $\bmu_j$ and covariance matrix $\sigma_j^2I_d$ we have
\[
 ||\bx_i-\bmu_j||<\sigma_j\sqrt{dG},\quad \forall \bx_i \in S_j.
 \]

To get an overall probability guarantee for C1 and C2, we have Proposition \ref{prop:allconditions}, based on the following assumptions:

\noindent{\bf A1:  Large $D$ assumption} 
\[
D>2\sigma_{\max}\sqrt{G} \text{ and } D\sqrt{d}>\|\bmu_j\|+2\sigma_{\max}\sqrt{dG} , \quad  \forall j \in\{1,...,k\}.
\]

\noindent{\bf A2: Separation Assumption Between Positive Clusters}

\[
||\bmu_l-\bmu_j||> 2\sigma_{\max}\sqrt{dG}, \quad  \forall j,l\in\{1,...,k\}, l\not = j.
\]

\noindent{\bf A3: Lower Bound  Assumption for $\sigma_{\max}$}

\[
\sigma_{\max}>2\sigma_j, \quad \forall j\in\{1,...,k\}.
\]

\begin{proposition}\label {prop:allconditions}
Consider  $n$ observations from a GMMUB of  $k$ isotropic Gaussians with mixture weights $\pi_1,\cdots,\pi_k$, true means $\bmu_1,\cdots,\bmu_k$  and variances $\sigma_{1}^2I_d,\cdots,\sigma_{k}^2I_d$ respectively,  and uniform distribution within a ball of radius $D\sqrt{d}$, with weight $\pi_{k+1}$. 
If A1 is satisfied, then C1 and C2 hold with probability at least $$ 1-2ne^{-d(G-1)^2/8}-nk(2\sigma_{\max}\sqrt{G}/D)^d.$$
\end{proposition}
\begin{proof}
Based on Lemma \ref{coverage} in the Appendix, C2 holds with probability at least $1-2ne^{-d(G-1)^2/8}$. This is mainly because for large $d$ the norm $\|\bx_i-\bmu_j\|$ for $\bx_i\in S_j$ is mostly concentrated around $\sigma_j\sqrt{d}$, as illustrated in Figure \ref{fig:gaussian_norm} and guaranteed by the proof of Lemma \ref{deltato0}.
\begin{figure*}[ht]
\centering
\includegraphics[width=10cm]{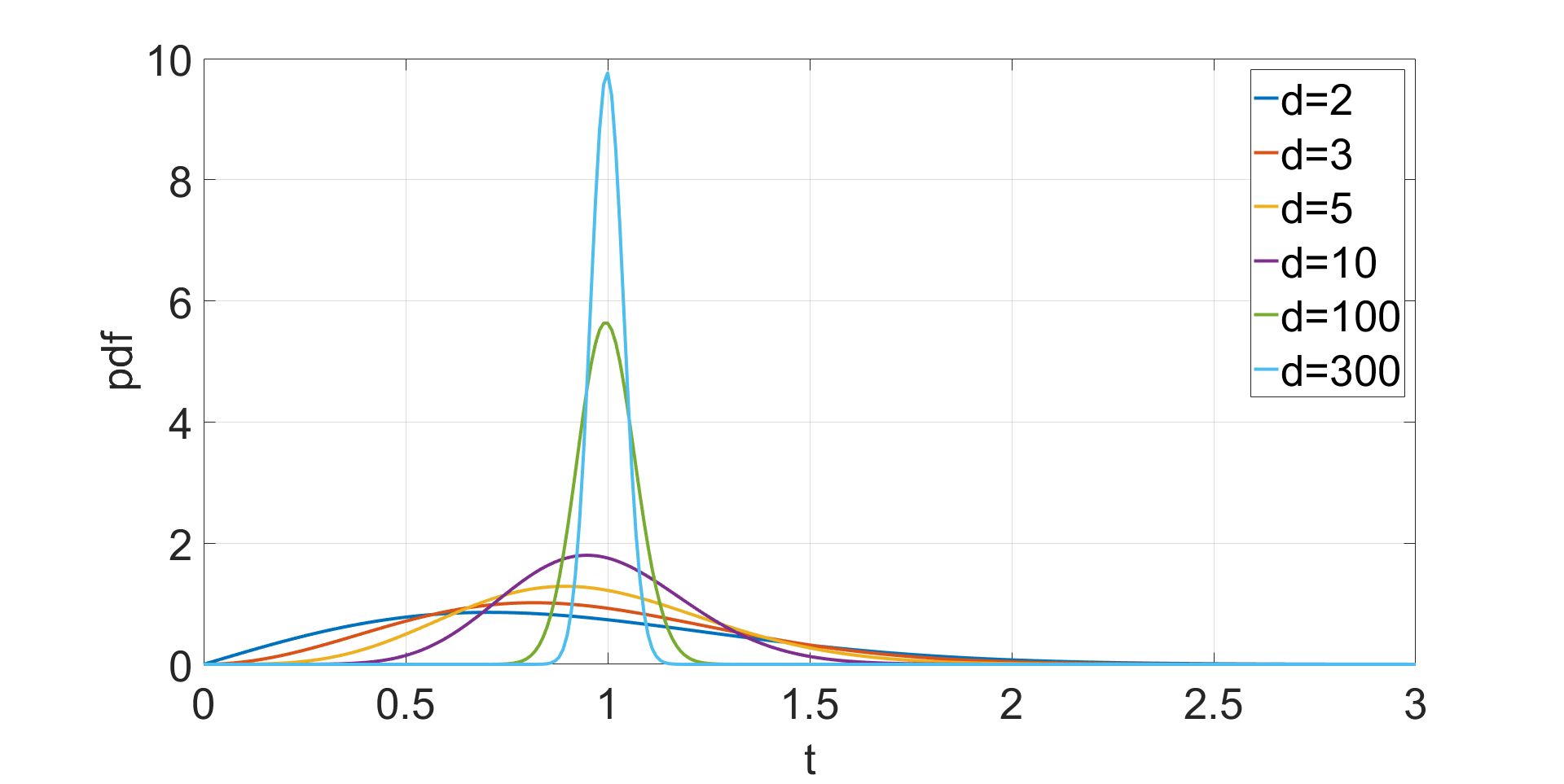}
\vskip -4mm
\caption{The pdf of $\frac{||\bx||}{\sqrt{d}}$ for $d$-dimensional normals $\bx\sim \N(0,I_d)$. Observe that this quantity is concentrated around 1 for large $d$.}\label{fig:gaussian_norm}
\vspace{-2mm}
\end{figure*} 

If A1 is satisfied then let $\bx_m$ be any  negative point. Based on Lemma \ref{C1bounds} , $$P(||\bx_m-\bmu_j||>2\sigma_{\max}\sqrt{dG}, \forall m, \forall j)\geq 1-nk(2\sigma_{\max}\sqrt{G}/D)^d.$$

Then, for any positive point $\bx_i\in S_j$, and for any  negative point $\bx_m$, we have:
\[
||\bx_m-\bx_i||>||\bx_m-\bmu_j||-||\bmu_j-\bx_i||>2\sqrt{dG}\sigma_{\max}-\sqrt{dG}\sigma_{\max}=\sqrt{dG}\sigma_{\max}.
\]
Therefore, C1-C2 hold with probability at least $1-nk(2\sigma_{\max}\sqrt{G}/D)^d-2ne^{\frac{-d(G-1)^2}{8}}$ if A1 holds.
\end{proof}

From C2, two other important results that will be useful for the proof of the main theorem have been derived in Lemma \ref{original condition} in the Appendix.


\subsection{ Theoretical Guarantees }

We start by giving  theoretical guarantees for OCRLM, assuming there is only one Gaussian cluster.

\begin{proposition}\label{prop:oneclust}
Let $\bx_{i}\in \RR^d, i=1,...,n$ be $n$ observations sampled from a mixture of a Gaussian $\N(\bmu_1, \sigma_1^2I_d)$  with weight $\pi_1$ and a uniform distribution inside the ball of radius $D\sqrt{d}$ centered at $0$.
 If for a given $\sigma_{\max}$, C1 and C2 are satisfied and  
\[
\pi_1>\frac{({\sigma_{\max}}\sqrt{G}/D)^dG}{(G-\frac{(1+G)\sigma_1^2}{\sigma_{\max}^2})(d/2+1)+({\sigma_{\max}}\sqrt{G}/D)^dG},
\]
 then with probability at least $1-2n\exp(-nW^2/2G^2)$, OCRLM will cluster all the observations correctly, where
 \begin{equation}
W= \pi_1\left(G-\frac{(1+G)\sigma_1^2}{\sigma_{\max}^2}\right)-(1-\pi_1)\frac{G}{d/2+1}({\sigma_{\max}}\sqrt{G}/D)^d \label{eq:w}
.\end{equation}
\end{proposition}
The proof of this proposition is given in the Appendix. 
 This proposition assumes that conditions C1 and C2 are satisfied but the following theorem replaces these conditions with assumptions A1 and A3.

\begin{theorem}\label{thm:onecluster}
Let $\bx_{i}\in \RR^d, i=1,...,n$ be $n$ observations sampled from a mixture of a Gaussian $\N(\bmu_1, \sigma_1^2I_d)$  with weight $\pi_1$ and a uniform distribution inside the ball of radius $D\sqrt{d}$ centered at $0$. If  A1 and A3 are satisfied and for  a given $\sigma_{\max}$ such that
\begin{equation}
\pi_1>\frac{({\sigma_{\max}}\sqrt{G}/D)^dG}{(G-\frac{(1+G)\sigma_1^2}{\sigma_{\max}^2})(d/2+1)+({\sigma_{\max}}\sqrt{G}/D)^dG}, \label{eq:wtcond1}
\end{equation}
then OCRLM will cluster all observations correctly with probability at least 
\begin{equation}
1-2n\exp(-nW^2/2G^2)-2ne^{-d(G-1)^2/8}-n({2\sigma_{\max}}\sqrt{G}/D)^d,\label{eq:probbnd1}
\end{equation} where $W$ has been defined in Eq. \eqref{eq:w}.
\end{theorem}

\begin{proof} Based on Prop \ref{prop:oneclust}, when C1 and C2 are satisfied,  
OCRLM clusters all observations correctly with probability at least $1-2n\exp(-nW^2/2G^2)$.

Using Prop \ref{prop:allconditions}, it is clear that the probability that both C1 and C2 hold is at least $1-n({2\sigma_{\max}}\sqrt{G}/D)^d-2ne^{-d(G-1)^2/8}$ when A1 and A3 are satisfied.

Hence, OCRLM correctly clusters all observations with probability at least 
 $$1-2n\exp(-nW^2/2G^2)-2ne^{-d(G-1)^2/8}-n({2\sigma_{\max}}\sqrt{G}/D)^d.
 $$
\end{proof}
When there is only one positive cluster, OCRLM will be employed 1 time to find all the positive points.
When  the dimension $d$ of the data and the number $n$ of observations are large enough, the probability in Theorem \ref{thm:onecluster} can converge to 1.

To generalize Theorem  \ref{thm:onecluster} to $k$ positive clusters, we need to prove  that  Statement 1 and Statement 2 from Lemma \ref{original condition} hold with high probability. For that we generalize Proposition \ref{prop:oneclust} to multiple Gaussians.

\begin{proposition}\label{prop:multiclust}
Let $\bx_{i}\in \RR^d, i=1,...,n$ be $n$ observations sampled  from a mixture of $k$  isotropic Gaussians with means $\bmu_1,\cdots,\bmu_k$, covariance  matrices $\sigma_1^2 I_{d},\cdots,\sigma_k^2I_{d}$, weights $\pi_1,\cdots,\pi_k$ and the uniform distribution within a ball of radius $D\sqrt{d}$ centered at the origin, with weight $\pi_{k+1}$, so that  $\pi_1+\cdots\pi_k+\pi_{k+1}=1$. 
 Assume C1-C2 and A1-A3 hold.. If  $\forall j \in \{1,\cdots k\} $, $${\pi_j}>\frac{(G/(d/2+1))({\sigma_{\max}}\sqrt{G}/D)^d}{(G-\frac{(1+G)\sigma_j^2}{\sigma_{\max}^2})+(G/(d/2+1))({\sigma_{\max}}\sqrt{G}/D)^d},$$ 
then CRLM will correctly cluster all the points  with probability at  least  \[1-2nk\exp(-n\min_j W_j^2/2G^2),\] where
 \begin{equation}
W_j= \pi_j(G-\frac{(1+G)\sigma_j^2}{\sigma_{\max}^2})-\pi_{k+1}({\sigma_{\max}}\sqrt{G}/D)^d\frac{G}{d/2+1}\label{eq:w_j}.
\end{equation}
\end{proposition}

The proof of this Proposition is given in the Appendix.

Based on this Proposition, we obtain a theorem for finding multiple positive clusters:

\begin{theorem}\label{thm:multclust}
Let $\bx_{i}\in \RR^d, i=1,...,n$ be $n$ observations sampled  from a mixture of  isotropic GMM with means $\bmu_1,\cdots,\bmu_k$, covariance  matrix $\sigma_1^2 I_{d},\cdots,\sigma_k^2I_{d}$, weight $\pi_1,\cdots,\pi_k$ and uniform distribution within radius $D\sqrt{d}$, with weight $\pi_{k+1}$. $\pi_1+\cdots\pi_k+\pi_{k+1}=1$.  
If  $\forall j \in \{1,\cdots k\} $,  A1-A3  and $G>1$ are satisfied and
$$
{\pi_j}>\frac{(G/(d/2+1))({\sigma_{\max}}\sqrt{G}/D)^d}{(G-\frac{(1+G)\sigma_j^2}{\sigma_{\max}^2})+(G/(d/2+1))({\sigma_{\max}}\sqrt{G}/D)^d},$$  then CRLM correctly clusters all the positives with probability at least  
$$1-2nk\exp(-n\min_j W_j^2/2G^2)-nk ({2\sigma_{\max}}\sqrt{G}/D)^d-2ne^{-d(G-1)^2/8},$$  
where $W_j$ has been defined in Eq. \eqref{eq:w_j}.

\end{theorem}

\begin{proof}

We already showed that C1 and C2 hold with probability at least $1-nk ({2\sigma_{\max}}\sqrt{G}/D)^d-2ne^{-d(G-1)^2/8}$, when A1-A3 hold.
According to Prop \ref{prop:multiclust}, if C1 and C2 hold, CRLM will cluster all the points correctly with probability at least $1-2nk\exp(-n\min_j W_j^2/2G^2).$ 
Hence, when A1-A3 are satisfied, CRLM correctly clusters all observations with probability at least 
 $$1-nk ({2\sigma_{\max}}\sqrt{G}/D)^d-2ne^{-d(G-1)^2/8}-2nk\exp(-n\min_j W_j^2/2G^2).$$
\end{proof}



Based on the result, we have the following Corollary \ref{cor:convergencerate} for the convergence of CRLM in terms of the norm of the distance between the estimated means and the true means.

\begin{corollary}\label{cor:convergencerate}
Let $\bx_{i}\in \RR^d, i=1,...,n$ be $n$ observations sampled  from a mixture of  isotropic GMM with means $\bmu_1,\cdots,\bmu_k$, covariance  matrices $\sigma_1^2 I_{d},\cdots,\sigma_k^2I_{d}$, weights $\pi_1,\cdots,\pi_k$ and a uniform distribution within radius $D\sqrt{d}$, with weight $\pi_{k+1}$, so that $\pi_1+\cdots\pi_k+\pi_{k+1}=1$.  
If A1-A3  hold, and
$$
{\pi_j}>\frac{(G/(d/2+1))({\sigma_{\max}}\sqrt{G}/D)^d}{(G-\frac{(1+G)\sigma_j^2}{\sigma_{\max}^2})+(G/(d/2+1))({\sigma_{\max}}\sqrt{G}/D)^d}, \forall j \in \{1,\cdots k\},
$$  
 denote $\hat{\bmu_j}$ as the estimated mean of the $j$-th positive cluster by CRLM. Then with probability at least 
\[
1-2nke^{-\frac{n\min_j W_j^2}{2G^2}}
\hspace{-1mm} -\hspace{-0.4mm}2\hspace{-1mm}\sum_{j=1}^{k}e^{
-\frac{d}{8}(\frac{2\pi_j\max_j\sigma_j^2}{\sigma_j^2\min_j\pi_j}-1)^2
}\hspace{-1mm}-\hspace{-1mm}
\sum_{j=1}^{k}e^{-\frac{n\pi_j^2}{2}}\hspace{-1mm}-nk \left(\frac{2\sigma_{\max}\sqrt{G}}{D}\right)^d
\hspace{-2mm}-2ne^{\frac{-d(G-1)^2}{8}},\]
we have 
\[
||\hat{\bmu_j}-{\bmu_j}||^2<\frac{4d\max_j\sigma_j^2}{n\min_j\pi_j}, \forall j,
\]
 where $W_j$ has been defined in Eq. \eqref{eq:w_j}.
\end{corollary}
\begin{proof}
Based on Proposition \ref{prop:allconditions}, C1-C2 hold with probability at least \[1-nk \left(\frac{2\sigma_{\max}\sqrt{G}}{D}\right)^d-2ne^{-d(G-1)^2/8}.\]
Based on Theorem \ref{thm:multclust} and Lemma \ref{original condition}, when C1-C2 and A1-A3 hold, then with probability at least $1-2nk\exp(-n\min_j W_j^2/2G^2),$
$\hat{\bmu_j}= \bmu_j^{'}$, $\forall j \in \{1,\cdots,k \}$, where $\bmu_j^{'}$ is the sample mean for $j$-th positive cluster.

Let $\epsilon=\sqrt{\frac{4\max_j\sigma_j^2}{\min_j\pi_j}\frac{d}{n}}$. 
It suffices to show that $||\bmu_j^{'}-\bmu_j||_2<\epsilon$, $\forall j$ ,  with certain  probability. Denote by $n_j$ the true number of positive points in cluster $j$.

Since $n_j \sim$ Binomial $(n, \pi_j)$, the Hoeffding's inequality generates the bound $$P(n_j<n\pi_j/2)\leq \exp\left(-\frac{n\pi_j^2}{2}\right)$$

Therefore, with probability at least $1-\sum_{j=1}^{k}\exp\left(-\frac{n\pi_j^2}{2}\right)$, we have that $n_j>n\pi_j/2$, $\forall j$.

If $\epsilon=\sqrt{\frac{4\max_j\sigma_j^2}{\min_j\pi_j}\frac{d}{n}}$, we therefore have that 
$\epsilon > \frac{\sigma_j\sqrt{2d}}{\sqrt{n\pi_j}} >\frac{\sigma_j\sqrt{d}}{\sqrt{n_j}}$, $\forall j$ thus $\hat{\epsilon}_j=\frac{\sqrt{n_j}}{\sigma_j}\epsilon > \sqrt{d}$, for any $j$. Observe that 
$\bmu_j^{'} \sim N(\bmu_j, \frac{\sigma_j^2}{n_j}  I_d)$, therefore $\frac{\sqrt{n_j}}{\sigma_j}(\bmu_j^{'}-\bmu_j) \sim N (0,I_d)$. Let $z =\frac{\sqrt{n_j}}{\sigma_j}(\bmu_j^{'}-\bmu_j) $, therefore, $G^{*} =\frac{\hat{\epsilon}^2}{d} >1$ . By Lemma \ref{deltato0} we have:
 \begin{eqnarray*}
P(||z||_2>\hat{\epsilon})&=&P\left (||\bmu_j^{'}-\bmu_j||_2>\epsilon\right)\\
& = & P (\frac{||z||_2}{\sqrt{d}}>\frac{\hat{\epsilon}}{\sqrt{d}}) = P (\frac{||z||_2^2}{d}>G^{*} ) \\ 
&<&2e^{-d(G^{*}-1)^2/8}=2\exp{(-\frac{d}{8}(\frac{n_j\epsilon^2}{d\sigma_j^2}-1)^2)}\\
\end{eqnarray*}

Hence, if CRLM obtains  right clusters, for  $j$, with probability at least $ 1-2\exp{(-\frac{d}{8}(\frac{n_j\epsilon^2}{d\sigma_j^2}-1)^2)}$, $||\bmu_j^{'}-\bmu_j||_2<\epsilon$. For all $j$, we have that:  $||\bmu_j^{'}-\bmu_j||_2<\epsilon$, $\forall j$,   with probability at least $ 1-2\sum_{j=1}^{k}\exp{(-\frac{d}{8}(\frac{n_j\epsilon^2}{d\sigma_j^2}-1)^2)}$, but since $n_j>n\pi_j/2$, $\forall j$, we have
\[
\sum_{j=1}^{k}2 \exp{(-\frac{d}{8}(\frac{n_j\epsilon^2}{d\sigma_j^2}-1)^2)}<2\sum_{j=1}^{k}\exp{(-\frac{d}{8}(\frac{n\pi_j\epsilon^2}{2d\sigma_j^2}-1)^2)}.
\]

Putting it all together,  when A1-A3 hold, we have that $||\hat{\bmu_j}-\bmu_j||_2<\epsilon, \forall j$,  with probability at least
$$1-2nke^{-\frac{n\min_j W_j^2}{2G^2}}\hspace{-1mm}-\hspace{-0.4mm}2\hspace{-1mm}\sum_{j=1}^{k}e^{
-\frac{d}{8}(\frac{2\pi_j\max_j\sigma_j^2}{\sigma_j^2\min_j\pi_j}-1)^2
}\hspace{-1mm}-\hspace{-1mm}
\sum_{j=1}^{k}e^{-\frac{n\pi_j^2}{2}}\hspace{-1mm}-nk \left(\frac{2\sigma_{\max}\sqrt{G}}{D}\right)^d
\hspace{-2mm}-2ne^{\frac{-d(G-1)^2}{8}}.$$
\end{proof}

Observe that the probabilities for correct clustering in Theorem \ref{thm:multclust} and Corollary \ref{cor:convergencerate} decrease as the sample size $n$ increases if the space dimension $d$ is fixed. This is because the ball where all observations are assumed to reside is bounded, and it becomes crowded as $n\to \infty$. However, we will see that this is not a problem for practical applications. Furthermore, by imposing some mild constraints on $d$ and $D$ we can make sure that every probability bound converges to 1.

\begin{corollary}\label{cor:probcheck}
Based on the conditions of Corollary \ref{cor:convergencerate}, let $c=\min(\frac{(G-1)^2}{8},\log \frac{D}{2\sigma_{\max}\sqrt{G}}) >0$. If $cd>2\log(n) $, then for any $\hat{\epsilon} >0$, there exists $N>0$, such that when $n>N$, the probability bound from  Corollary \ref{cor:convergencerate} satisfies 
\[
1-2nke^{-\frac{n\min_j W_j^2}{2G^2}}
\hspace{-1mm} -\hspace{-0.4mm}2\hspace{-1mm}\sum_{j=1}^{k}e^{
-\frac{d}{8}(\frac{2\pi_j\max_j\sigma_j^2}{\sigma_j^2\min_j\pi_j}-1)^2
}\hspace{-1mm}-\hspace{-1mm}
\sum_{j=1}^{k}e^{-\frac{n\pi_j^2}{2}}\hspace{-1mm}-nk \left(\frac{2\sigma_{\max}\sqrt{G}}{D}\right)^d
\hspace{-2mm}-2ne^{\frac{-d(G-1)^2}{8}} > 1- \hat{\epsilon}\]
\end{corollary}

\begin{proof}
Take $\tilde{\epsilon} = \hat{\epsilon}/5$, it suffices to show each term in the bound is smaller than $\tilde{\epsilon} $.

Since $ \lim_{n\to \infty} 2nke^{-\frac{n\min_j W_j^2}{2G^2}} =0 $, $\exists N_1>0$, when $n>N_1$,   $2nke^{-\frac{n\min_j W_j^2}{2G^2}} < \tilde{\epsilon} $.

For the second term we have
 \begin{eqnarray*}
\lim_{n\to \infty} \exp{({
-\frac{d}{8}(\frac{2\pi_j\max_j\sigma_j^2}{\sigma_j^2\min_j\pi_j}-1)^2
})}&\leq &\lim_{n\to \infty} \exp{(-\frac{d}{8})}\leq \lim_{n\to \infty} \frac{1}{8n}=0
\end{eqnarray*}

Therefore, $\exists N_2>0$, when $n>N_2$,   $\sum_{j=1}^{k}e^{
-\frac{d}{8}(\frac{2\pi_j\max_j\sigma_j^2}{\sigma_j^2\min_j\pi_j}-1)^2
} < \tilde{\epsilon} $.

It is clear that $ \lim_{n\to \infty} \sum_{j=1}^{k}e^{-\frac{n\pi_j^2}{2}}=0 $, so $\exists N_3>0$, when $n>N_3$,   $\sum_{j=1}^{k}e^{-\frac{n\pi_j^2}{2}} <\tilde{\epsilon}$.

For the fourth term, let $C=\log \frac{D}{2\sigma_{\max}\sqrt{G}}>0$. Then $nk \left(\frac{2\sigma_{\max}\sqrt{G}}{D}\right)^d < \tilde{\epsilon} $ is equivalent to $dC>\log(nk/\tilde{\epsilon})$. Since $dc>\log(n)$, then $dC>2\log (n)$ so  $\exists N_4>0$ so that when $n>N_4$, $dC>2\log(n)>\log(n) + \log(k/\tilde{\epsilon})=\log(nk/\tilde{\epsilon})$.

Finally, $2ne^{\frac{-d(G-1)^2}{8}} < \tilde{\epsilon} $ is equivalent to $d\frac{(G-1)^2}{8}>\log(2n/\tilde{\epsilon})$. But since $dc>2\log(n)$ then $\exists N_5>0$ such that when $n>N_5$ we have $d\frac{(G-1)^2}{8}>2\log(n)>\log(n)+\log(2/\tilde{\epsilon})=\log(2n/\tilde{\epsilon})$.

To sum it up, taking $N = \max(N_1,N_2,N_3,N_4,N_5)$, we have that when $n>N$,
\[
1-2nke^{-\frac{n\min_j W_j^2}{2G^2}}
\hspace{-0.4mm}-2\hspace{-0.4mm}\sum_{j=1}^{k}e^{
-\frac{d}{8}(\frac{2\pi_j\max_j\sigma_j^2}{\sigma_j^2\min_j\pi_j}-1)^2
}
\hspace{-2mm}-\sum_{j=1}^{k}e^{-\frac{n\pi_j^2}{2}}-nk \left(\frac{2\sigma_{\max}\sqrt{G}}{D}\right)^d
\hspace{-1mm}-2ne^{\frac{-d(G-1)^2}{8}} > 1-5\tilde{\epsilon} = 1- \hat{\epsilon}\]
\end{proof}

The purpose of Corollary \ref{cor:probcheck} is to show that in order for the probability bound from Corollary \ref{cor:convergencerate} to stay high, we need the dimension $d$ to grow on the order of $\log(n)$.  

For example, one typical setting in our experiments is that $G=4$, $k=1$, $\sigma_{max} = 10, \sigma_1 =1$.
The constant $D$ is an important parameter as well and in the experiment, $D$ is usually large enough compared to 
  $2\sigma_{\max}\sqrt{G}$. A very common choice is $D=100$. 
Then in the case when the dimension $d =500$,  using the results from Corollary \ref{cor:probcheck}, we can achieve a $99.99\%$ probability in Corollary \ref{cor:convergencerate}, as long as the sample size is $n<3.0 \times 10^{99}$, which sufficient for all practical applications. Observe that in computer vision applications, the dimension $d$ of the representation space (obtained using convolutional neural networks) is typically in the interval $d\in [1024,4096]$, obtaining an even larger maximum sample size.

\subsection{ Computational Complexity}
In CRLM, for each iteration we run OCRLM once. Hence, the complexity of CRLM is just $k$ times the complexity of OCRLM. For OCRLM, in the worst case,  it suffices to calculate $\ell(\bx_i - \bx_j, \sigma_{\max}), i,j \in \{1,\cdots n\}.$ Hence the computational complexity for CRLM is 
$\mathcal{O}(kn^2d)$. Since $k$ and $d$ are fixed, the computational complexity of CRLM is therefore $\mathcal{O}(n^2).$

\section{Experiments}
In this section, we will perform experiments to compare our method with other clustering algorithms. 
First of all, we conduct experiments on synthetic data and show an analysis of the effect of $\sigma_{\max}$ on the clustering results, and a method to find the number of clusters $k$. 
Finally, we perform experiments on several kinds of real image data to show the value of our algorithm in real applications.

The algorithms involved in the comparison of the experiments are: K-means (\cite{lloyd1982least}), where we use the K-means++ version (\cite{arthur2007k}), 
DBSCAN\footnote{from \url{https://www.mathworks.com/matlabcentral/fileexchange/52905-dbscan-clustering-algorithm}} (\cite{dempster1977maximum}), 
Complete Linkage Clustering (CL) (\cite{johnson1967hierarchical}) based on Euclidean distance, 
EM (\cite{dempster1977maximum}), 
Spectral Clustering (SC)\footnote{from \url{https://www.mathworks.com/matlabcentral/fileexchange/26354-spectral-clustering-algorithms}} (\cite{ng2002spectral,shi2000normalized}),
Tensor Decomposition (TD) \cite{hsu2013learning}, T-SNE \cite {maaten2008visualizing}, KNN-CLUST(KC)\cite{tran2006knn}
 and CRLM.

For Spectral Clustering we choose the method from \cite{ng2002spectral}.

For Tensor Decomposition (TD), we use an implementation of Theorem 2 from \cite{hsu2013learning}.

For KNNCLUST(KC), we implement the method based on the MATLAB implementation in the toolbox as specified in \cite {tran2006knn}. We select the best result with different  $k$.

We use MATLAB for most of the experiments. 
 For K-means, we use the built-in function 'kmeans' from Matlab that actually implements K-means++.
 For Hierarchical Clustering, we use the built-in function 'linkage' with complete linkage based on Euclidean distance.
For T-SNE + K-means, the Matlab built-in functions 'tsne' and 'kmeans' were employed.

In terms of EM, a standard EM for GMM is used for Table \ref{Kimia216}.  For all other experiments regarding EM, we derived updates for GMMUB by adding a likelihood term for the uniform background and choose the result with largest likelihood from 10 different initializations.

\subsection{Simulation Experiments}

In this section we perform experiments on data coming from a GMMUB that satisfies  A1-A3.

\begin{figure}[t]\centering
\includegraphics[height=3.5cm]{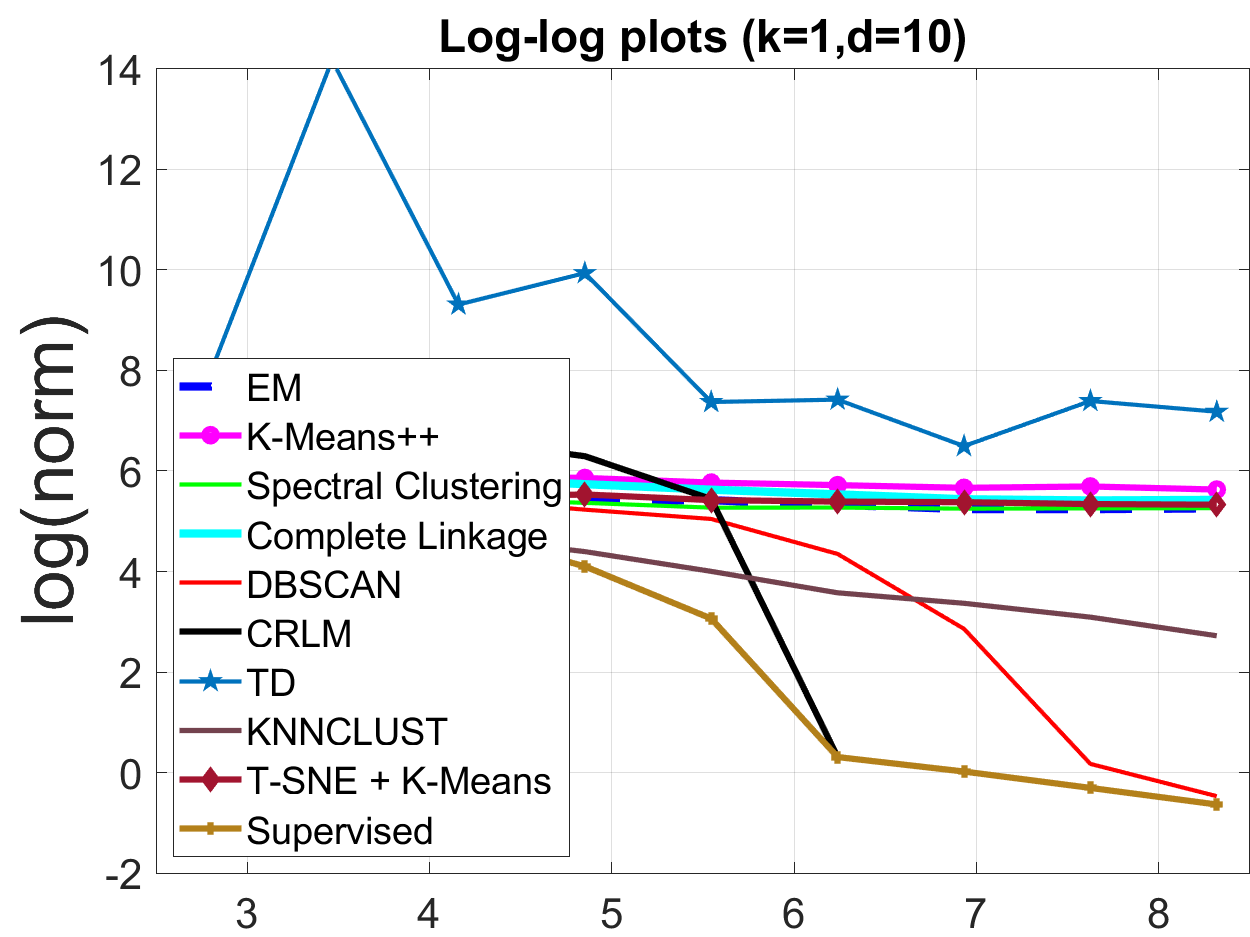}
\includegraphics[height=3.5cm]{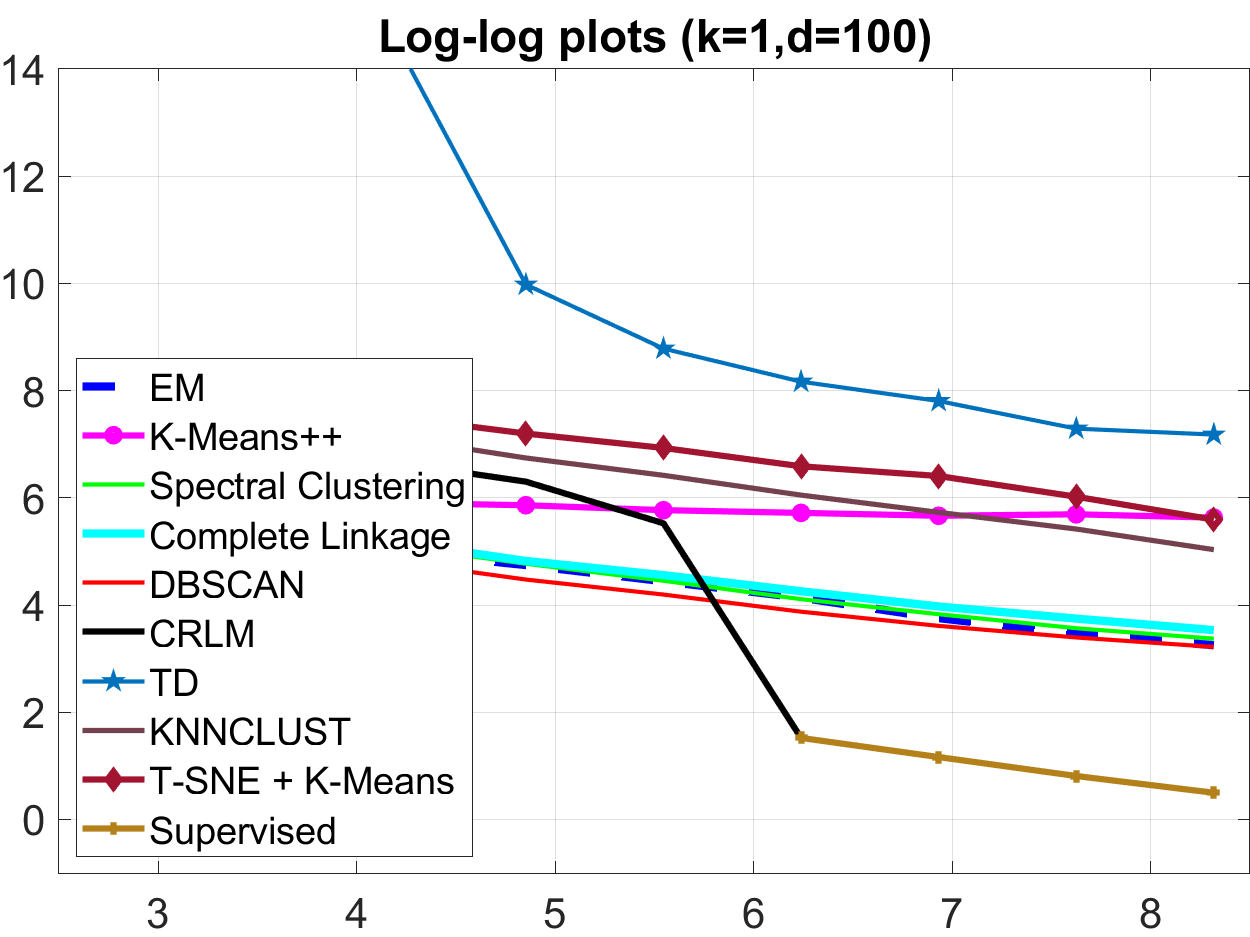}
\includegraphics[height=3.5cm]{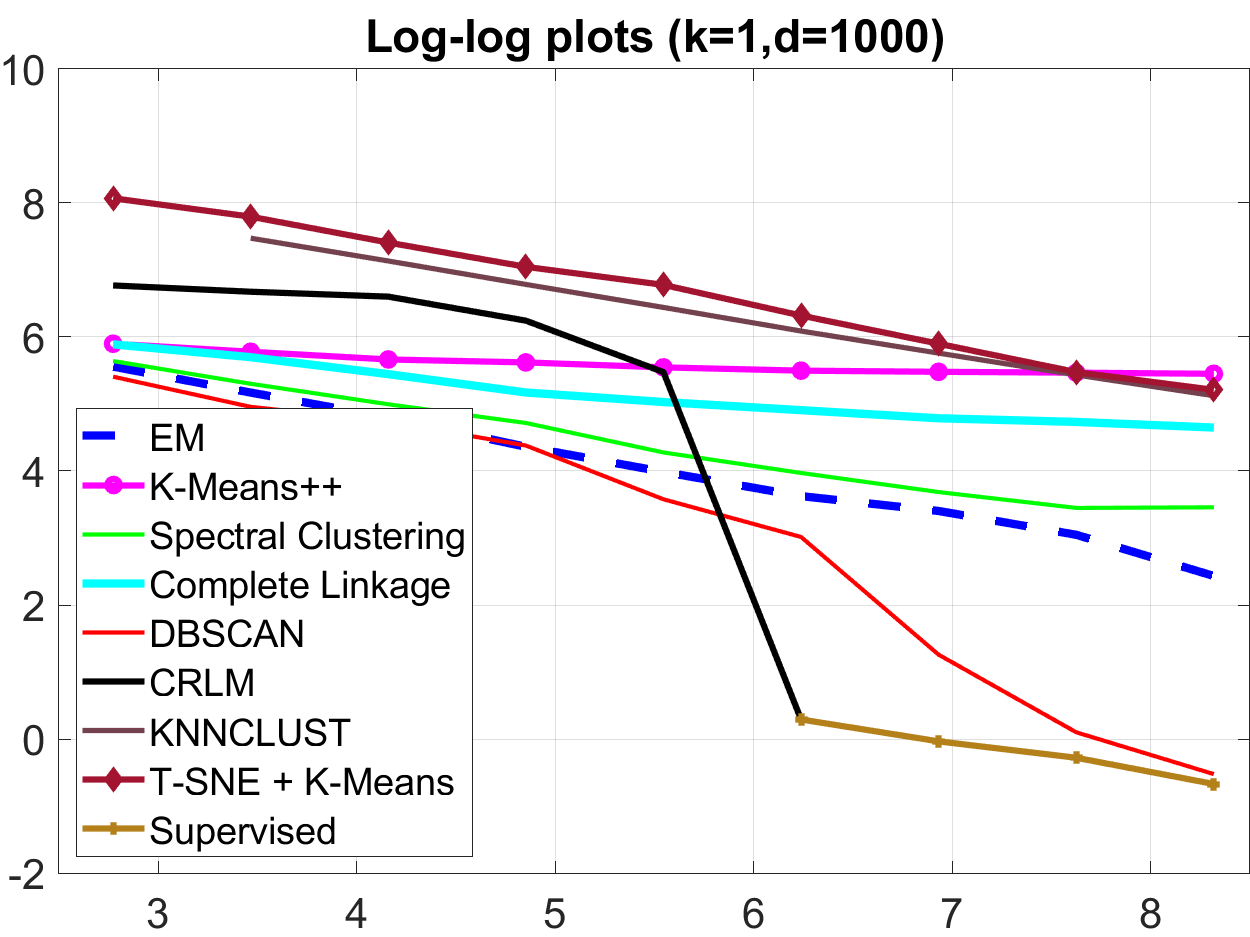}
\includegraphics[height=3.5cm]{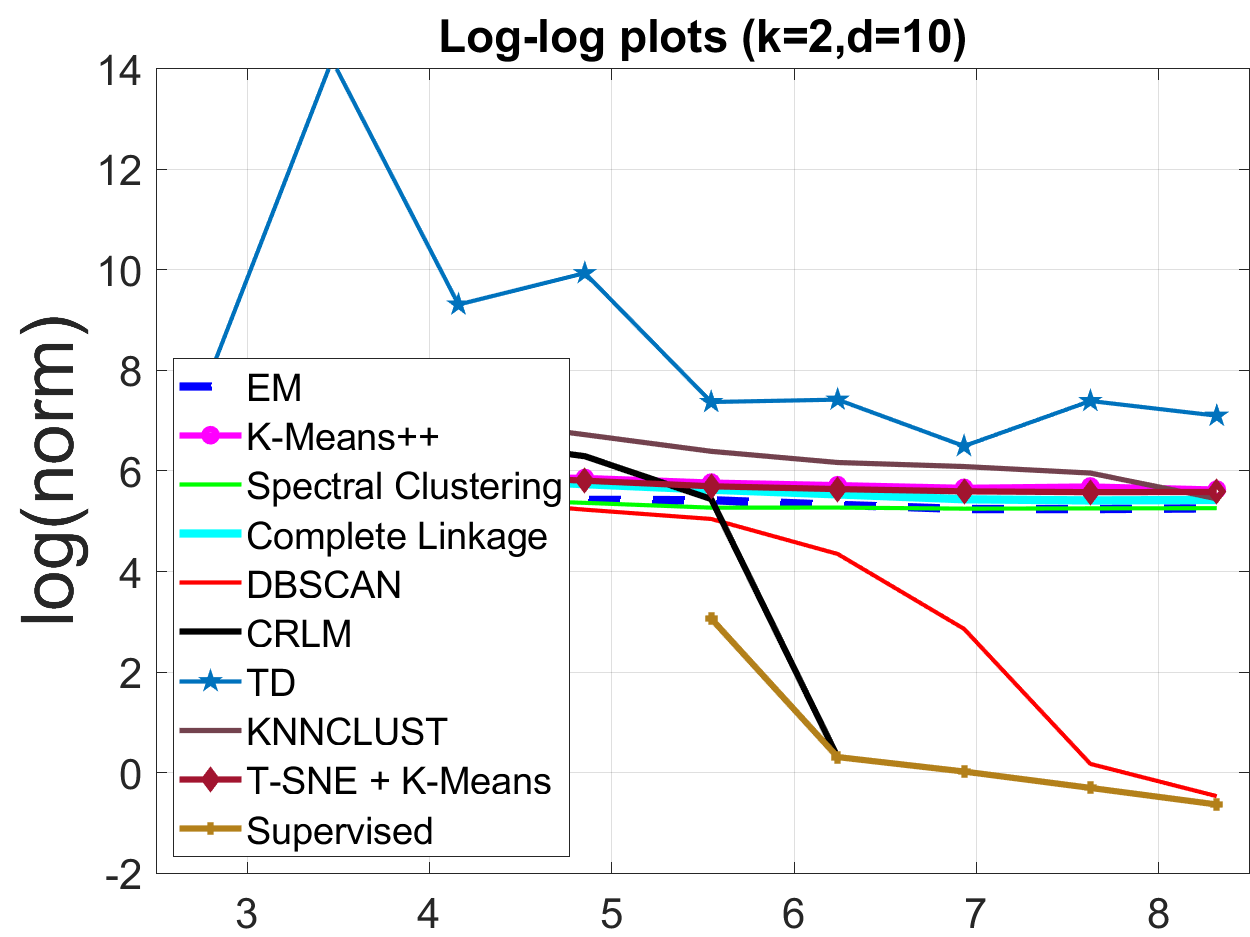}
\includegraphics[height=3.5cm]{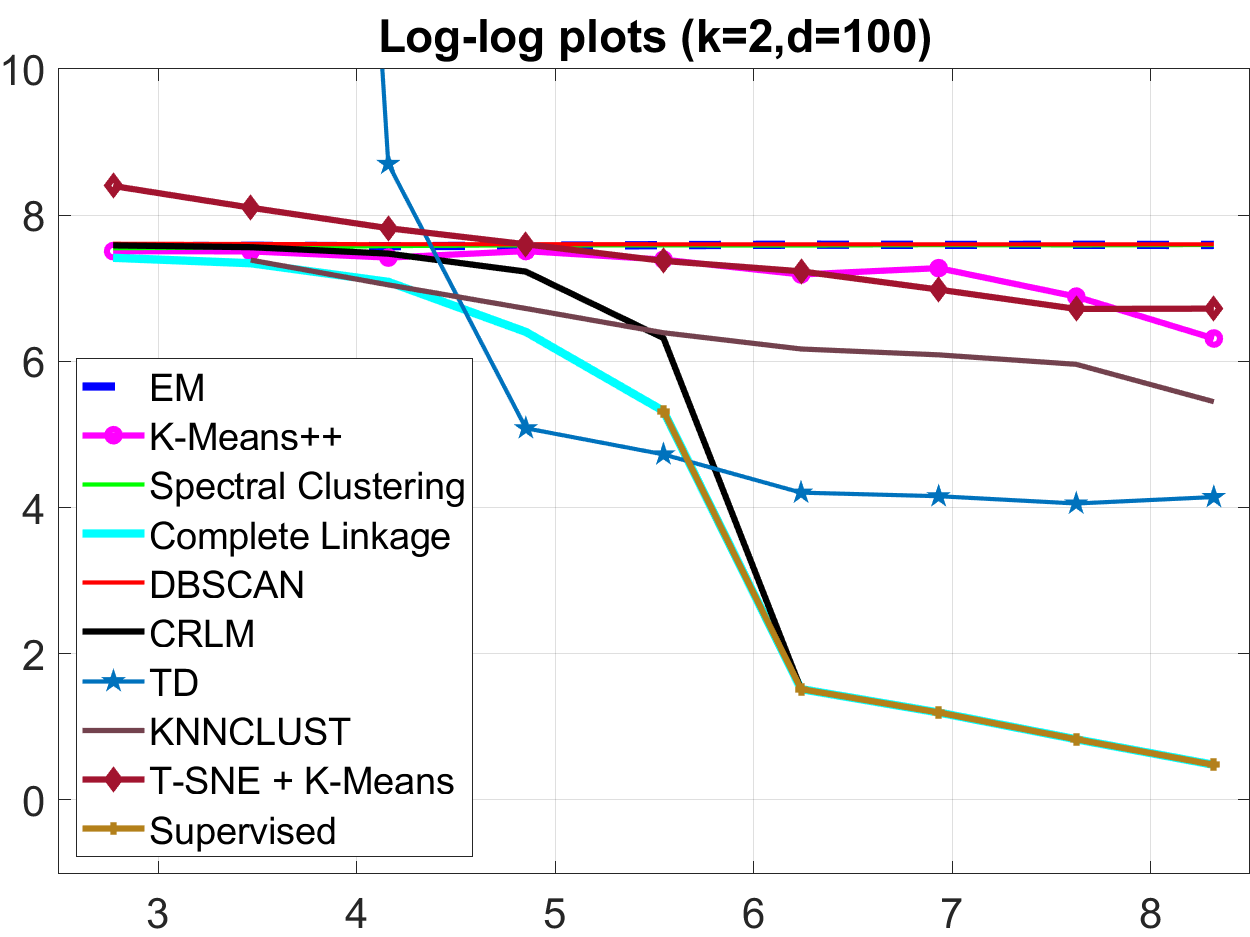}
\includegraphics[height=3.5cm]{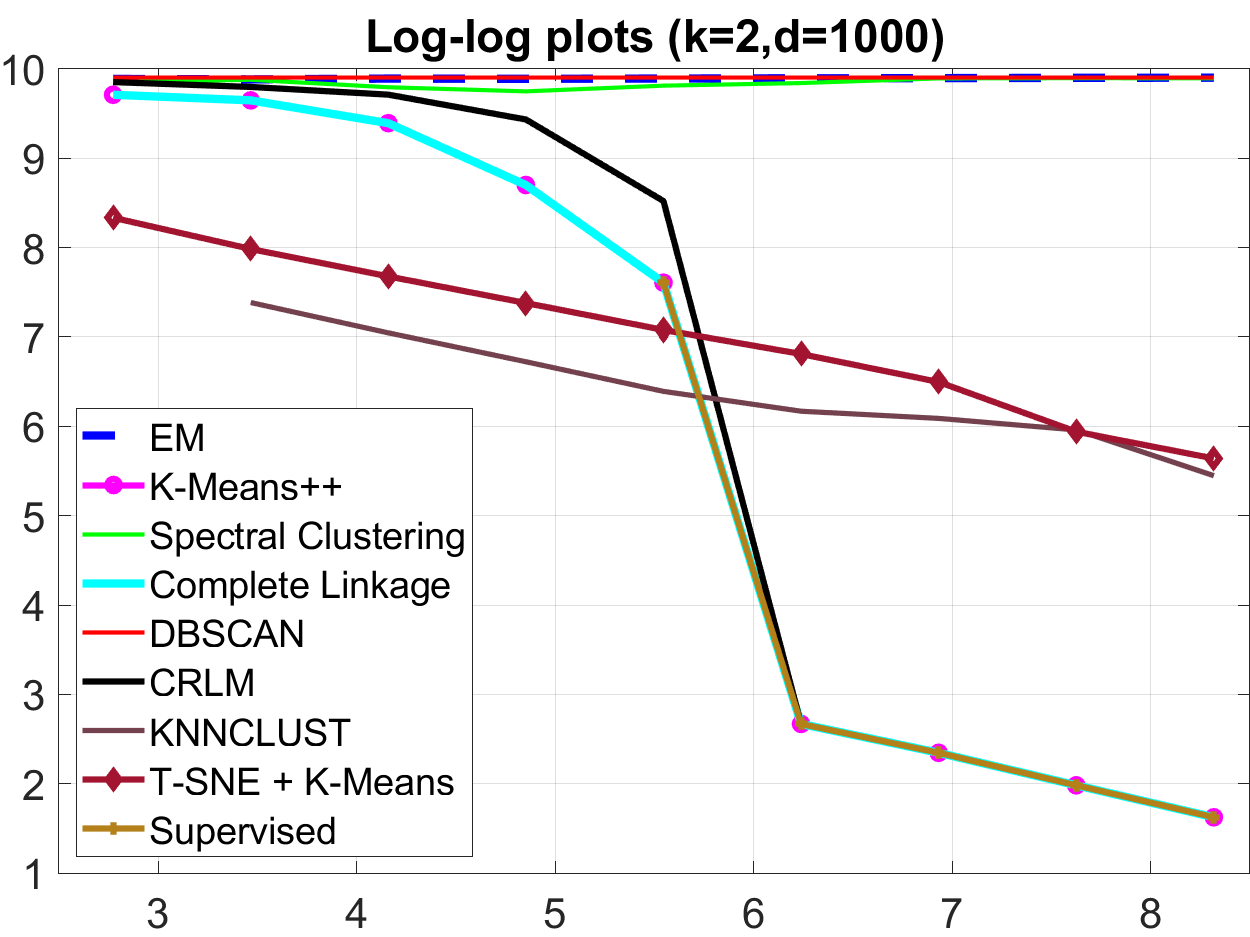}
\includegraphics[height=3.5cm]{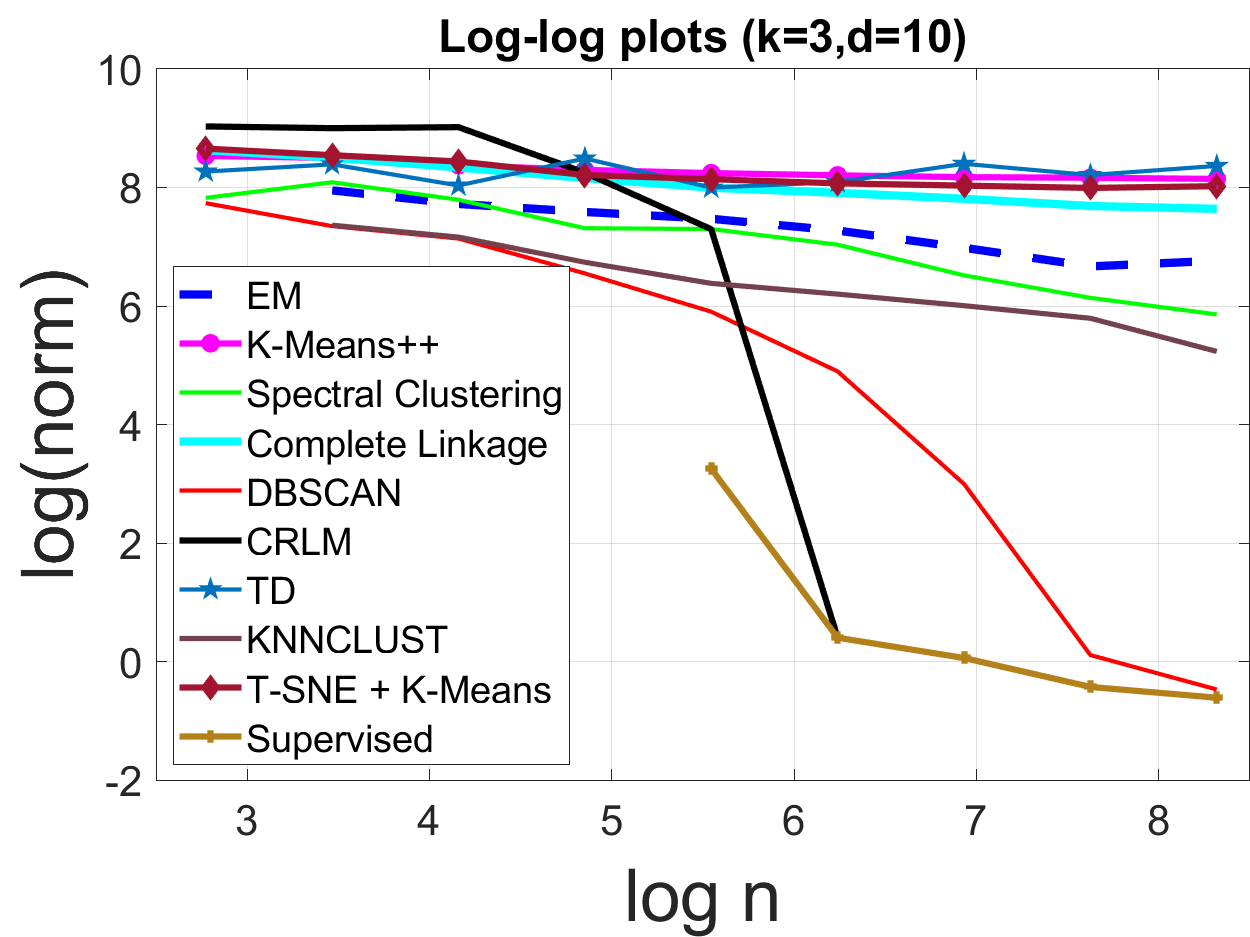}
\includegraphics[height=3.5cm]{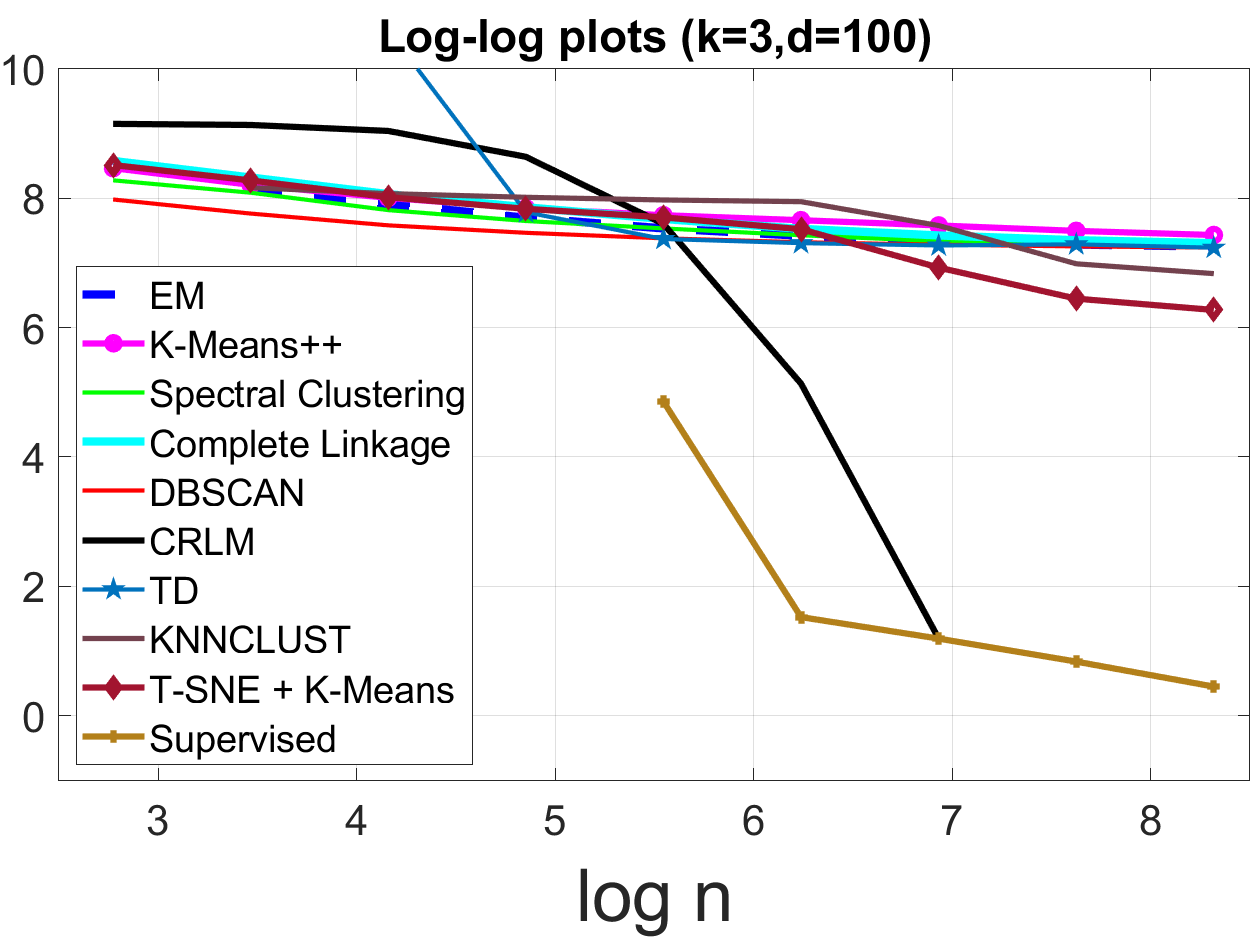}
\includegraphics[height=3.5cm]{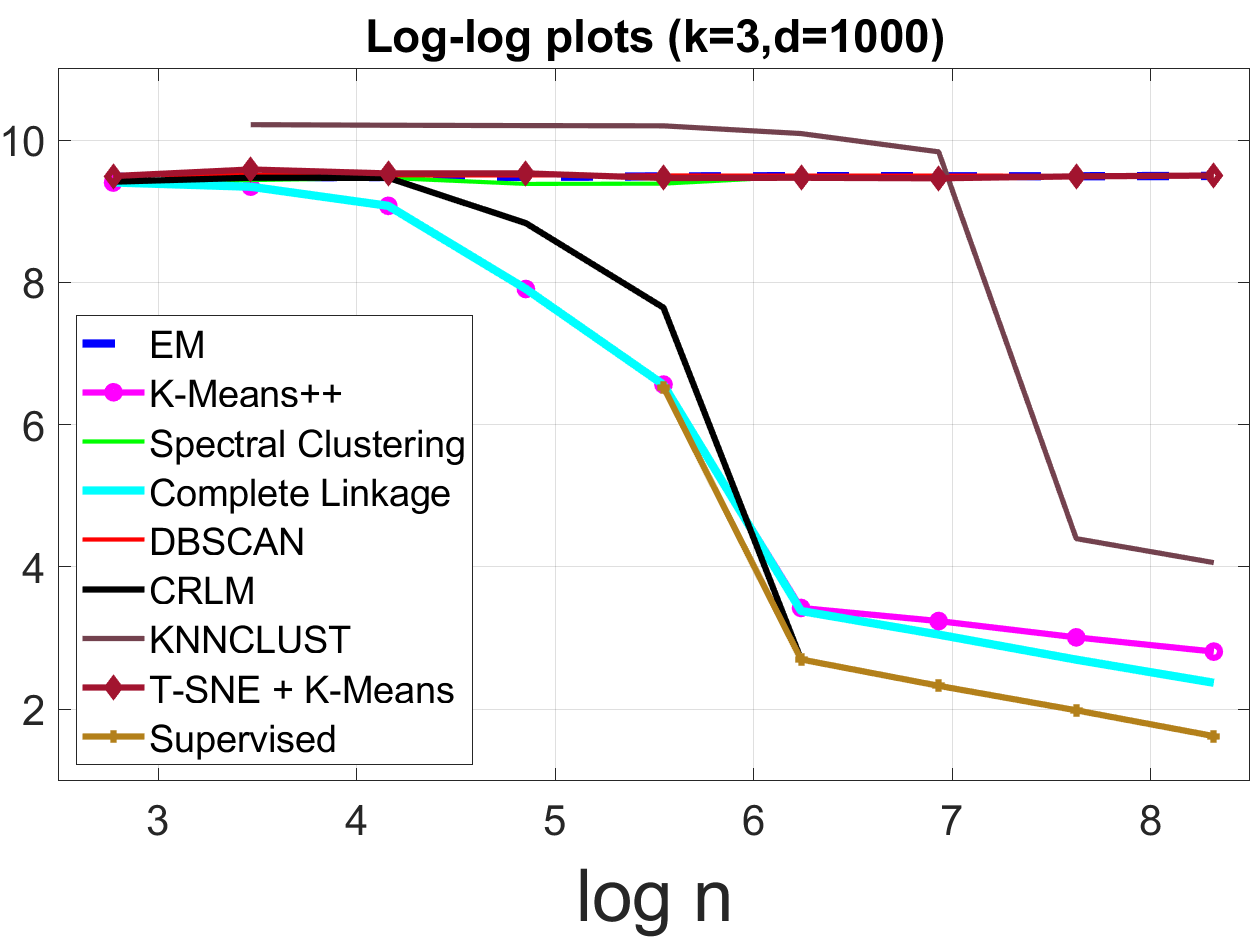}
\vskip -3mm
\caption{Convergence comparison of the proposed algorithm with other popular clustering algorithms.}\label{fig:loglog}
\vspace{-6mm}
\end{figure}
\subsubsection {Convergence Plots}

In  the simulation, we generate our data with the model proposed at the beginning of Section \ref{sec:form} to satisfy the assumptions A1-A3 from Section \ref{sec:assump}. To sample from  a uniform distribution within a $d$-dimension closed ball, we employ a standard method proposed by \cite{muller1959note}.

For the mixture proportion parameter ${\pi_1,\cdots,\pi_k}$ for each positive cluster, in experiments, we set them all equal to $\pi_j=0.01,j=1,...,k$. We take $\sigma_1=1$, $\sigma_2=2$, $\sigma_3=3$ and $\sigma_{\max}=10$. 
For the radius $D\sqrt{d}$ of the uniform ball, based on our previous theoretical probability bounds, we set $D$ large enough to make  all the probability bounds close 
 to 1. 
We generate simulated data with different $d$ and $k$ and make comparison plots with the clustering methods described above, along with our method. 

For Tensor Decomposition, we could only perform experiments on $d=10$ and $d =100$ as this method is computationally expensive for high-dimensional data. 
For $d=1000$, we observed that the whole Tensor Decomposition experiment will take many days to run with the same hardware as the other methods.

To compare the convergence rate of the estimated means obtained by different algorithms to the true means, we record the criterion $(1/k) \sum_{j=1}^{k}||\bmu_j-\hat{\bmu_j}||$, where $\bmu_j$ is the true mean for positive cluster  $j$ and $\hat{\bmu_j}$ is estimated mean obtained by the  clustering algorithm. 
The estimated mean  for  positive cluster $j$ is calculated by taking the average of the data samples clustered with  the same label by the algorithm. Since the labels given by the algorithm might not correspond to the true labels, we search over all permutations $\pi$ of the true labels to obtain the smallest value of the $\sum_{j=1}^{k}||\bmu_{\pi(j)}-\hat{\bmu_j}||$ criterion.
The 'Supervised' results are generated using the true cluster labels. We take the log-log plots with $\log(n)$ as x-axis and $\log((1/k) \sum_{j=1}^{k}||\bmu_j-\hat{\bmu_j}||)$ as y-axis and obtain the  convergence plots shown in Figure \ref{fig:loglog}.

One could see from Figure \ref{fig:loglog} that CRLM always converges to the supervised results based on the true cluster labels, reaching a convergence rate of $O(1/\sqrt{n})$. 
DBSCAN, K-Means and Complete Linkage Clustering also converge to the supervised results in some of the experiments, but not always. Tensor Decomposition and T-SNE never reach the accuracy of the supervised results in these experiments.


\subsubsection {Stability of Clustering with Respect to \texorpdfstring{$\sigma_{\max}$}{Lg}}
Our algorithm has two tuning parameters, the bandwidth $\sigma_{\max}$, and  the estimated number of positive clusters $\hat{k}$.  
In this section, we will discuss the selection of $\sigma_{\max}$ and $\hat{k}$. 

In terms of $\sigma_{\max}$, in order for A1-A3 to hold, we need  $\sigma_{\max}\geq 2\max_{j}{\sigma_{j}}$. In fact, if the data exactly follows the GMMUB structure and $D$ is sufficiently large,  the selection for $\sigma_{\max}$ is very flexible.
 The flexibility increases with the increasing value of $D\sqrt{d}$. 
 
To measure the impact of different values of $\sigma_{\max}$ on the clustering results, we introduce three measures of quality of a clustering result for two or more clusters : Rand Index, F-measure and Purity (\cite{sokolova2009systematic}). 
Among these measures, F-measure is the most relevant measure for our purpose since it can measure the clustering accuracy much better than the Rand Index for an unbalanced dataset. 

However, the F-measure is defined for binary labeled data, while in our setup we have $k\geq 1$ positive clusters and one negative cluster. 
Furthermore, the labels obtained by the clustering algorithm might only correspond to the true labels up to a permutation. 
For these reasons, we define the F-measure as the average of the $k$ F-measures obtained by evaluating the detection of each positive cluster with respect to the rest. 
For that, for any $j\in \{1,...,k\}$ we compute the F-measure $F_j$ obtained by considering the observations with label $j$ as positives and the other observations as negatives. Then we compute the average F-measure as the average of  $\{F_j\}_{j=1}^{k}$.
In a similar way we compute the Rand Index.
\begin{figure*}[ht]
\vspace{-3mm}
\centering\includegraphics[height=4.5cm]{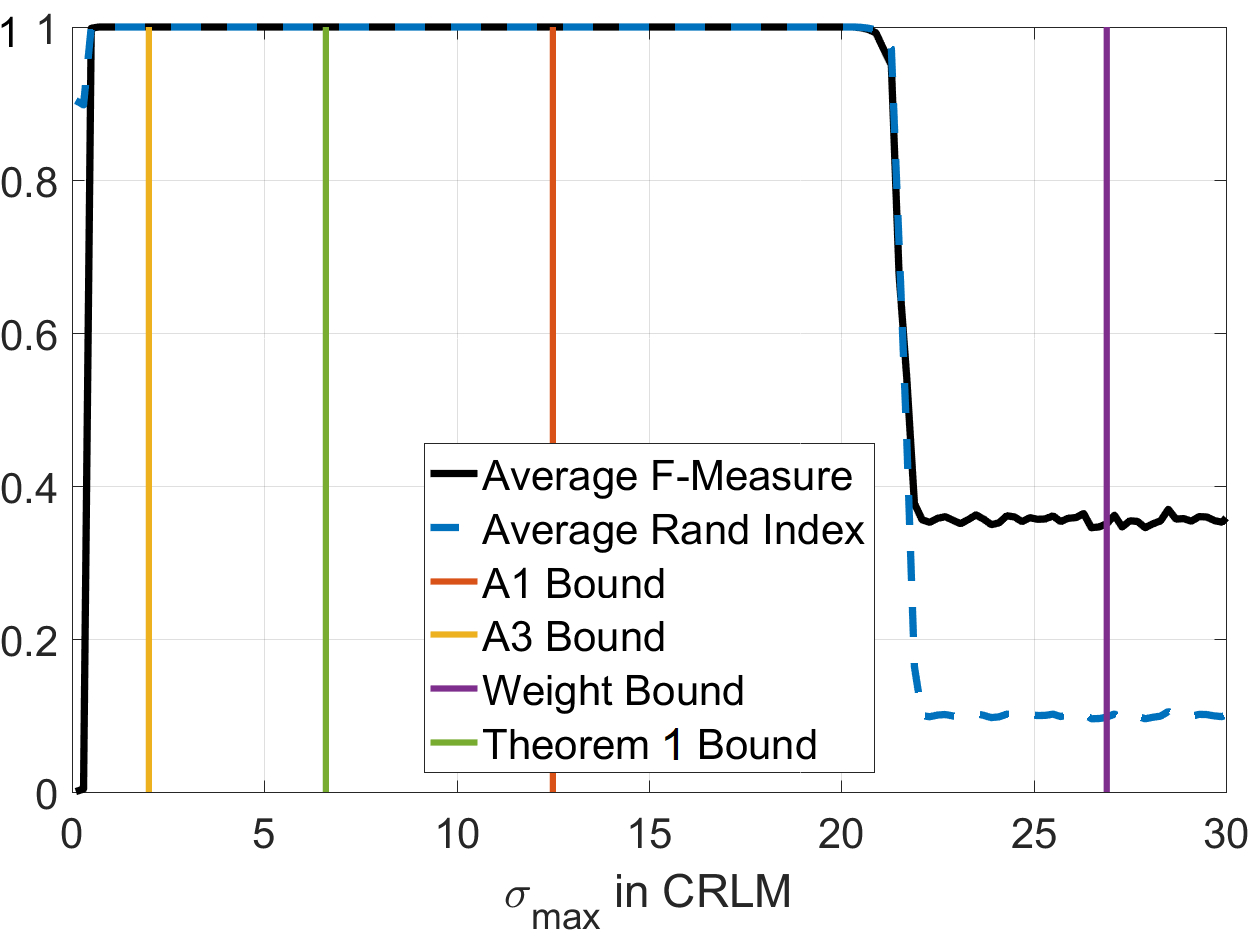}
\vskip -4mm
\caption{Rand Index and F-Measure vs $\sigma_{\max}$, for $k=1,d=20,n=10^4$.}\label{fig:find_sigma_max_d10_D60}
\vspace{-4mm}
\end{figure*}

When $k=1$, we  perform clustering for different values of the dimension $d$, keeping $D=50$ fixed. In Figure \ref{fig:find_sigma_max_d10_D60} is shown the average result of 20 runs when $n=10^4,d=20,\sigma_1=1$. We see that the F-measure is close to 1 for a large range of values of $\sigma_{max}$.

We obtain an experimental upper and lower bound of $\sigma_{max}$ where a F-Measure of at least 0.99 is obtained. 
\begin{figure*}[htb]
\vspace{-2mm}
\centering
\includegraphics[height=4.5cm]{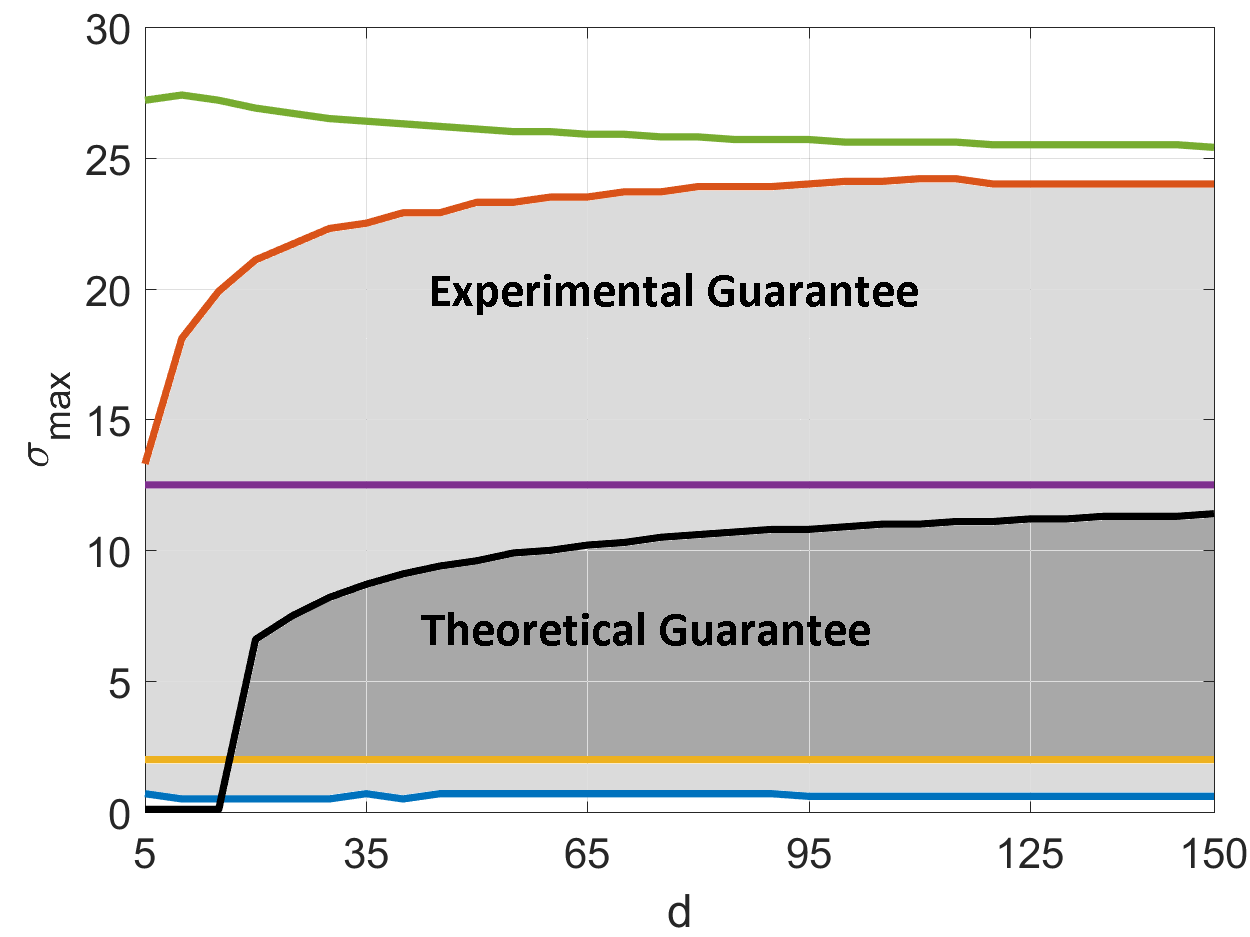}
\includegraphics[height=4.5cm]{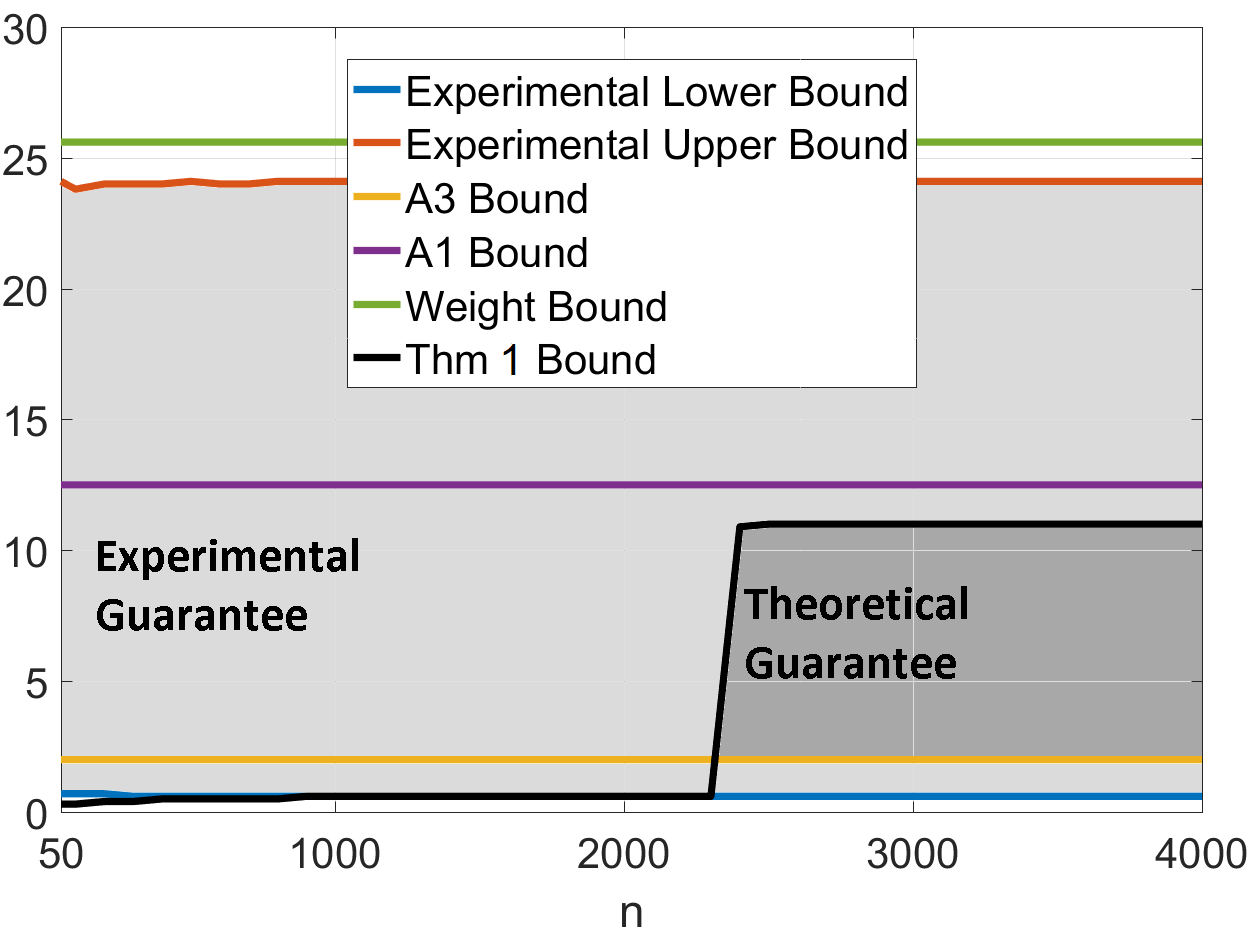}
\includegraphics[height=4.5cm]{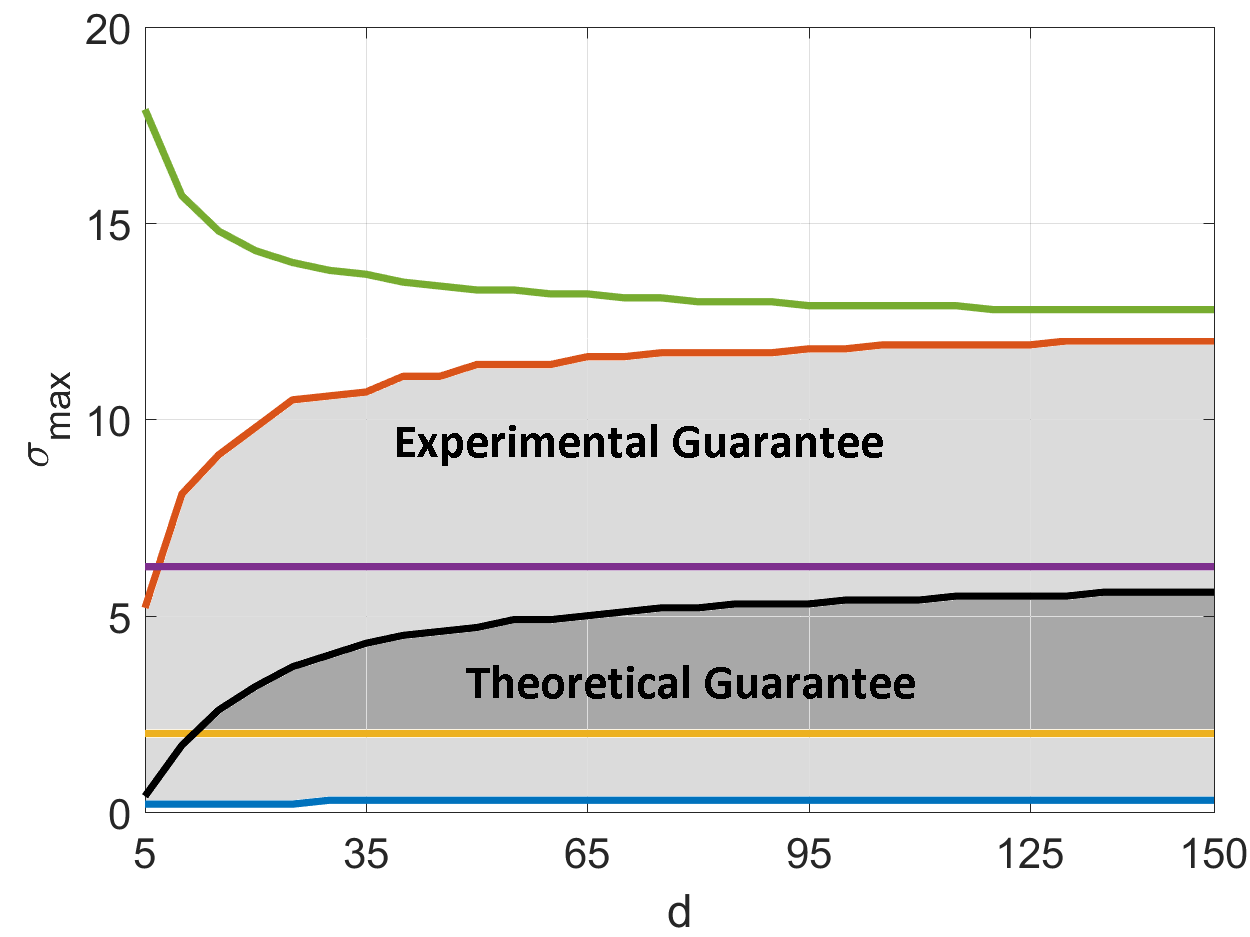}
\includegraphics[height=4.5cm]{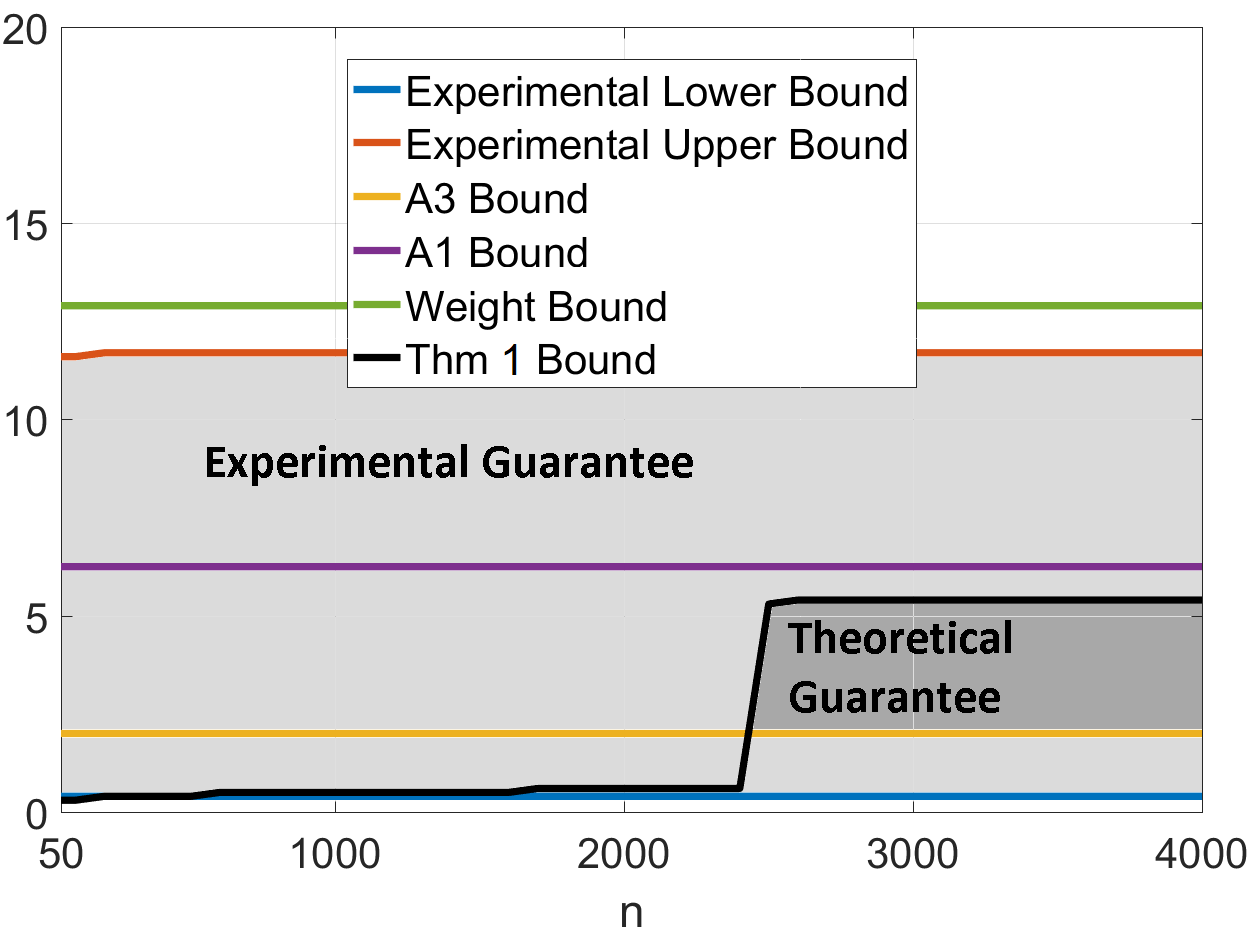}
\vskip -4mm
\caption{Comparison between the theoretical bounds and the region of experimental success of OCRLM ($k=1$) vs. space dimension $d$ (left) and number of observations $n$ (right), when $n=10^{4}$ (left) and $d=100$ (right). First row: $G=4$, second row: $G=16$.}\label{fig:find_sigma_maxd10_150}
\vspace{-4mm}
\end{figure*}

Shown in Figure \ref{fig:find_sigma_maxd10_150} as lighter gray and darker gray areas are such experimental bounds and respectively the theoretical bounds, for different dimensions $d$ (left) and sample sizes $n$ (right).

For the theoretical bounds, the lower bound is from assumption A3 and the three upper bounds are from A1, the weight condition \eqref{eq:wtcond1} from Theorem \ref{thm:onecluster}, and the condition that the probability \eqref{eq:probbnd1} from Theorem \ref{thm:onecluster} is at least $0.99$. These theoretical bounds are labeled in Figure \ref{fig:find_sigma_maxd10_150} as the A3 bound, A1 bound, weight bound, and Thm \ref{thm:onecluster} bound respectively.  
 Our theoretical guarantee for $\sigma_{\max}$ is the area between the A3 bound curve and the Theorem \ref{thm:onecluster} bound, shown as a darker gray area in Figure \ref{fig:find_sigma_maxd10_150}.
 
The theoretically guaranteed region is a sub region of the lighter gray region that contains the practical choices of the $\sigma_{\max}$ for which an F-measure of at least 0.99 was obtained.
 
Figure \ref{fig:find_sigma_maxd10_150} also shows a lower bound for $d$. With increasing $d$, the range for $\sigma_{\max}$ becomes larger. 
For the impact of $n$, the area becomes stable when $n$ is large enough.
 However, the experimental bounds shows that CRLM can work with a more flexible choice of   $\sigma_{\max}$ and for a larger range of $d$ and $n$.
\begin{figure*}[t]
\vspace{-3mm}
\centering
\includegraphics[height=4.4cm]{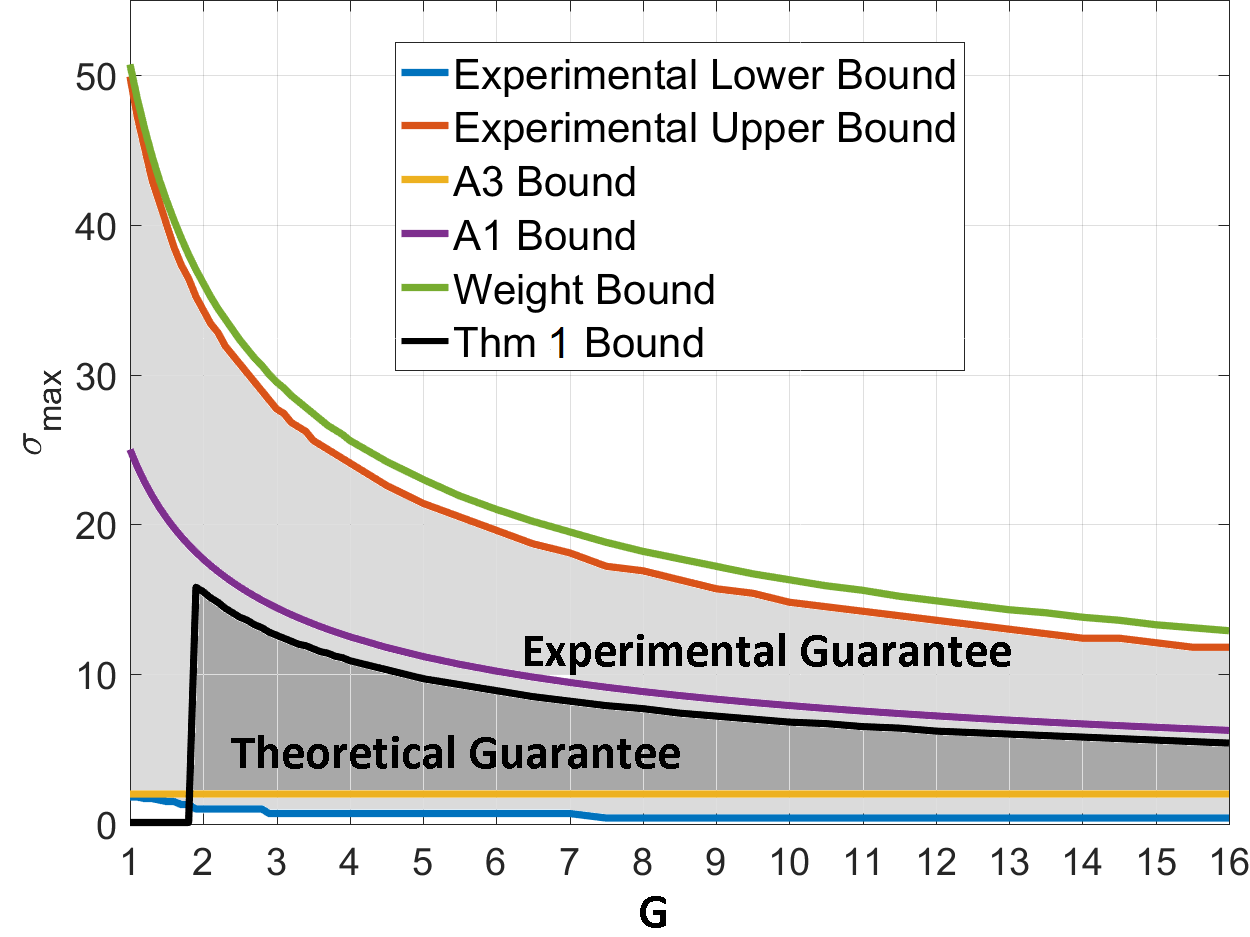}
\vskip -4mm
\caption{Comparison between the theoretical bounds and the region of experimental success of OCRLM ($k=1$) vs. $G$, when $n=10^{4},D=80,d=100$.}\label{fig:find_sigma_max_F}
\vspace{-5mm}
\end{figure*} 

In Figure \ref{fig:find_sigma_max_F} are shown the same bounds for different values of the $G$ parameter, for fixed $d=100$ and $n=10^4$.
For the Thm \ref{thm:onecluster} bound, $G$ needs to be large enough to make the probability close to 1 for a certain $\sigma_{\max}$. 
With large enough $G$, the theoretical region and experimental region will decrease as $G$ increases.

\begin{figure*}[ht]
\vspace{-3mm}
\centering
\includegraphics[height=4.4cm]{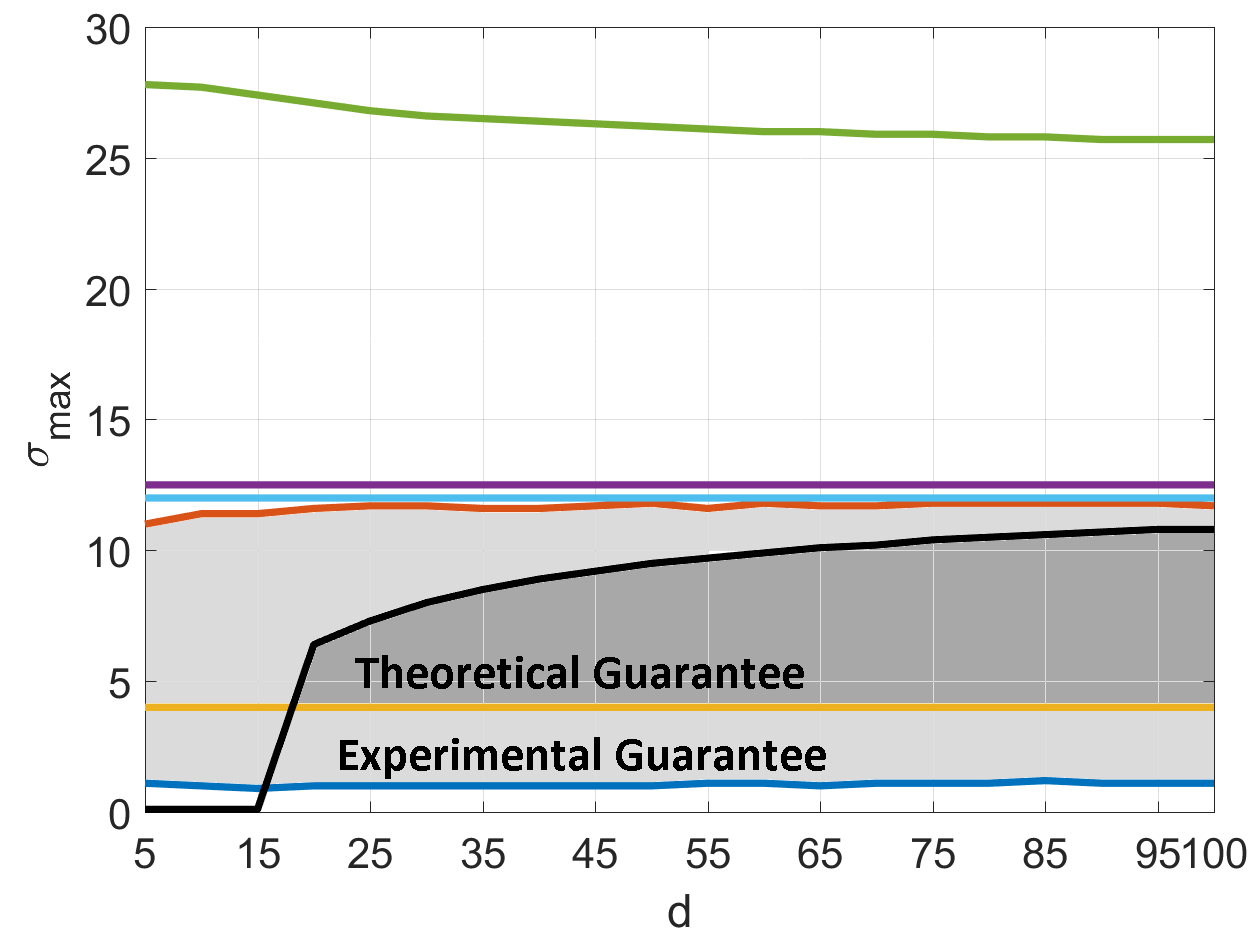}
\includegraphics[height=4.4cm]{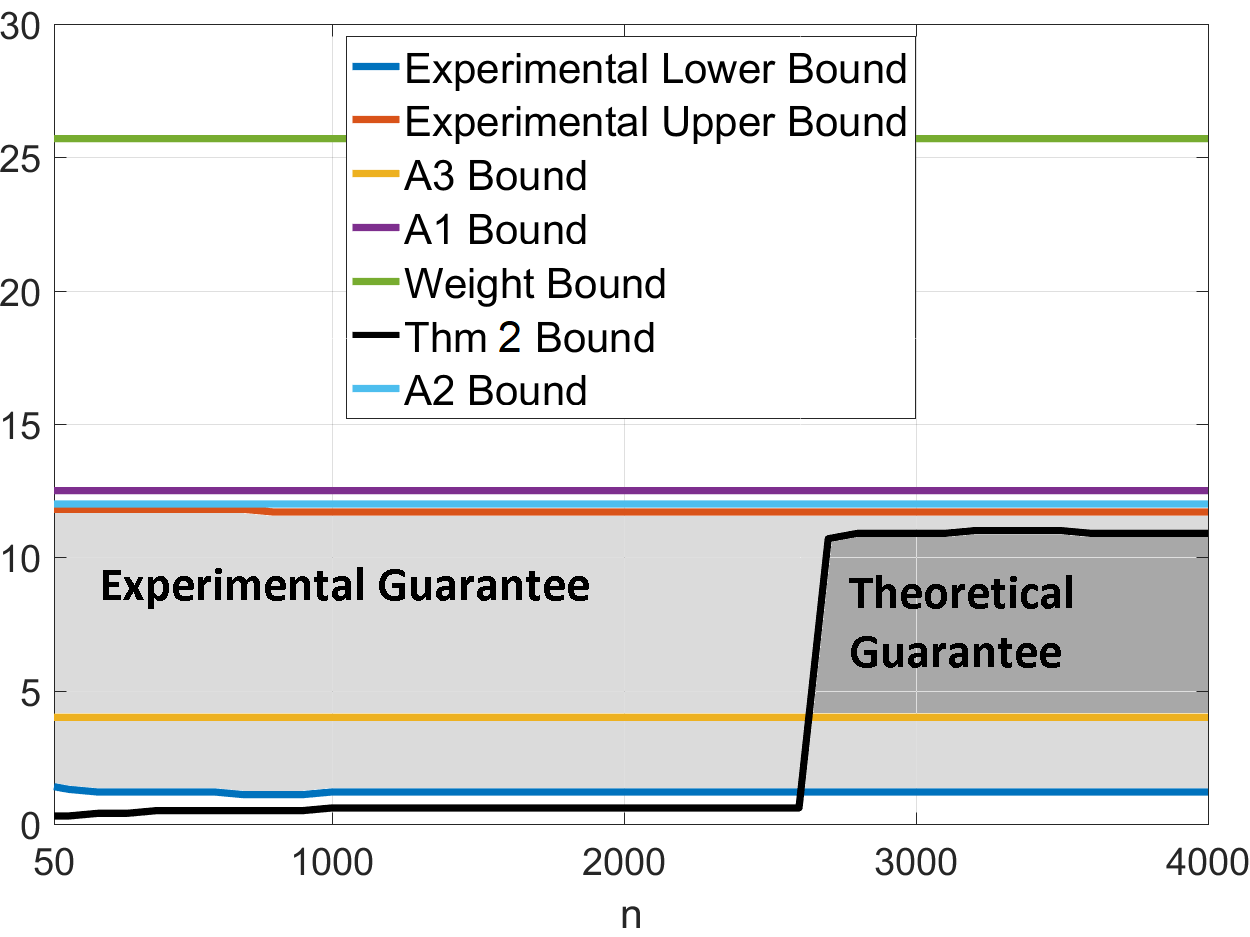}
\vskip -3mm
\caption{Comparison between the theoretical bounds and the region of experimental success of CRLM vs. space dimension $d$ (left) and number of observations $n$ (right), when $k=2,G=4$ and $n=10^{4}$ (left) and $d=100$(right).}\label{fig:find_sigma_maxd10_100_k2}
\vspace{-3mm}
\end{figure*}When the number of clusters $k>1$, the results are similar. 
In Figure \ref{fig:find_sigma_maxd10_100_k2} is shown the case when $k=2, \pi = [0.1,0.1,0.8] , \sigma_1=1$, $\sigma_2=2$, $D=50$. 
 Adding the A2 bound and replacing Thm \ref{thm:onecluster} bound by Thm \ref{thm:multclust} bound are the major differences from the case of one positive cluster. Also, we  can see that experimental upper bound is mainly bounded by the A2 bound.

\begin{figure*}[ht]

\centering
\includegraphics[height=4.5cm]{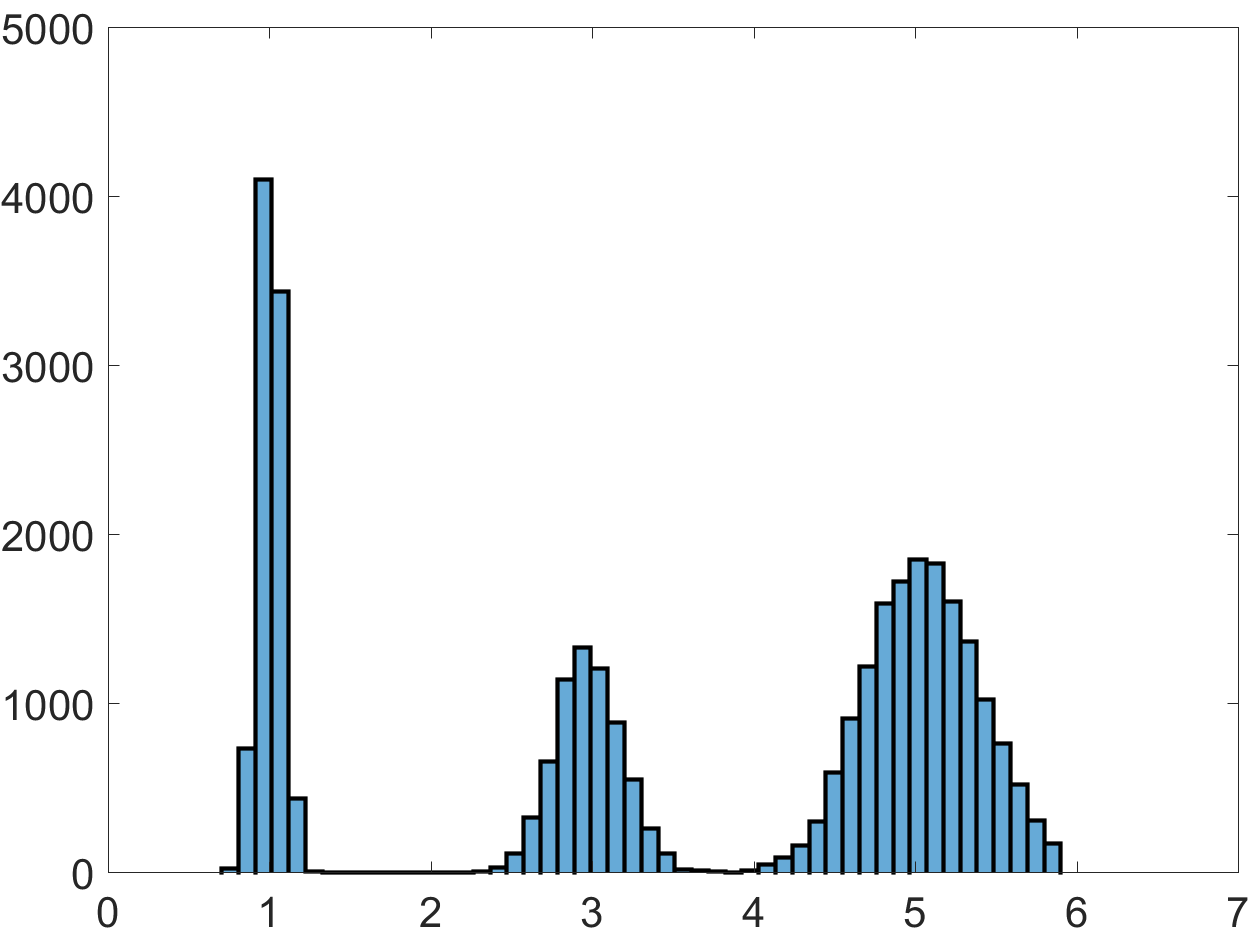}
\includegraphics[height=4.5cm]{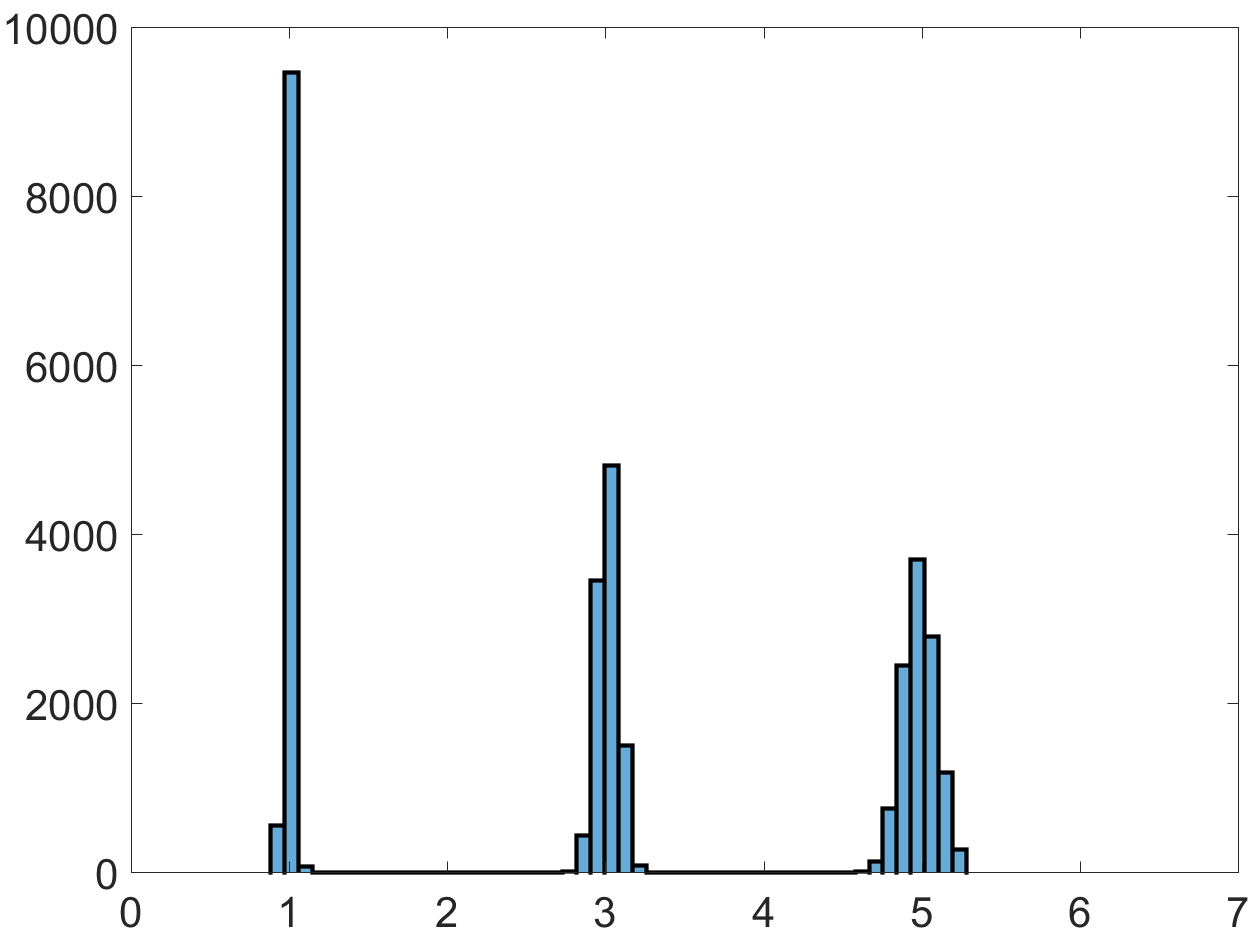}

\caption{Histogram plot of the first 5 $\%$ shortest distances between the observations divided by $\sqrt{2d}$ (Left: $d=100$, Right: $d=1000$) }\label{find_sigma2}
\end{figure*}To choose an appropriate $\sigma_{max}$ in practice, we need to roughly estimate $\sigma$ for each positive cluster and then choose a $\sigma_{max}$ just slightly larger than twice the largest estimated $\sigma$. We propose a novel way to estimate the $\sigma$ for the positive clusters, using the histogram of pairwise distances between observations.

 Here, we simulated a GMMUB with $k=3,d=100$ or $d=1000$, $p_1 = p_2=p_3 =0.01$, $\sigma =[1 ,3 ,5]$, with C1-C2 and A1-A3 satisfied. 
We calculate the the distances between the observations and compute the histogram of the 5\% shortest distances, divided by $\sqrt{2d}$. This histogram is shown in Figure \ref{find_sigma2}.
One can see that clear peaks are formed around the true standard deviations $1,3,5$. This is because the pairwise distances between samples from a Gaussian $\N(\bmu,\sigma^2I_d)$ are actually norms of samples from a Gaussian $\N(0,2\sigma^2I_d)$, so they cluster around $\sigma \sqrt{2d}$, as it was illustrated in Figure \ref{fig:gaussian_norm}.

\subsubsection {Estimating the number of clusters \texorpdfstring{$k$}{Lg}}
Some relevant methods for estimating of the number of positive clusters $k$ include the elbow methods (\cite{thorndike1953belongs}), X-means (\cite{pelleg2000x}) and the silhouette method (\cite{lleti2004selecting}). 
However, a simple way to estimate the number of clusters can be derived naturally and directly from our algorithm. 

Suppose  $D$ is sufficiently large, then for a well chosen parameter $\sigma_{\max}$, we can run CRLM for a large number $K> k$ of iterations and record the number of observations in each cluster. 
The algorithm is stopped when the number of observations in the new clusters becomes $1$. If this happens at iteration $i$, then the estimated number of positive clusters is $i-1$.
\begin{figure*}[htb]
\vspace{-3mm}
\centering
\includegraphics[height=4.5cm]{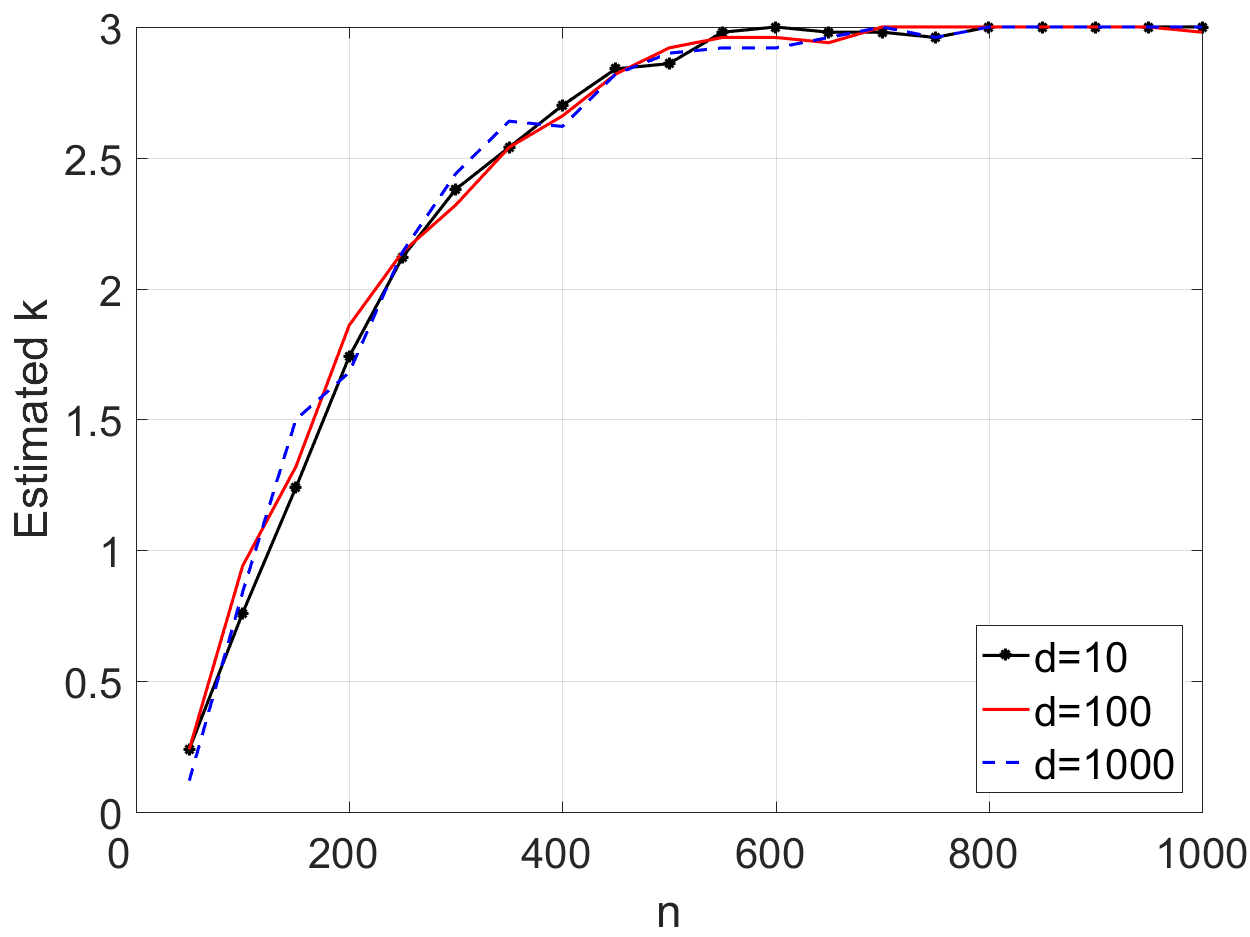}
\vskip -3mm
\caption{Estimated $k$ versus $n$, for $k =3, D=1000, G=4$) }\label{fig:find_k}
\vspace{-4mm}
\end{figure*}

We conduct experiments on synthetic data with $k=3$ to see how well this method finds the correct number of clusters. In Figure \ref{fig:find_k} is shown the average value of the estimated $k$ obtained this way vs the number $n$ of observations. This  average is obtained from 20 independent runs.

It is clear that this method for estimating the number of clusters $k$ is efficient and it converges to actual value of $k$ if $n$ i slarge enough. 
Moreover, the dimension $d$ does not have much of an impact on number of observations needed to find a good estimate for $k$. 
The results from these experiments are based on conditions C1-C2 being satisfied and $D$ taking sufficiently large values. 
When either of these two conditions is violated, the minimum cluster size for stopping the CRLM procedure should be adjusted to a positive integer greater than or equal to 1. 

\subsection{Real Data}

To show the potential of our CRLM algorithm, we also conduct some experiments on real datasets from two different sources. 
The original datasets are sets of images from different classes. We employ clustering with various clustering algorithms on these images and measure the performance by F-measure and Rand Index based on the true labels. 

\subsubsection{Kimia 216 Dataset And 1070  Shape Database}

The Kimia 216 \cite{sebastian2004recognition} contains 18 classes each consisting of 12  black and white  binary shape images.  
It contains shapes silhouettes for birds, bones, brick, camels, car, children, classic cards, elephants, faces, forks, fountains, glasses, hammers, hearts, keys, rays, turtles and a miscellaneous class.
 Most of images of the Kimia 216 datasets are in the 1070 Shape Database.
In Figure \ref{fig:Kimia dataset} are shown all the images of the Kimia 216 dataset.
 The datasets can be downloaded at \url{http://vision.lems.brown.edu/content/available-software-and-databases}. 
 \begin{figure*}[t]
 \vspace{4mm}
\centering
\includegraphics[width=12cm]{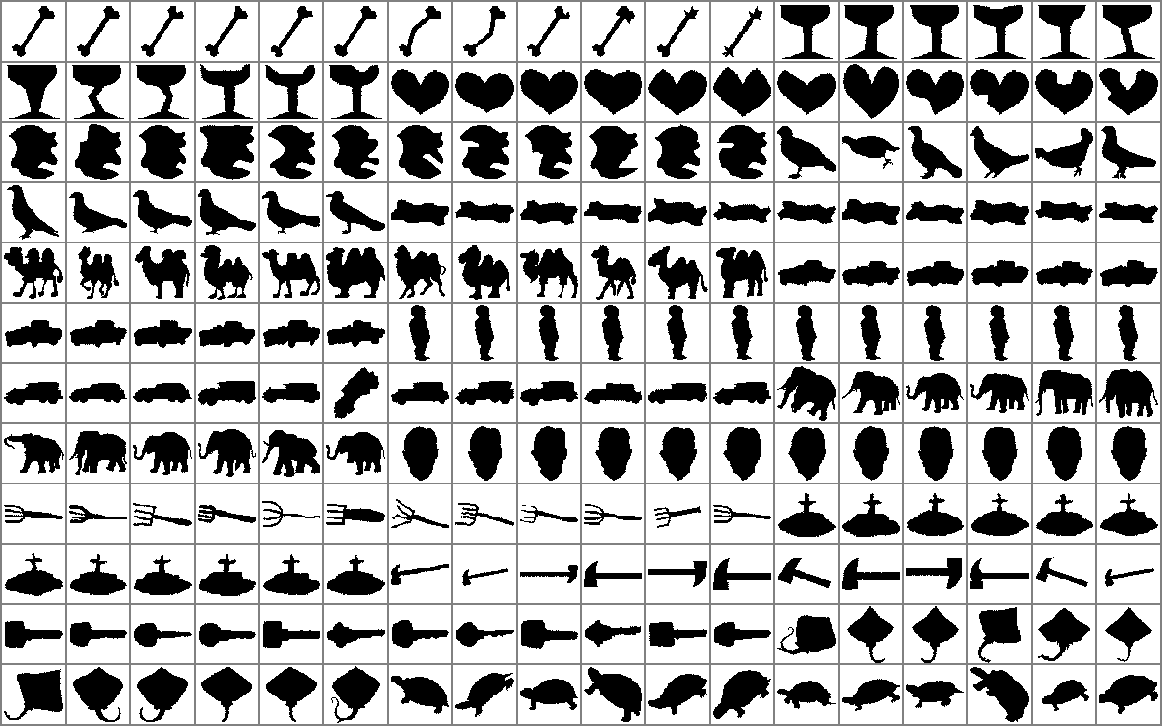}

\caption{The images in Kimia 216 shape dataset (\cite{sebastian2004recognition})}\label{fig:Kimia dataset}
\vspace{-6mm}
\end{figure*}

\noindent{\bf Data preprocessing.} The images are resized to 256$\times$256 pixels and vectorized. 
After that, we perform PCA  and use the 215 PC coefficients as the input for different clustering methods. 

\noindent{\bf Kimia results.} Because the Kimia 216 dataset has 18 classes, it can be fitted as a GMM model with 18 Gaussian clusters.  In Table \ref{Kimia216} are shown the clustering results measured as Rand Index, since the classes are balanced (12 observations each).

\begin{table}[h]
\vspace{-3mm}
\centering
\scalebox{0.95}{
\begin{tabular}{|c|c|c|c|c|c|c|c|c|c|} 
\hline
 Methods & K-means & DBSCAN &CL&TD & EM &SC&T-SNE &KC&Our (CLRM) \\ \hline 
  Rand Index (\%) &67.99 & 69.91 & 60.19& 28.24 & 18.06 & 68.98&68.06 &73.49& {\bf 73.61} \\ 
\hline
\end{tabular}}
\caption{Accuracy of different clustering algorithms on the Kimia 216 dataset}\label{Kimia216}
\vspace{-3mm}
\end{table}

From Table \ref{Kimia216} one can see that both our method and KNNCUST rank first and second, followed by DBSCAN and k-means. However, one could see that all clustering results are far from being acceptable. 
It is possible that GMM might not be a good model for the Kimia 216 dataset, and the similarity between observations from each group may not be accurately measured simply by distance or density. 

\begin{figure*}[t]
\vspace{-4mm}
\centering
\includegraphics[width=0.07\textwidth]{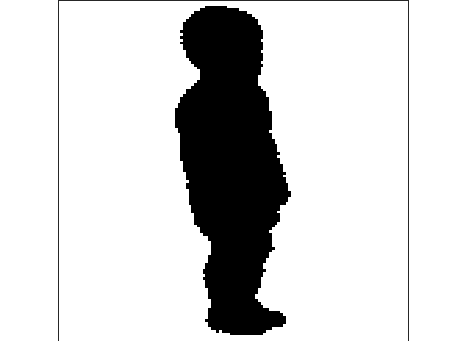}
\includegraphics[width=0.07\textwidth]{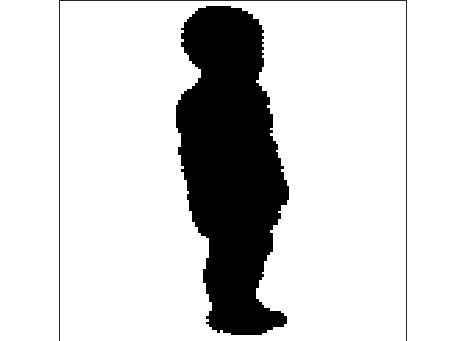}
\includegraphics[width=0.07\textwidth]{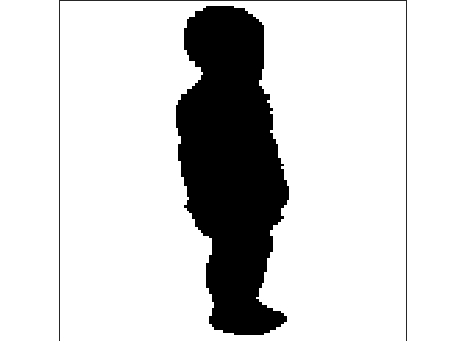}
\includegraphics[width=0.07\textwidth]{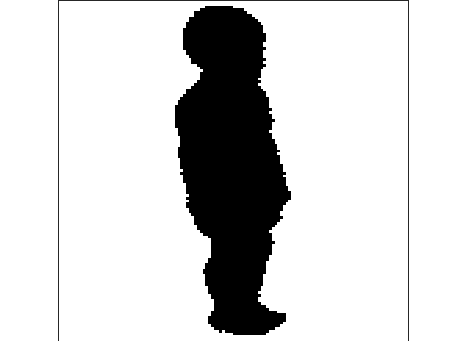}
\includegraphics[width=0.07\textwidth]{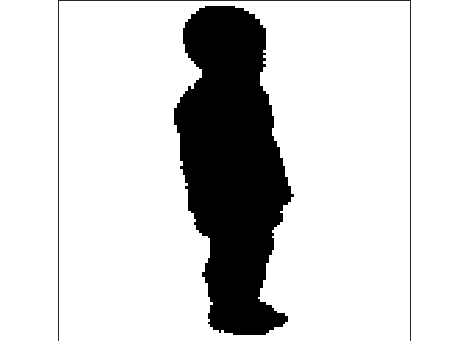}
\includegraphics[width=0.07\textwidth]{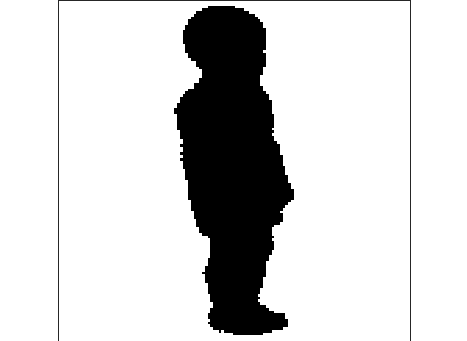}
\includegraphics[width=0.07\textwidth]{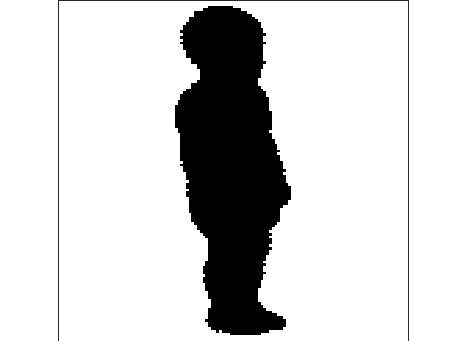}
\includegraphics[width=0.07\textwidth]{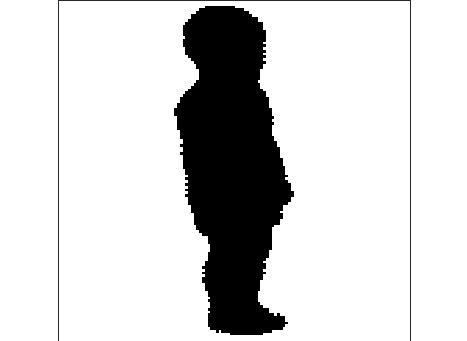}
\includegraphics[width=0.07\textwidth]{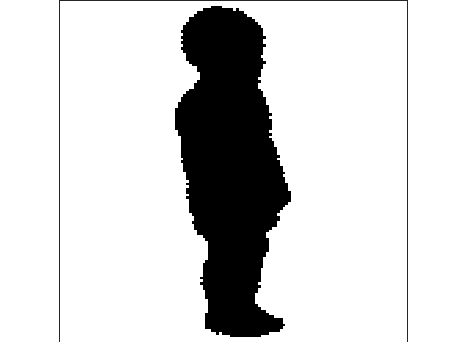}
\includegraphics[width=0.07\textwidth]{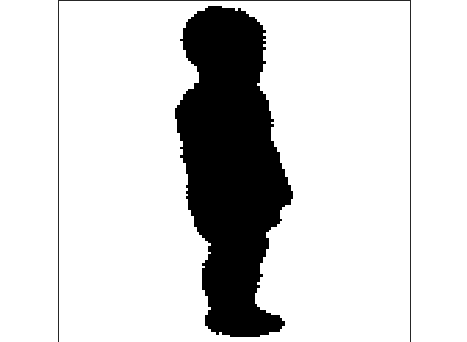}
\includegraphics[width=0.07\textwidth]{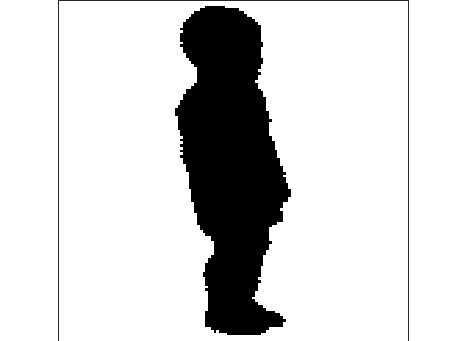}
\includegraphics[width=0.07\textwidth]{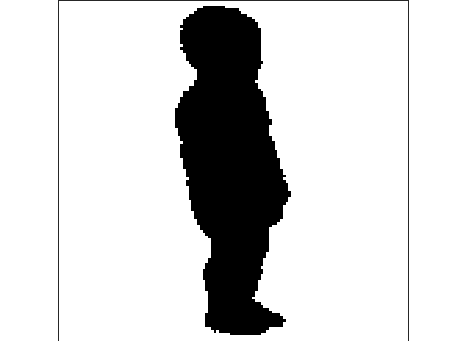}
\vskip 2mm
\includegraphics[width=0.07\textwidth]{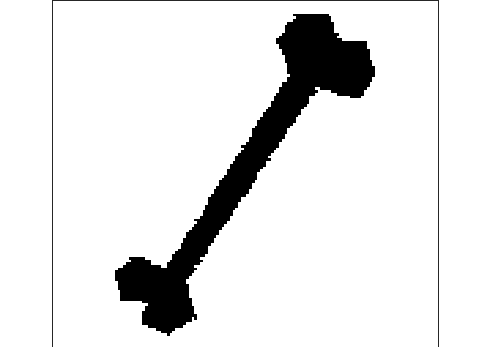}
\includegraphics[width=0.07\textwidth]{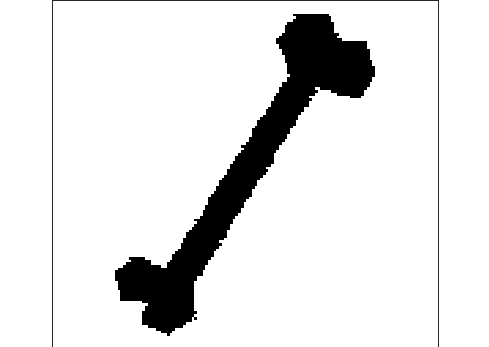}
\includegraphics[width=0.07\textwidth]{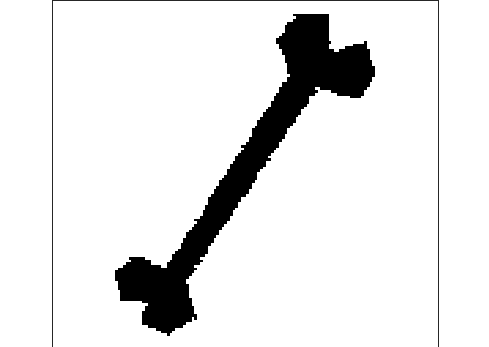}
\includegraphics[width=0.07\textwidth]{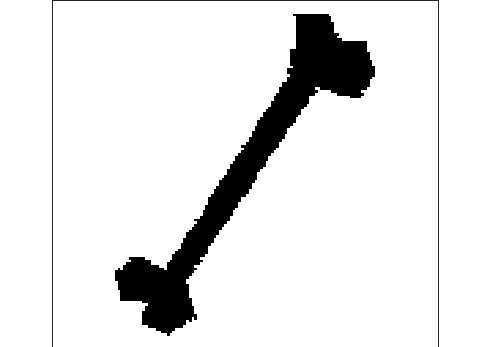}
\includegraphics[width=0.07\textwidth]{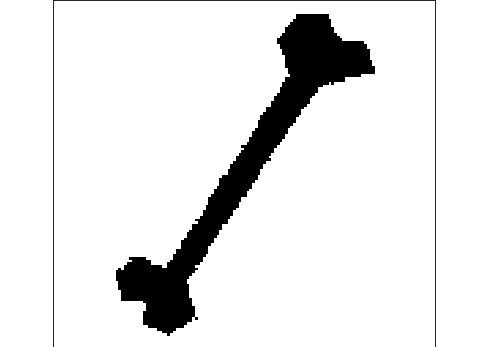}
\includegraphics[width=0.07\textwidth]{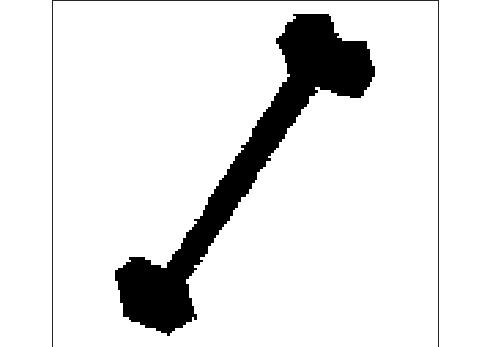}
\includegraphics[width=0.07\textwidth]{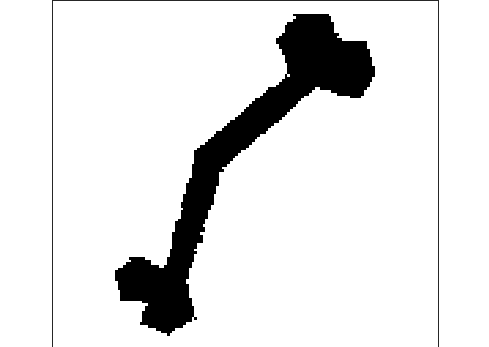}
\includegraphics[width=0.07\textwidth]{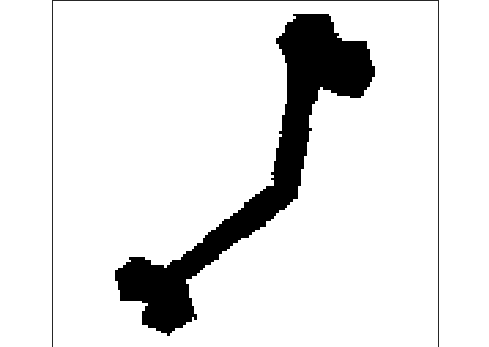}
\includegraphics[width=0.07\textwidth]{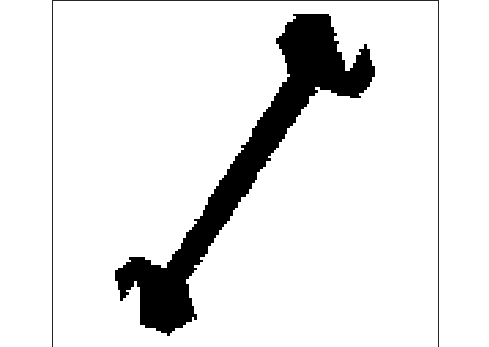}
\includegraphics[width=0.07\textwidth]{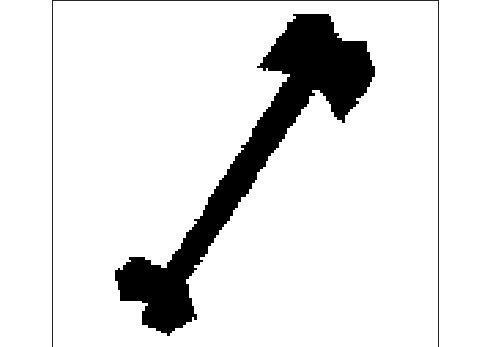}
\includegraphics[width=0.07\textwidth]{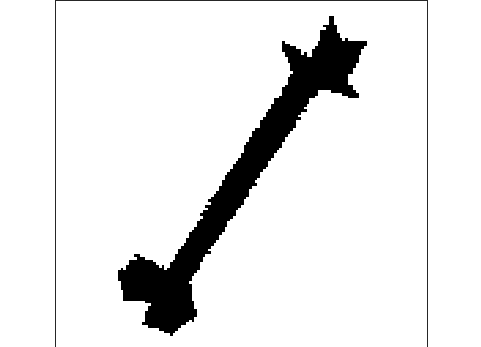}
\includegraphics[width=0.07\textwidth]{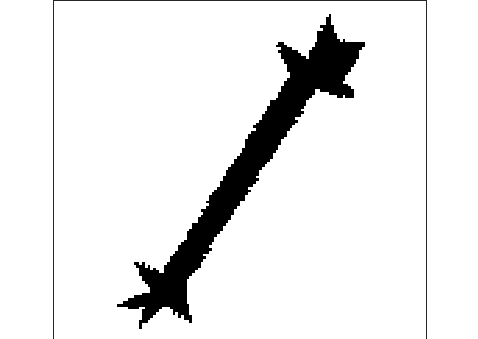}
\vskip 2mm
\includegraphics[width=0.064\textwidth]{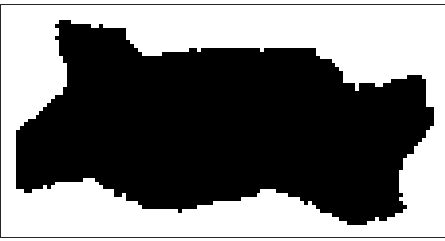}
\includegraphics[width=0.064\textwidth]{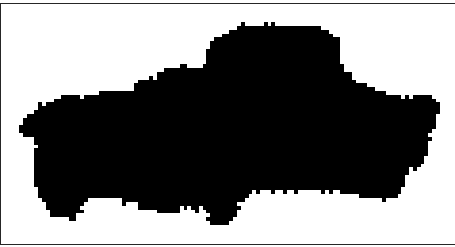}
\includegraphics[width=0.064\textwidth]{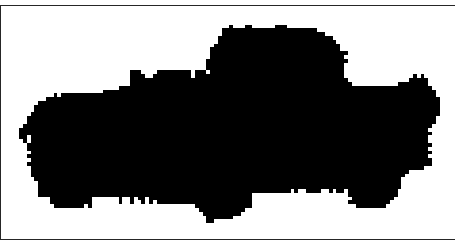}
\includegraphics[width=0.064\textwidth]{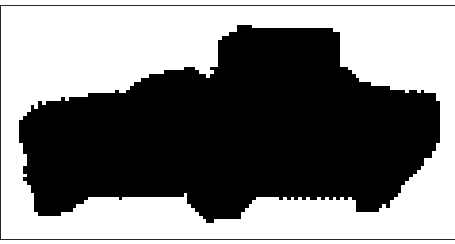}
\includegraphics[width=0.064\textwidth]{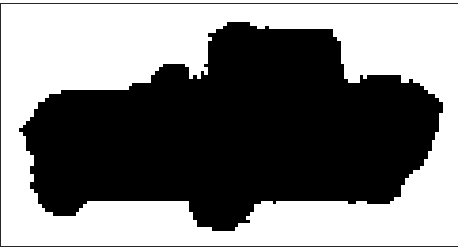}
\includegraphics[width=0.064\textwidth]{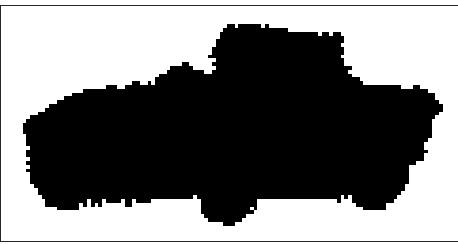}
\includegraphics[width=0.064\textwidth]{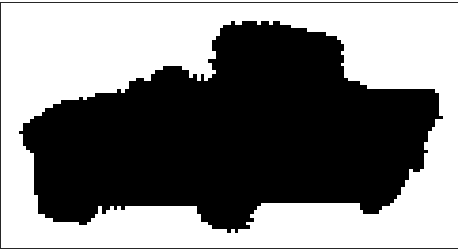}
\includegraphics[width=0.064\textwidth]{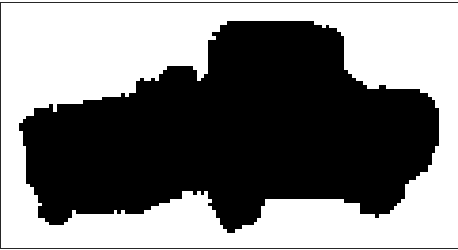}
\includegraphics[width=0.064\textwidth]{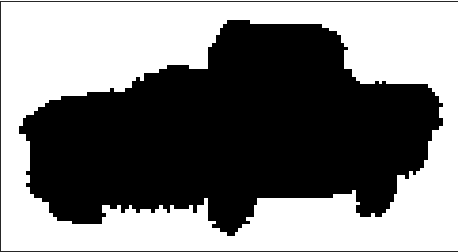}
\includegraphics[width=0.064\textwidth]{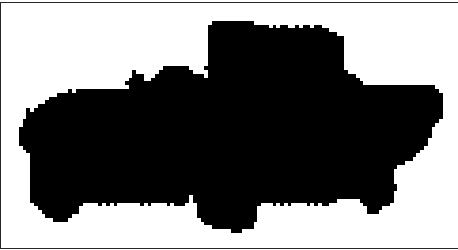}
\includegraphics[width=0.064\textwidth]{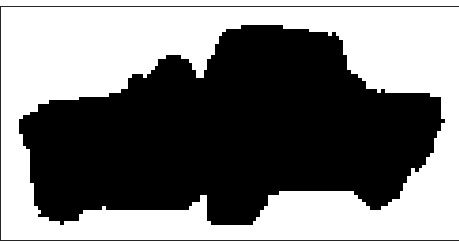}
\includegraphics[width=0.064\textwidth]{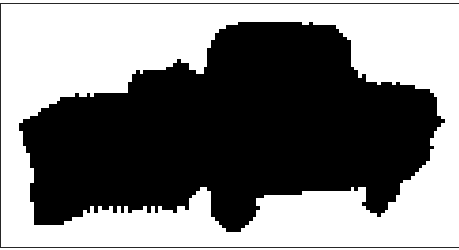}
\includegraphics[width=0.064\textwidth]{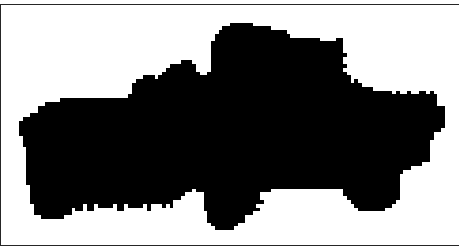}
\vskip 2mm
\includegraphics[width=0.07\textwidth]{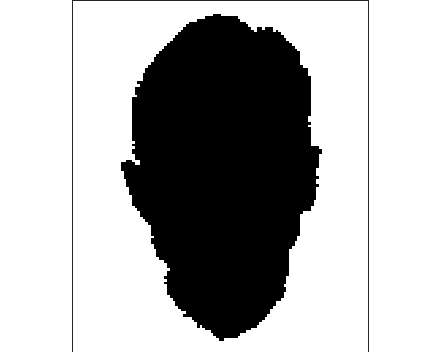}
\includegraphics[width=0.07\textwidth]{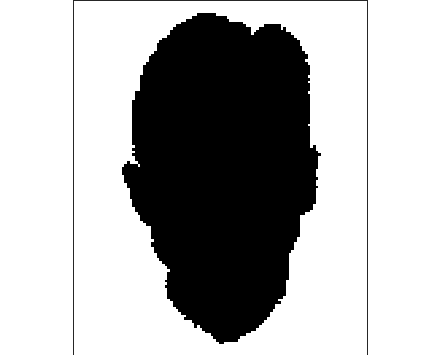}
\includegraphics[width=0.07\textwidth]{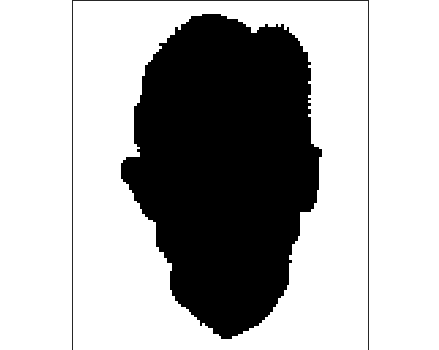}
\includegraphics[width=0.07\textwidth]{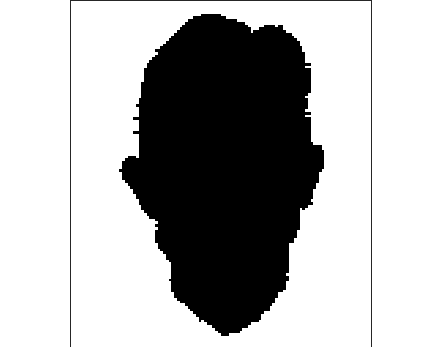}
\includegraphics[width=0.07\textwidth]{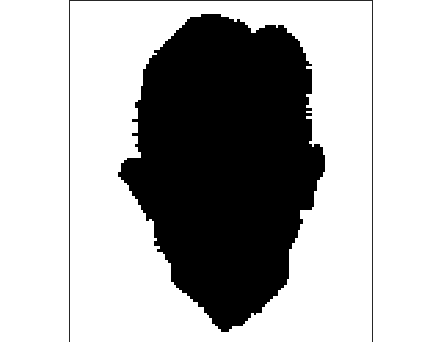}
\includegraphics[width=0.07\textwidth]{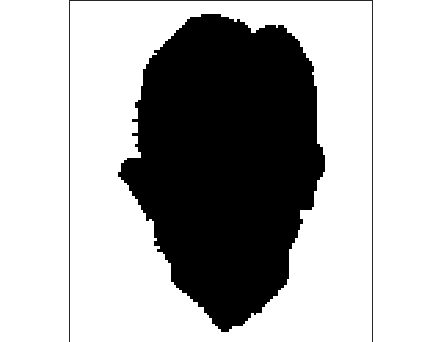}
\includegraphics[width=0.07\textwidth]{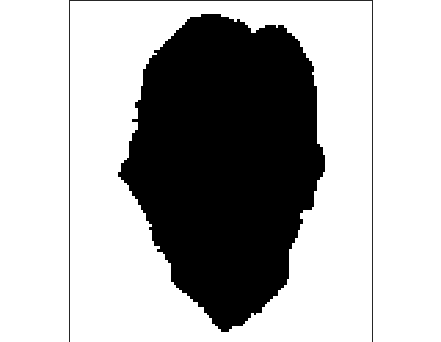}
\includegraphics[width=0.07\textwidth]{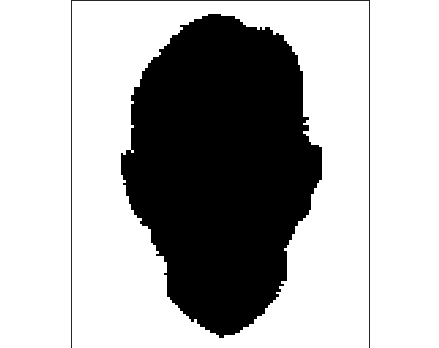}
\includegraphics[width=0.07\textwidth]{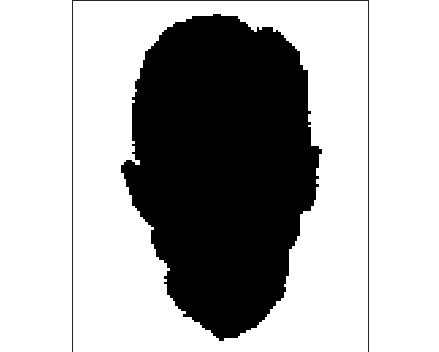}
\includegraphics[width=0.07\textwidth]{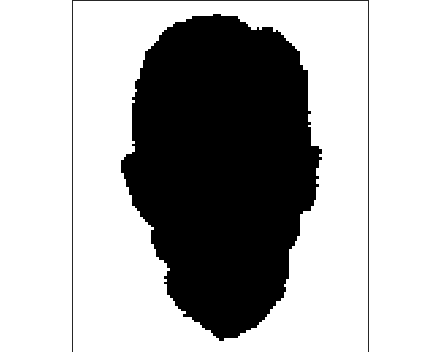}
\includegraphics[width=0.07\textwidth]{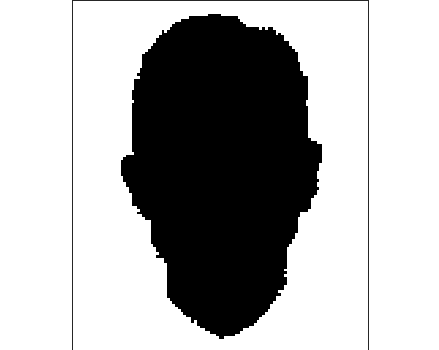}
\includegraphics[width=0.07\textwidth]{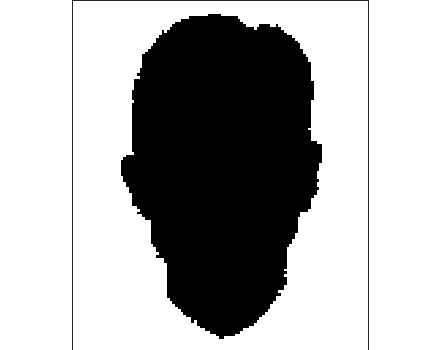}
\vskip 2mm
\includegraphics[width=0.075\textwidth]{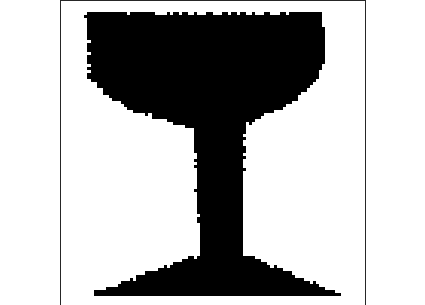}
\includegraphics[width=0.075\textwidth]{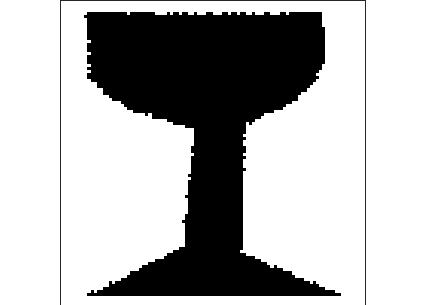}
\includegraphics[width=0.075\textwidth]{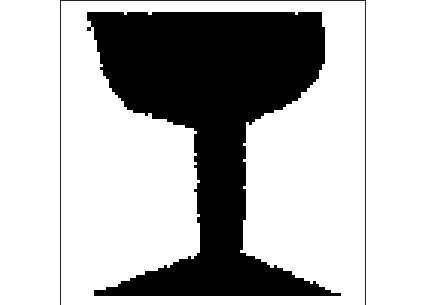}
\includegraphics[width=0.075\textwidth]{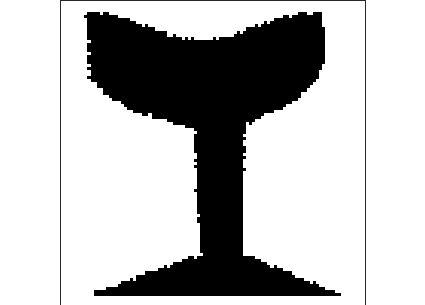}
\includegraphics[width=0.075\textwidth]{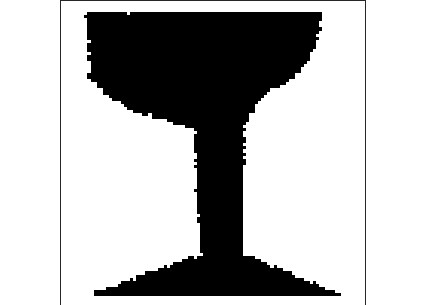}
\includegraphics[width=0.075\textwidth]{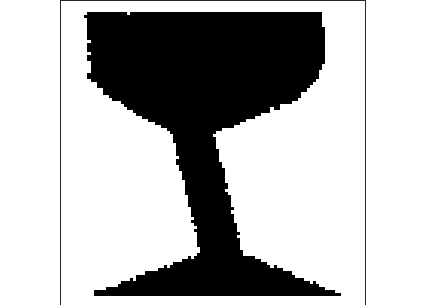}
\includegraphics[width=0.075\textwidth]{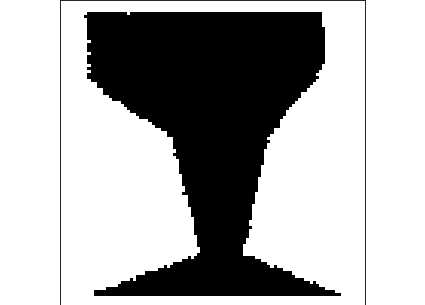}
\includegraphics[width=0.075\textwidth]{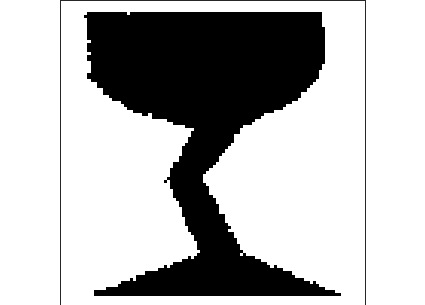}
\includegraphics[width=0.075\textwidth]{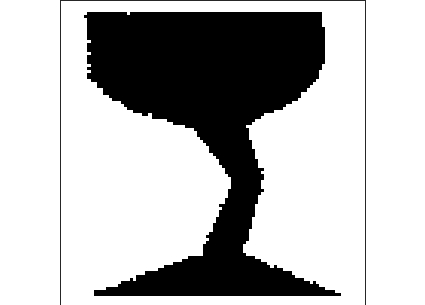}
\includegraphics[width=0.075\textwidth]{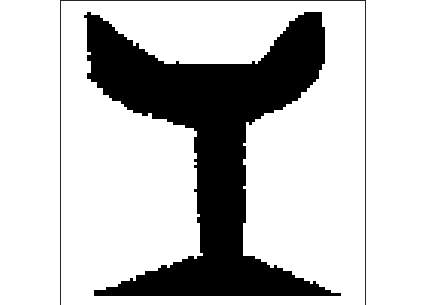}
\includegraphics[width=0.075\textwidth]{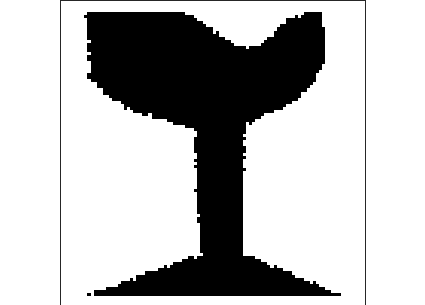}
\vskip 2mm
\includegraphics[width=0.07\textwidth]{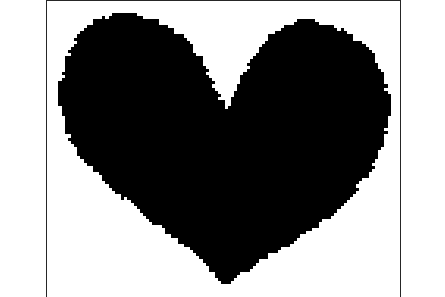}
\includegraphics[width=0.07\textwidth]{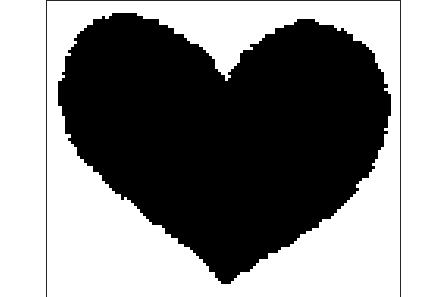}
\includegraphics[width=0.07\textwidth]{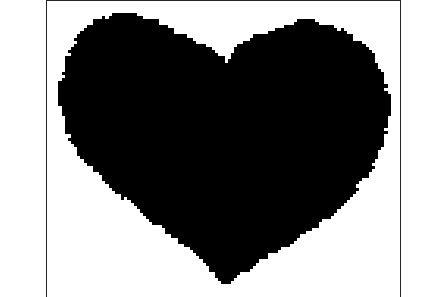}
\includegraphics[width=0.07\textwidth]{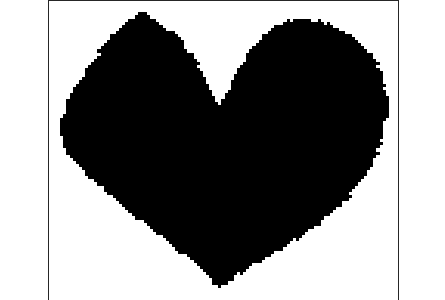}
\includegraphics[width=0.07\textwidth]{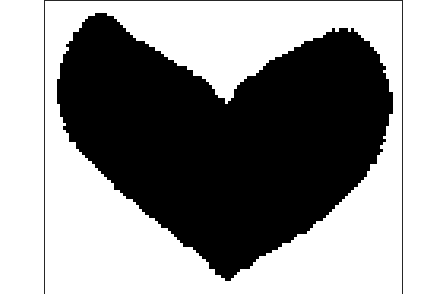}
\includegraphics[width=0.07\textwidth]{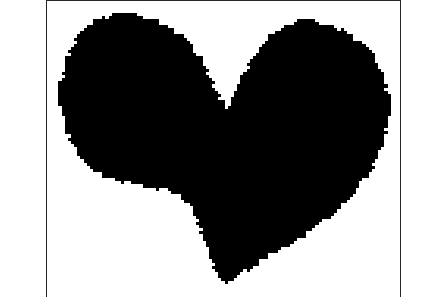}
\includegraphics[width=0.07\textwidth]{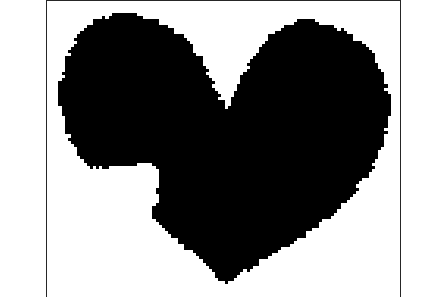}
\includegraphics[width=0.07\textwidth]{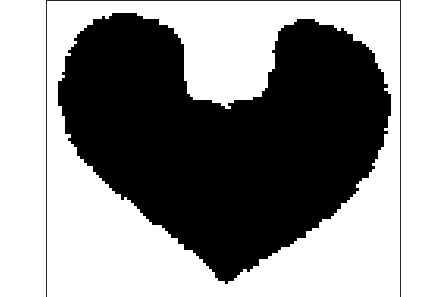}
\includegraphics[width=0.07\textwidth]{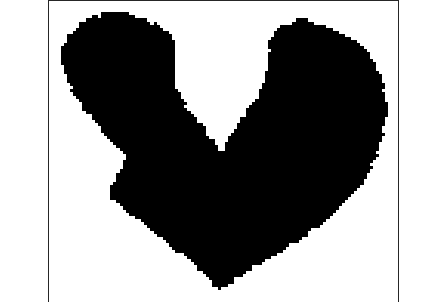}
\vskip 2mm
\includegraphics[width=0.07\textwidth]{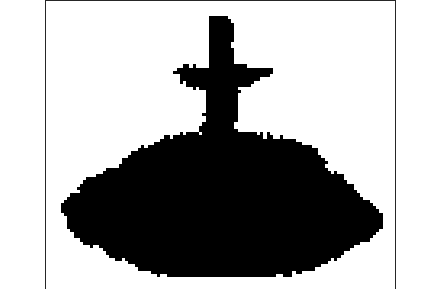}
\includegraphics[width=0.07\textwidth]{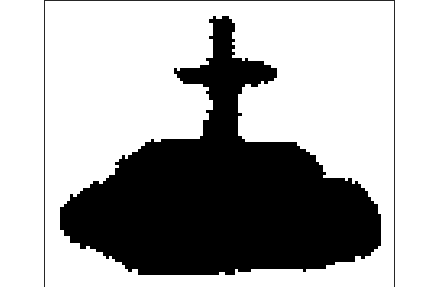}
\includegraphics[width=0.07\textwidth]{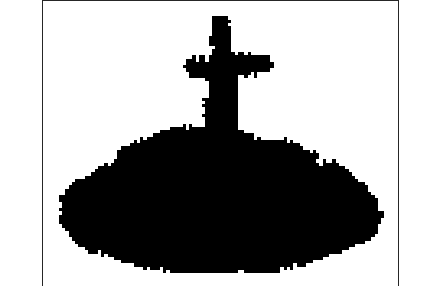}
\includegraphics[width=0.07\textwidth]{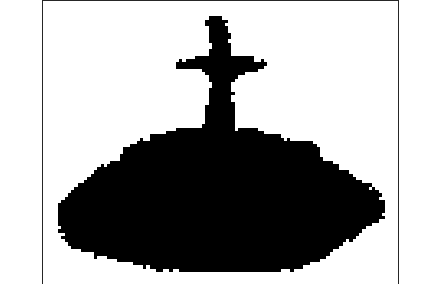}
\includegraphics[width=0.07\textwidth]{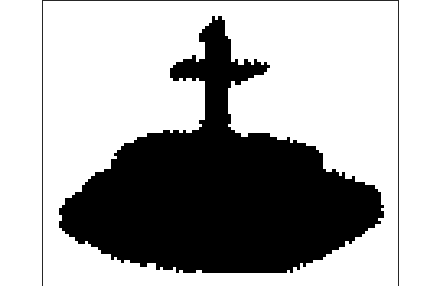}
\includegraphics[width=0.07\textwidth]{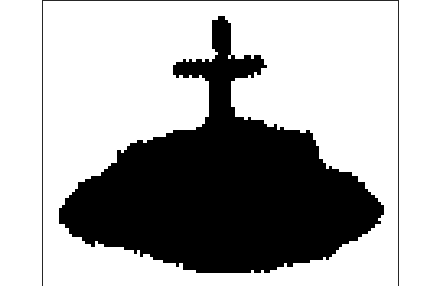}
\includegraphics[width=0.07\textwidth]{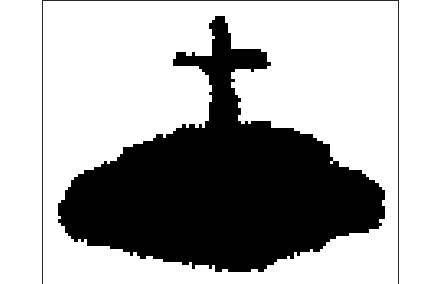}
\includegraphics[width=0.07\textwidth]{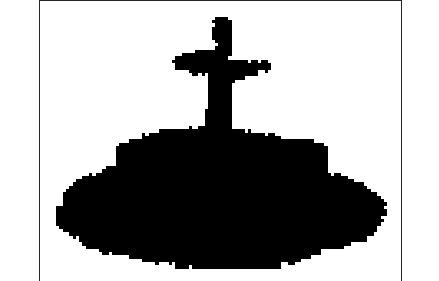}
\includegraphics[width=0.07\textwidth]{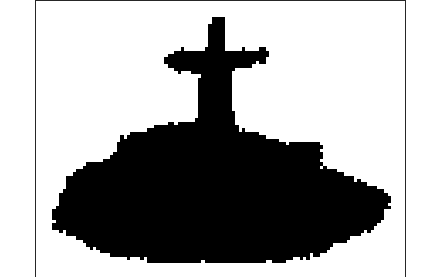}
\includegraphics[width=0.07\textwidth]{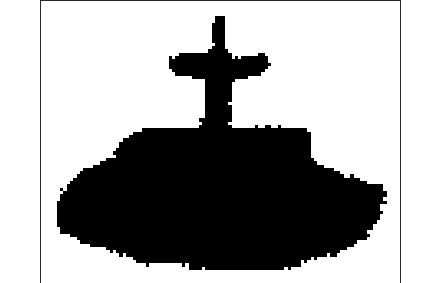}
\includegraphics[width=0.07\textwidth]{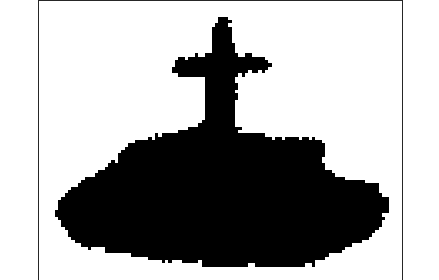}
\includegraphics[width=0.07\textwidth]{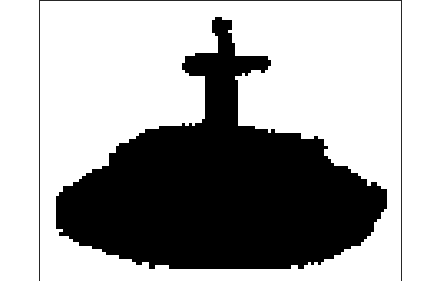}
\vskip 2mm
\includegraphics[width=0.06\textwidth]{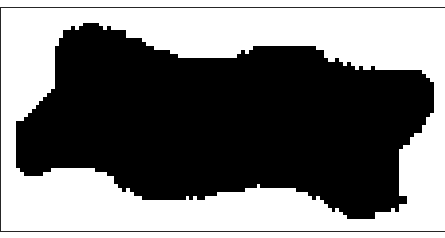}
\includegraphics[width=0.06\textwidth]{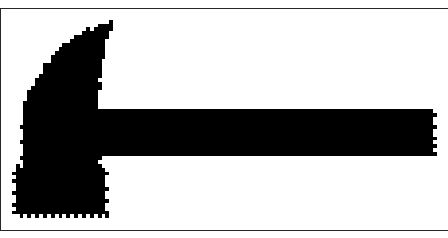}
\includegraphics[width=0.06\textwidth]{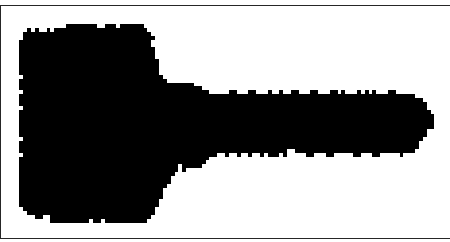}
\includegraphics[width=0.06\textwidth]{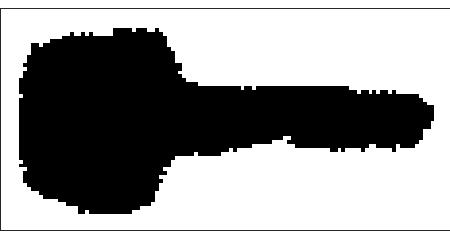}
\includegraphics[width=0.06\textwidth]{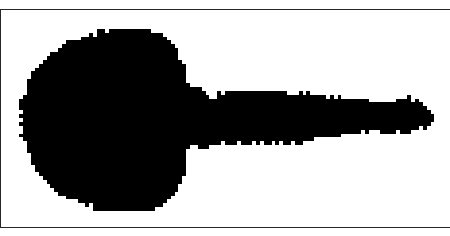}
\includegraphics[width=0.06\textwidth]{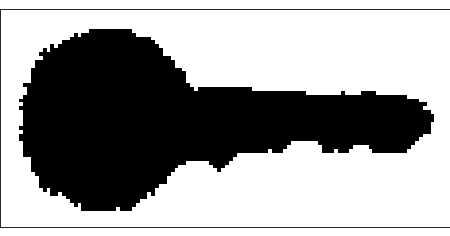}
\includegraphics[width=0.06\textwidth]{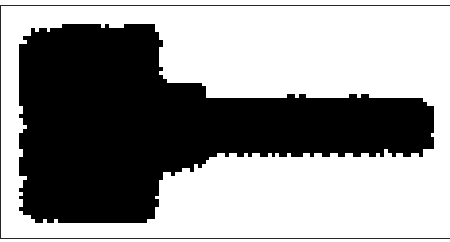}
\includegraphics[width=0.06\textwidth]{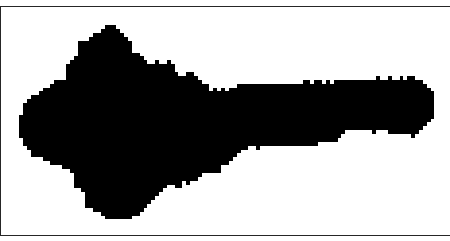}
\includegraphics[width=0.06\textwidth]{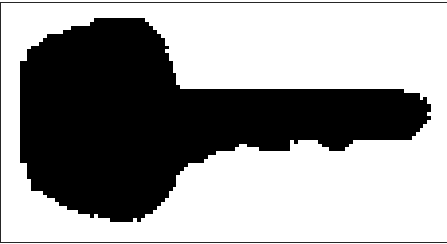}
\includegraphics[width=0.06\textwidth]{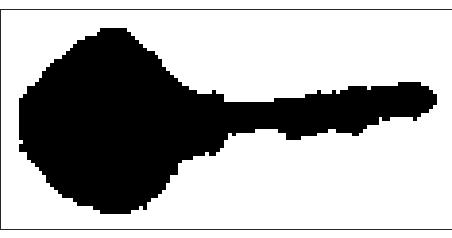}
\includegraphics[width=0.06\textwidth]{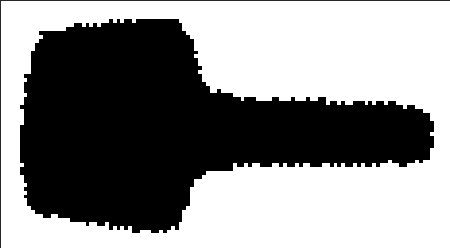}
\includegraphics[width=0.06\textwidth]{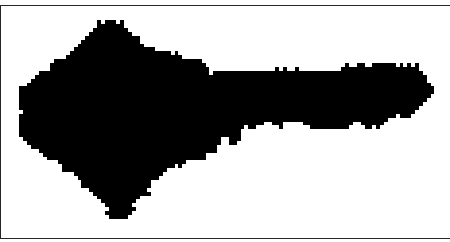}
\includegraphics[width=0.06\textwidth]{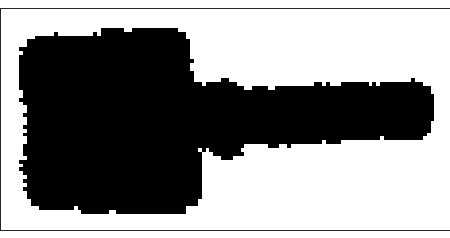}
\includegraphics[width=0.06\textwidth]{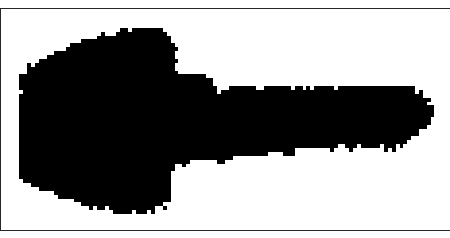}
\vskip 2mm
\includegraphics[width=0.05\textwidth]{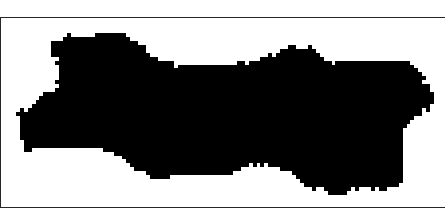}
\includegraphics[width=0.05\textwidth]{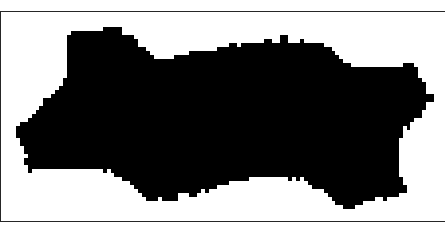}
\includegraphics[width=0.05\textwidth]{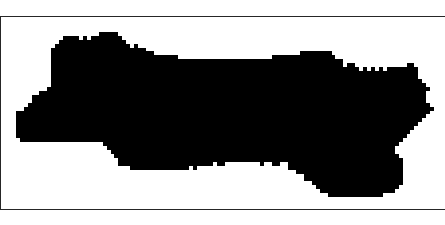}
\includegraphics[width=0.05\textwidth]{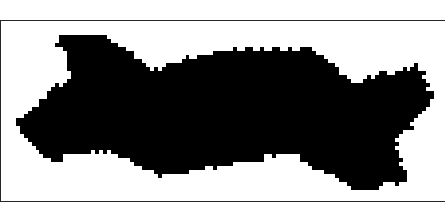}
\includegraphics[width=0.05\textwidth]{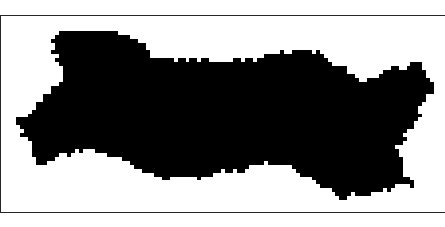}
\includegraphics[width=0.05\textwidth]{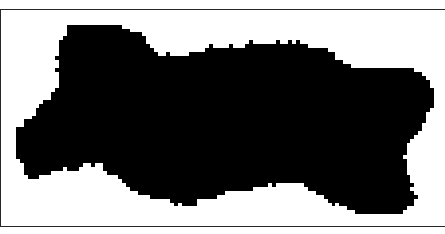}
\includegraphics[width=0.05\textwidth]{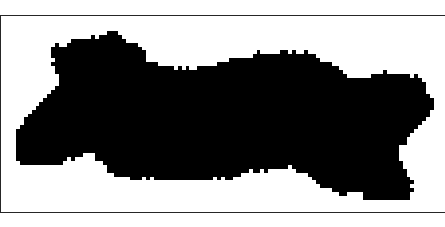}
\includegraphics[width=0.05\textwidth]{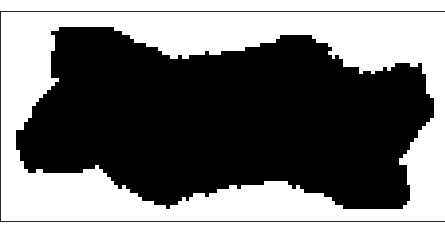}
\includegraphics[width=0.05\textwidth]{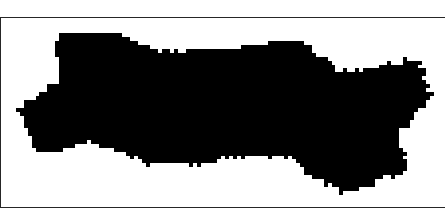}
\includegraphics[width=0.05\textwidth]{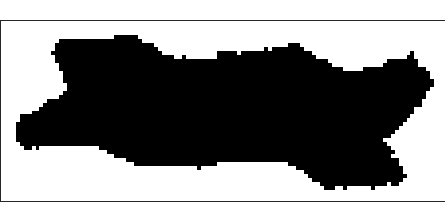}
\includegraphics[width=0.05\textwidth]{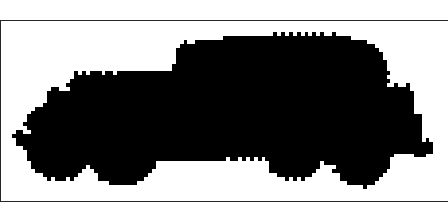}
\includegraphics[width=0.05\textwidth]{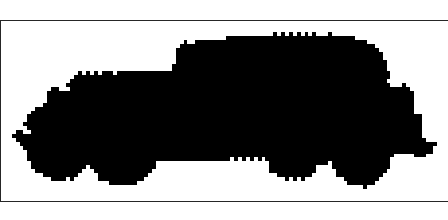}
\includegraphics[width=0.05\textwidth]{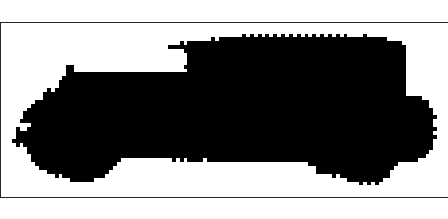}
\includegraphics[width=0.05\textwidth]{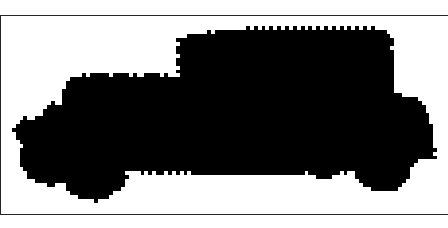}
\includegraphics[width=0.05\textwidth]{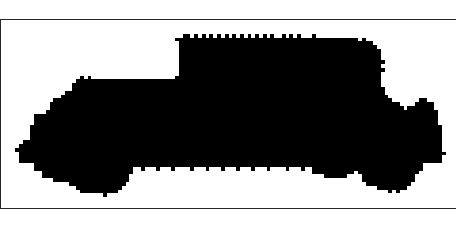}
\includegraphics[width=0.05\textwidth]{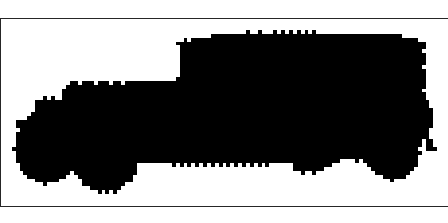}
\includegraphics[width=0.05\textwidth]{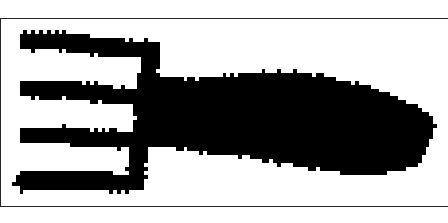}
\vskip 2mm
\includegraphics[width=0.06\textwidth]{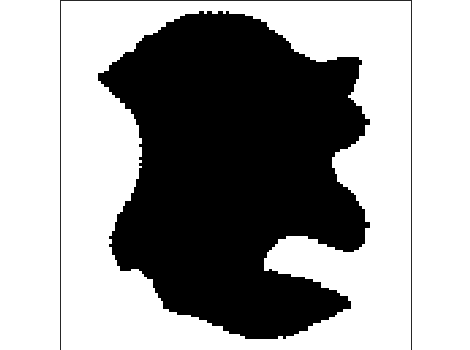}
\includegraphics[width=0.06\textwidth]{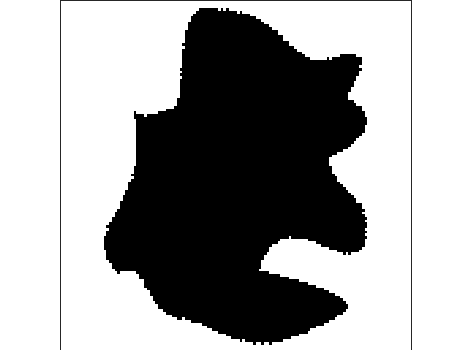}
\includegraphics[width=0.06\textwidth]{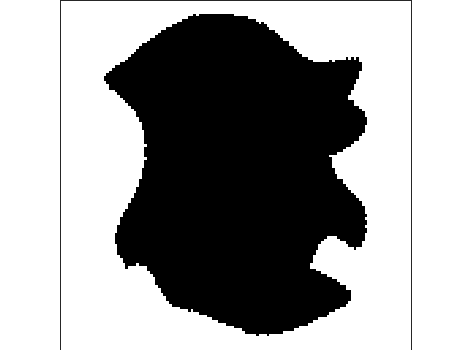}
\includegraphics[width=0.06\textwidth]{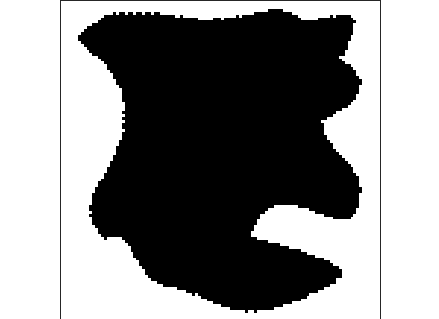}
\includegraphics[width=0.06\textwidth]{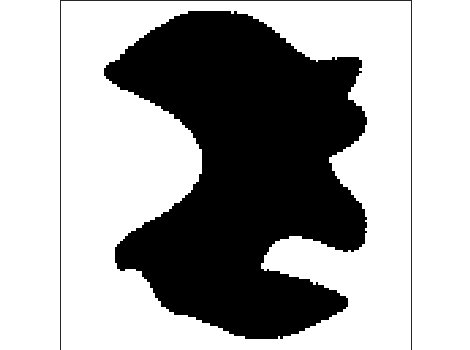}
\includegraphics[width=0.06\textwidth]{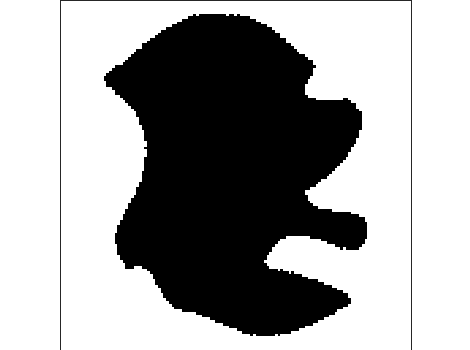}
\includegraphics[width=0.06\textwidth]{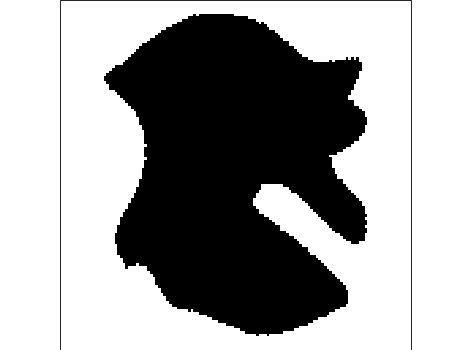}
\includegraphics[width=0.06\textwidth]{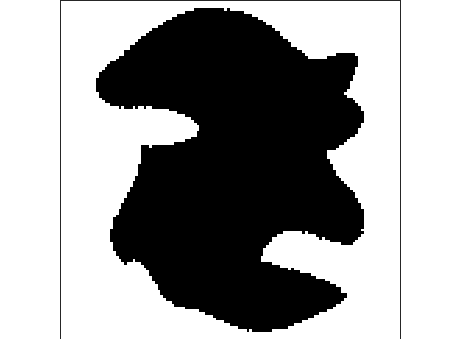}
\includegraphics[width=0.06\textwidth]{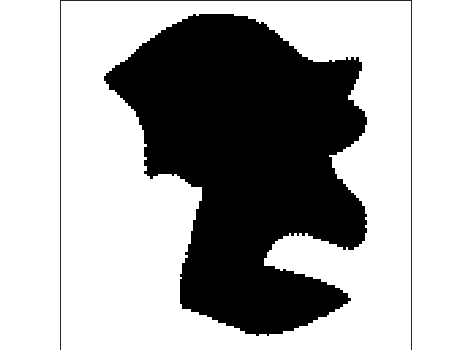}
\includegraphics[width=0.06\textwidth]{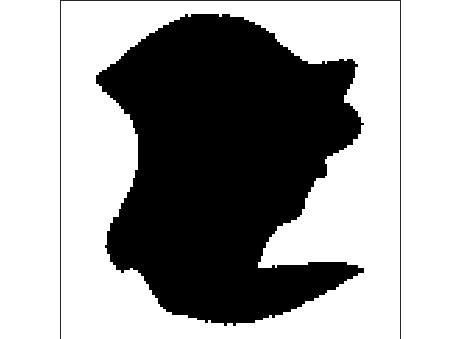}
\includegraphics[width=0.06\textwidth]{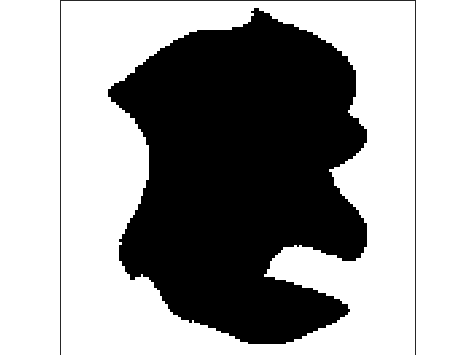}
\includegraphics[width=0.06\textwidth]{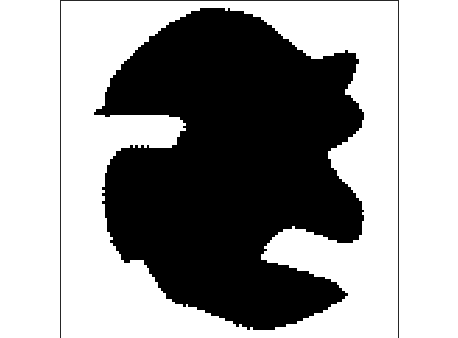}
\includegraphics[width=0.06\textwidth]{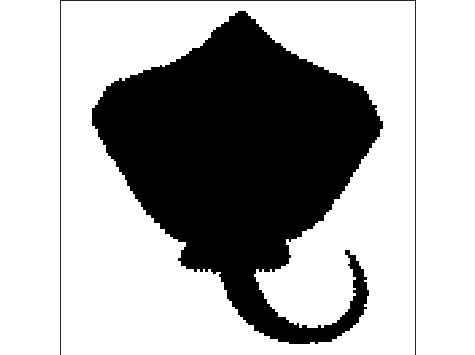}
\includegraphics[width=0.06\textwidth]{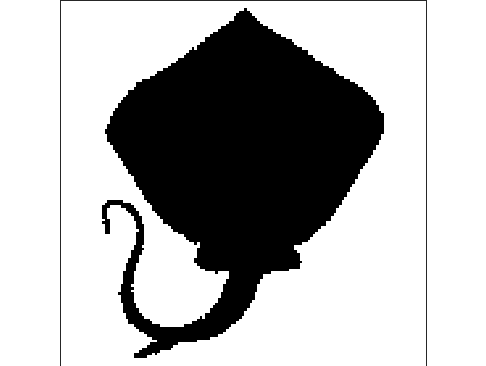}
\vskip -3mm
\caption{ The first 10 clusters clustered by CRLM (every row is one cluster)  }\label{fig:Kimia 216 Alg 2}
\vspace{-5mm}
\end{figure*}

In Figure \ref{fig:Kimia 216 Alg 2} are shown the observations clustered  by Algorithm 2 after each of the first 10 iterations. These clusters can be regarded as top 10 clusters that are more separated from other clusters and that have smaller inner distances within each cluster. 

\noindent{\bf Results on the 1070 shape database.}  For the 1070 shape database, we first conduct similar clustering analysis with different methods.
 However, the results are not as expected since this dataset contains 66 classes and the numbers of images in the different classes is not balanced.
 The class with max number of pictures has over 50 pictures while there are some classes with only  one picture. 
Furthermore, some classes are very similar to each other and they are prone to be clustered as a single cluster.
For example there are several different classes for different types of airplanes and bunnies, which can be easily clustered as just two clusters.

Due to the drawbacks of the labels of the 1070 Shape Database, we reassign the labels and make it similar to our theoretical data structure, GMM with a uniform background. 
We first perform clustering and use the original labels to obtain measurements of the clustering accuracy. 
Then we pick the top 6 clusters that are clustered with high accuracy for every method and label them as 6 positive clusters. Figure \ref{pos_examples} shows the six clusters selected as positive clusters. 
\begin{figure*}[h]
\vspace{-4mm}
\centering
\includegraphics[width=.13\textwidth]{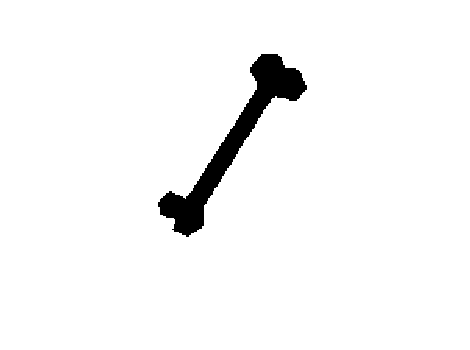}
\includegraphics[width=.13\textwidth]{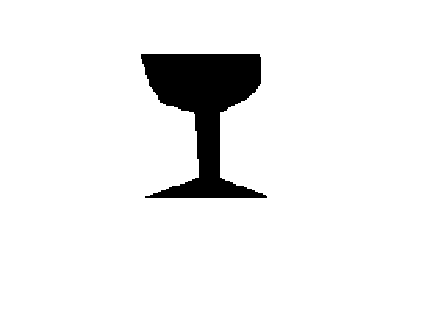}
\includegraphics[width=.13\textwidth]{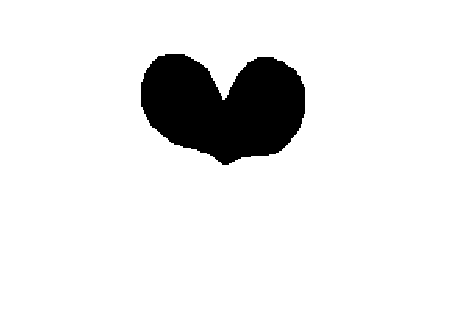}
\includegraphics[width=.13\textwidth]{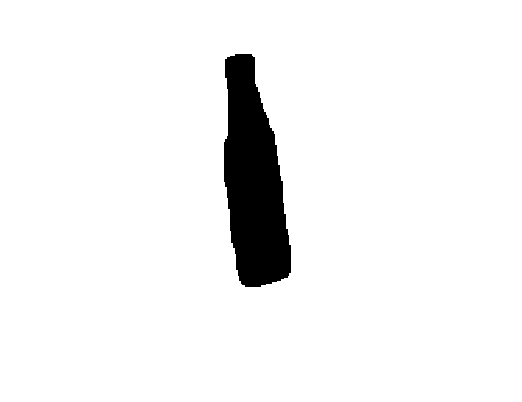}
\includegraphics[width=.13\textwidth]{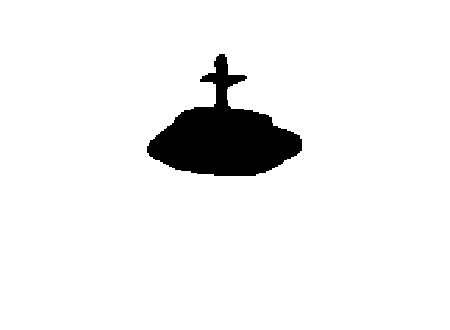}
\includegraphics[width=.13\textwidth]{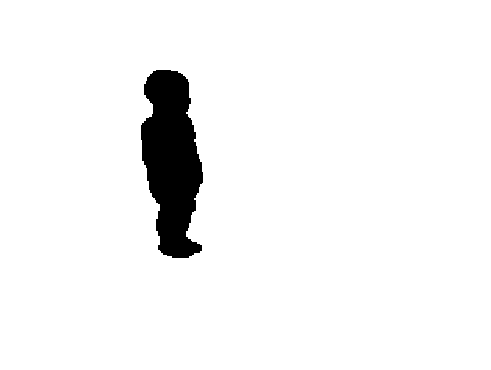}
\vskip -6mm
\caption{ Samples from six  distinct positive clusters that will be used in our experiment.}\label{pos_examples}
\vspace{-4mm}
\end{figure*}

The remaining observations that are not part of these six clusters are labeled as negatives. 
It is obvious that with the reassignment of the ground truth labels of 1070 Shape Data, the clustering accuracy for various clustering algorithms can be greatly improved. 
We take 10 samples from each positive cluster and perform two experiments with different numbers of negatives added to the dataset. 
In the first experiment, we take 60 negative samples, one from each of the 60 remaining original labels. 
In the second experiment, we take 60 more negative samples from the 60 original labels excluding the 6 positive clusters.  
The comparison results are in Table \ref{1070 Shape Database}.

\begin{table}[t]
\vspace{-3mm}
\center{
\scalebox{0.95}{
\begin{tabular}{c|ccccccccc}
\hline
$$  &K-means &DBSCAN &CL& EM & SC &TD & T-SNE&KC&Our (CRLM)  \\
\hline
\multicolumn{4}{c}{60 positives, 60 negative samples}  \\
\hline
Rand Index (\%)      &51.25   &89.17    &45.00 &22.50 &32.61  & 47.50&51.67&72.50&{\bf 99.17}      \\ 

F-measure (\%)  &55.06  &68.37    &41.10  &18.96  &35.83 &17.09&54.59&73.18&{\bf 98.79} \\
\hline
\multicolumn{4}{c}{60 positives, 120 negative samples} \\
\hline
Rand Index (\%)    &65.86   &66.11    &55.00 &46.67 &27.78 &55.56&47.22&56.11  &{\bf98.89}      \\ 
F-measure (\%)  &47.77  &  73.86   &37.43  &21.39 &29.13 &33.61 &54.24&66.64 &{\bf 98.38}\\
\hline
\end{tabular}}}
\vskip -2mm
\caption{Comparison for different methods on subsets of the 1070 Shape Database}\label{1070 Shape Database}
\vspace{-4mm}
\end{table}

In both cases, our loss-based approach (CLRM) outperforms the other methods evaluated in terms of Rand Index and F-measure. 
The clustering results indicate that our method obtains high Rand Index when the data follows the assumptions used for obtaining the theoretical guarantees. 
As more negatives are added to the data, there is still a large gap between the performance of other methods and our method. 
To sum up, when the data fits the GMMUB model well, our model will work very well. 

\subsubsection{IMDB-WIKI Face Dataset and ImageNet Dataset}

In this section we construct a dataset containing images from a mixture of instances of an object (the human face) and various images from many other classes, to see how well different algorithms can cluster the faces correctly. 
For this purpose we obtain faces from the IMDB-WIKI Face dataset  \cite{Rothe-ICCVW-2015} and the rest of the images from the ILSVRC 2012 dataset \cite{ILSVRC15}. Some examples are shown in Figure \ref{face_example}.

\begin{figure*}[ht]
\centering
\includegraphics[width=0.09\textwidth]{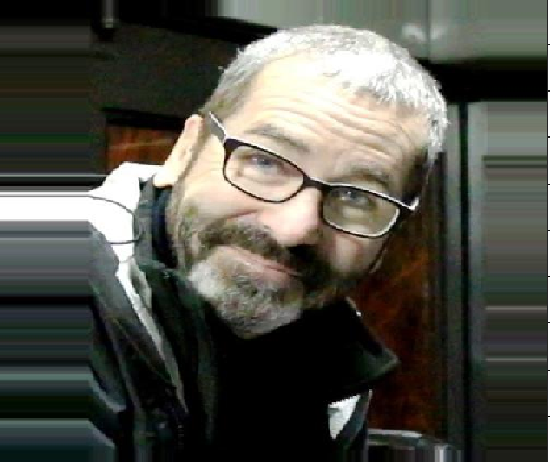}
\includegraphics[width=0.09\textwidth]{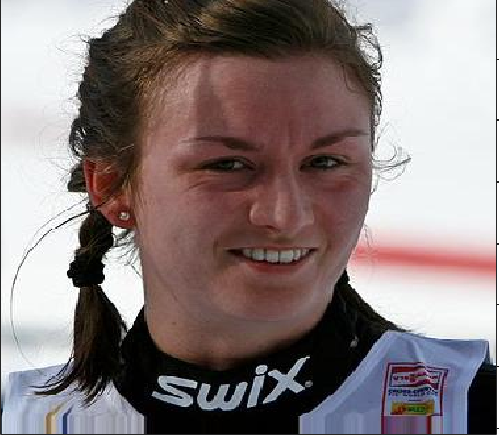}
\includegraphics[width=0.09\textwidth]{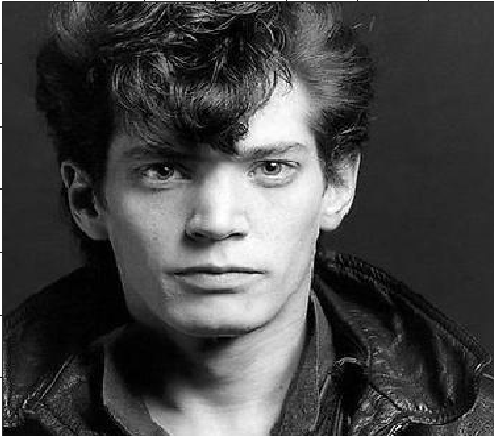}
\includegraphics[width=0.09\textwidth]{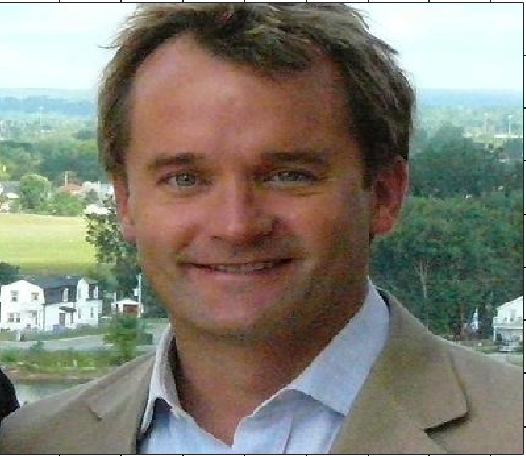}
\includegraphics[width=0.09\textwidth]{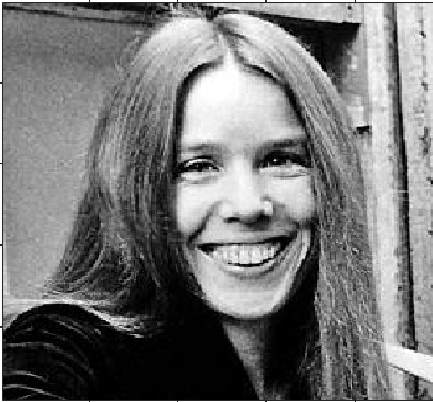}
\includegraphics[width=0.09\textwidth]{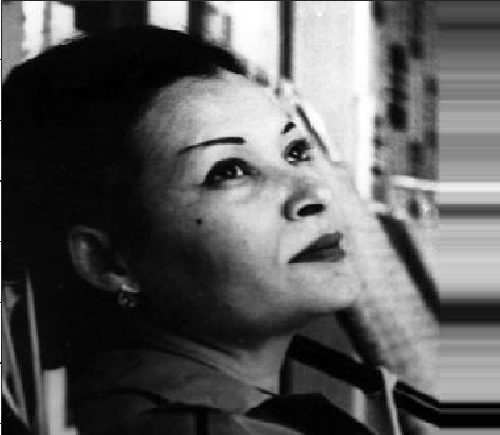}
\includegraphics[width=0.09\textwidth]{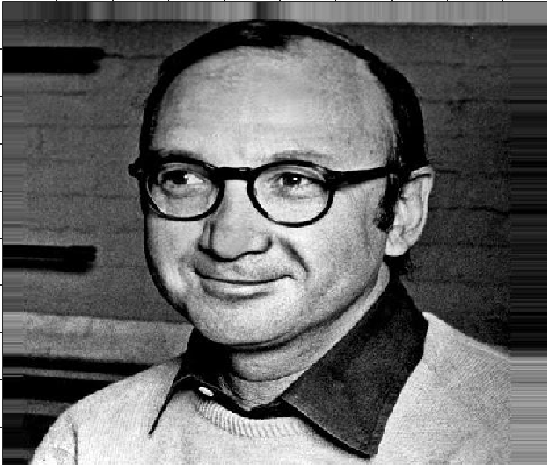}
\includegraphics[width=0.09\textwidth]{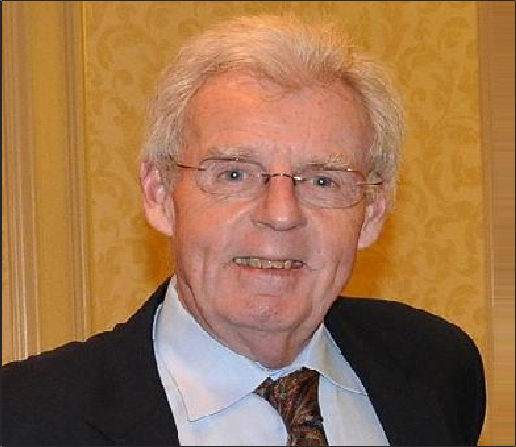}
\includegraphics[width=0.09\textwidth]{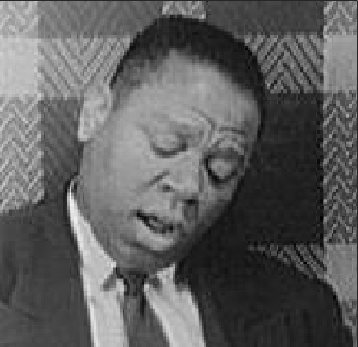}
\includegraphics[width=0.09\textwidth]{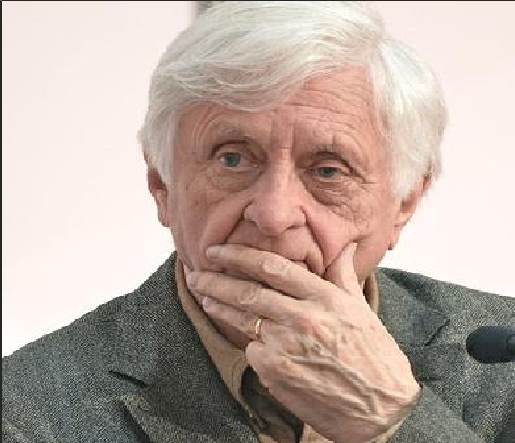}
\vskip 1mm
\includegraphics[width=0.09\textwidth]{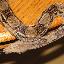}
\includegraphics[width=0.09\textwidth]{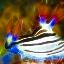}
\includegraphics[width=0.09\textwidth]{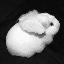}
\includegraphics[width=0.09\textwidth]{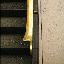}
\includegraphics[width=0.09\textwidth]{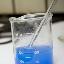}
\includegraphics[width=0.09\textwidth]{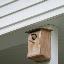}
\includegraphics[width=0.09\textwidth]{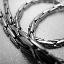}
\includegraphics[width=0.09\textwidth]{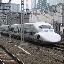}
\includegraphics[width=0.09\textwidth]{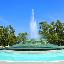}
\includegraphics[width=0.09\textwidth]{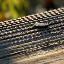}
\vskip -3mm
\caption{ Example images (top : positives, bottom: negatives) }\label{face_example}
\vspace{-5mm}
\end{figure*}

The IMDB-WIKI Face dataset contains  over 500k  face images, while the ILSVRC 2012 dataset contains images of 1000 different classes of objects with 600-1300 images from each class. Observe that the human face is not one of the 1000 classes of the ILSVRC dataset.

\noindent{\bf Data preprocessing.} The images from the IMDB-WIKI  and ILSVRC data sets are resized to $224\times 224$ pixels, and a  pre-trained CNN named VGG-very-deep-16 \cite{simonyan2014very} is used to obtain a 4096 dimensional feature vector for each image. Since the human face was not a label on which the CNN was trained, the clustering of the faces happened just because the CNN obtained a very meaningful representation that generalizes well to other classes. In fact it was observed that the features obtained by these pretrained CNNs can be very useful for many applications \cite{sharif2014cnn}.

As already mentioned, the training data is constructed by taking a subset of the IMDB-Wiki dataset and labeled them as the positive cluster and the rest of the images are picked  randomly from ILSVRC 2012 and are labeled as negatives. The positive/negative labels are not used for clustering, they are only used for computing the clustering accuracy measures.

 The clustering results are presented in Table \ref{ImageNet}. For any algorithm returning two labels, there are two ways of  mapping from the output labels to the true labels. 
The Rand Index and F-measure are calculated for each of the two ways and the maximum value is reported in Table \ref{ImageNet}. 
 
 For DBSCAN, the number of output labels (clusters) can be larger than 2. In this case we set the class with the largest number of observations as one class and the rest of the observations are labeled as the other class.  

\begin{table}[t]
\vspace{-4mm}
\center{
\scalebox{0.95}{
\begin{tabular}{c|cccccccc}
\hline
$$  &K-means & DBSCAN &CL & EM & SC& TSNE&KC&Our (CRLM)  \\
\hline
\multicolumn{4}{c}{$n=50$, $n_p=30$}  \\
\hline
Rand Index (\%)      &78.20   &89.40    &62.40 &57.20 & 55.31&55.80 &71.82&{\bf 99.60}    \\ 
F-Measure(\%)    &85.79 &97.27   &76.15 &62.23 &59.60&59.46&81.90&{\bf 99.66}     \\ 
\hline
\multicolumn{4}{c}{$n=100$,$n_p=30$} \\
\hline
Rand Index (\%)    &66.20  &94.90  &67.90  &55.40 &54.40&55.6&89.88& {\bf 99.80}      \\ 
F-Measure(\%)     &61.05  &91.86    &46.92  &43.02 &42.15&44.03&87.84&{\bf 99.67}      \\ 
\hline
\multicolumn{4}{c}{$n=200$, $n_p=30$} \\
\hline
Rand Index (\%)    &71.25   &95.85  &83.95 &53.15 & 54.55&51.75&90.35&{\bf 99.70}      \\ 
F-Measure(\%)        &31.43  &86.39    &26.33 &26.88 &26.38&25.71&91.17&{\bf 98.99}     \\ 
\hline
\multicolumn{4}{c}{$n=500$, $n_p=30$} \\
\hline
Rand Index (\%)    &75.40   &96.72    &93.48 &52.86  &51.84&51.00& 92.22&{\bf 99.76}      \\ 
F-Measure(\%)     &13.77  &76.91   &11.38 &12.01  &12.29&12.53&90.78&{\bf 98.02}      \\ 
\hline
\multicolumn{4}{c}{$n=1000$, $n_p=30$} \\
\hline
Rand Index (\%)    &74.88  &97.09   &96.68 &52.44 & 51.46 &50.45&93.35& {\bf 99.86}      \\ 
F-Measure(\%)    &7.43  & 53.21  &5.84&6.27 &6.66&6.39&92.00&{\bf 97.70}  \\ 
\hline
\multicolumn{4}{c}{$n=3000$, $n_p=30$} \\
\hline
Rand Index (\%)    &75.76  &98.90   &98.74 &51.35 & 50.45 &50.23& 90.88&{\bf 99.88}      \\ 
F-Measure(\%)    &2.57  &54.44  &1.99 &2.19 &2.34&2.21&91.10&{\bf 94.16}  \\ 
\hline
\end{tabular}}}
\caption{Accuracy of clustering algorithms on subset of  IMDB-WIKI face data set and ImageNet Dataset}\label{ImageNet}
\vspace{-6mm}
\end{table}

The results show the best accuracy is obtained by our method. However, note that when  the number of negatives increase, a high Rand Index can be achieved by just clustering all data as a single cluster. 
Hence,  comparison using the F-Measure and Rand Index together can provide a more reliable measure of the clustering accuracy. CRLM outperforms other methods with high  F-Measure and Rand Index. This indicates that the construction of the dataset probably satisfies the conditions of the GMMUB. 
Besides, when more negative observations are used, although the Rand Index doesn't decrease, the decreasing F-Measure shows that the clustering performance for CRLM decreases slightly. 
Among the other clustering algorithms, KNNCLUST outperforms the others with an acceptable Rand Index and F-measure when $n$ is small. EM and Spectral Clustering have similar results that cluster the dataset half into positives and half into negatives.

\section{Conclusion and Future Work}

In this paper, we proposed a novel  method (CRLM) based on robust loss minimization for clustering Gaussian Mixtures together with an extra mixture component that is a uniform distribution. This setup comes closer to the reality existent in image data where the images contain from some objects of interest (e.g. faces, buildings, cars, etc), but most of the images contain just clutter (e.g. outdoor images, or other rare random objects). In a standard GMM such clutter images are considered outliers, and therefore must be removed. However, in our setup these outliers dominate the data by a ratio 100:1 or more.

The basic assumptions for our algorithm are: 1. Isotropic Gaussians for the foreground (positive) clusters.  2. Large radius $D\sqrt{d}$ for the  background samples. 3.  Sufficient separation between any two positive clusters. 
Unlike other clustering methods, our algorithm enjoys strong theoretical guarantees that it finds the correct clusters with high probability, and does not depend on an initialization.
 Moreover, it can work with a predefined number of clusters or it can estimate the number of clusters. 

In the simulations, we generated data as a GMM with  a Uniform Background  under the assumptions above, with the majority of data points coming from the background component. 
The simulation experiments indicate that CRLM can obtain results close to perfect as long as the sample size is large enough. 
We also conducted an analysis of the robustness of CRLM with regards to  $\sigma_{\max}$ and the estimation of the number of clusters $k$. 
For real data analysis, we experimented with an original dataset and with two subsets constructed to have a structure similar to GMM with a uniform background. 
The real data results witness that CRLM usually outperforms other classic, regular clustering methods.

However, there are  still some drawbacks of CRLM that could lead to potential future work to overcome them.  
On one hand, the effectiveness of CRLM relies on some assumptions that are sometimes difficult to satisfy.  
On the other hand, real data clustering results of CRLM and other clustering methods on large image datasets  are far from being satisfactory.  
Hence, our future work will focus on two aspects.
 Firstly, we plan to modify and improve  our algorithm to make it applicable to more general cases.  
Secondly, we plan to apply our algorithm to other image datasets such as bioinformatics data, and investigate scalable semi-supervised learning approaches based on our algorithm for real image data.

\section*{Acknowledgment}

The work is supported in part by DARPA ARO W911NG-16-1-0579. 

\appendix

\section{Proofs}\label{app}

There are several technical lemmas and propositions that will be useful for the proofs.

\begin{proposition}
If $\bx\in \RR^d$ is a uniform sample inside the ball of radius $R$ centered at $0$, then the pdf of $u=\|\bx\|^2$ is \label{prop:pdfu}
$$f(u)\propto 
\begin{cases}
u^{d/2-1} &\text{if }\; u<R^2\\
0 &\text{else}
\end{cases}.$$
Therefore $u/R^2\sim Beta(d/2,1)$. 
\end{proposition}

\begin{proof} We have the CDF
$$ P(\Vert \bx\Vert^{2}\leq u)= P(\Vert \bx\Vert\leq \sqrt{u}) =\frac{u^{d/2}}{R^{d}},$$ because of the volume ratio of the $d$-dimensional balls of radius $\sqrt{u}$ and $R$. By taking the derivative of the CDF, we obtain the pdf.
\end{proof}

\begin{corollary} \label{cor:pdfuniform}
If $\bx\in \RR^d$ is a uniform sample inside the ball of radius $R=\sigma\sqrt{dG}$ centered at $0$, then the pdf of the random variable $L=\displaystyle \min(\frac{\Vert \bx\Vert^{2}}{d\sigma^{2}}-G,0)$ is
\begin{equation}
f_\sigma(L)\propto 
\begin{cases}
(L+G)^{d/2-1} &\text{if }\; L\in [-G,0]\\
0 &\text{else} \label{eq:negpdf}
\end{cases}.
\end{equation}
and the expected value of  L is $E[L]=-G/(d/2+1)$.
\end{corollary}

\begin{proof} Denoting $u=\|\bx\|^2$ we have  
\[
L=\min(\frac{u}{d\sigma^{2}}-G,0)=
\min(\frac{u-R^2}{d\sigma^{2}},0)
\]
so $0\leq u\leq R^2$ iff $L\in [-G,0]$ and in this case $u=d\sigma^2L+R^2=d\sigma^2(L+G)$ and the proof follows from Proposition \ref{prop:pdfu}.

Since, from Proposition 1, $u/R^2=(L+G)/G \sim Beta(d/2,1)$
\[E[L]=E[L+G]-G=\frac{d/2}{d/2+1}G-G=-G/(d/2+1). \]
\end{proof}
\begin{proposition} \label{Guassiannorm}
For an isotropic Gaussian random variable $\bx\sim N(0,\sigma_1^2I_d)$ the pdf of  $u=\|\bx\|^2$ is \label{prop:pdfnor}
$\Gamma(d/2,2\sigma_1^2)$ i.e.
\[g(u)\propto
u^{d/2-1}\exp(-u/2\sigma_1^{2}).
\]
Thus $E[u]=d\sigma_1^2$.
\end{proposition}
\begin{proof} We have the CDF
\[ P(\Vert \bx\Vert^{2}\leq u)= P(\Vert \bx\Vert\leq \sqrt{u})=cS_d\int_{0}^{\sqrt{u}}\exp(\frac{-r^{2}}{2\sigma_1^{2}})r^{d-1}dr,\]
where $S_d$ is the area of the unit ball in $\RR^d$. By taking the derivative of the CDF, we obtain the pdf.
\end{proof}
\begin{corollary} \label{cor:pdfgaussian}
For an isotropic Gaussian random variable $\bx\sim N(0,\sigma_1^2I_d)$ the pdf of  $L=\displaystyle \min(\frac{\Vert \bx\Vert^{2}}{d\sigma^{2}}-G,0)$ is
\begin{equation}
g_\sigma(L)\propto
\begin{cases}(L+G)^{d/2-1}\exp(-L\sigma^{2}/\sigma_1^{2})&\text{if }\; L\in [-G,0]\\
0 &\text{else} \label{eq:pospdf}
\end{cases}.
\end{equation}
and  $E[L]\leq\sigma_1^2/\sigma^2-G$.
\end{corollary}
\begin{proof}
Using $u=d\sigma^2(L+G)$ and Proposition \ref{prop:pdfnor} we get that:
\begin{equation}
f(L)\propto
\begin{cases}(L+G)^{d/2-1}\exp(-\frac{d\sigma^2(L+G)}{2\sigma_1^{2}})&\text{if }\; L\in [-G,0]\\
0 &\text{else,} \label{eq:posspdf}
\end{cases}
\end{equation}

so

\begin{equation}
f(L)\propto
\begin{cases}(L+G)^{d/2-1}\exp(-Ld\sigma^{2}/2\sigma_1^{2})&\text{if }\; L\in [-G,0]\\
0, &\text{else.} \label{eq:possspdf}
\end{cases}
\end{equation}

Then $E[L]\leq E[\displaystyle \frac{\Vert \bx\Vert^{2}}{d\sigma^{2}}-G]=\frac{d\sigma_1^2}{d\sigma^2}-G=\sigma_1^2/\sigma^2-G$. 
\end{proof}
\begin{corollary} \label{cor:meanncchi}
Suppose $\bx \sim \N(\bmu, \sigma_1^2I_d)$, then the random variable $u=\Vert \bx\Vert^2$ has  $E[u]=\sigma_1^2d+\|\bmu\|^2$.
\end{corollary}
\begin{proof} We have $\bx=\bmu+\sigma_1\beps$ with $\beps\sim N(\mathbf{0},I_d)$. We have:
\[
E(||\bx||^2)=E(\bx^{T}\bx)=E[(\bmu+\sigma_1\beps)^{T}(\bmu+\sigma_1\beps)]
=\|\bmu\|^2+\sigma_1^2E (\beps^{T}\beps).
\]
Then using Proposition  \ref{Guassiannorm} we obtain:
\[
E(||\bx||^2)=\|\bmu\|^2+\sigma_1^2E (\|\beps\|^2)=\|\bmu\|^2+\sigma_1^2d. \; 
\]
\end{proof}
%
%
\begin{lemma}\label{deltato0}
Let $\bx \sim \N(0,I_d)$. Then if $G\geq 1$ we have
\[
P(\|\bx\|_2^2/d>G)<2e^{-d(G-1)^2/8}
\]
\end{lemma}
\begin{proof} (Based on Wainwright \cite{wainwright2019high} Chapter 2). Since  $Y=\|\bx\|_2^2=\sum_{i=1}^d x_i^2$ and each $x_i^2$ is sub-exponential with parameters $(2,4)$, then $Y$ is sub-exponential with parameters $(\sigma,b)=(2\sqrt{n},4)$ and we obtain the tail bound
\[
P(|\frac{1}{d}\sum_{i=1}^d x_i^2-1|>t)\leq 2 e^{-dt^2/8}
\]
from which we obtain our result taking $t=G-1$.
\end{proof}
\begin{corollary} \label{cor:probgauss}
Suppose $\bx \sim \N(\bmu, \sigma^2I_d)$, then
\[
P(||\bx-\bmu||>\sigma\sqrt{dG})>2e^{-d(G-1)^2/8}.
\]
\end{corollary}
%
\begin{proof} Taking $\hat{x}=(\bx -\bmu)/\sigma$ in Lemma \ref{deltato0},$\hat{\bx} \sim \N(0,I_d)$, we have:
\[
P(||\bx-\bmu||>\sigma\sqrt{dG}) = P(\|\hat{\bx}\|^2/d>G)<2e^{-d(G-1)^2/8}
\]
\end{proof}
\begin{lemma}\label{C1bounds}
Let $\bx_{i}\in \RR^d, i=1,...,n$ be $n$ observations  uniformly distributed in a ball of radius $D\sqrt{d}$ centered at 0. Let ${\bmu_j}, j \in {1,\cdots,k}$ be a set of  points in $\RR^d$. If assumption {\rm A1} is satisfied, then  with  probability at least $1-nk(2\sigma_{\max}\sqrt{G}/D)^d$, $||\bx_i-\bmu_j||>2\sigma_{\max}\sqrt{dG}$, $\forall i\in \{1,...,n\},\forall j\in \{1,...,k\}$.


\end{lemma}

\begin{proof} For any  sample $\bx_i$ from the uniform distribution within the ball of radius $D\sqrt{d}$, the probability that $||\bx_i-\bmu_j||\leq R$ is just the ratio between the volume of the $d$-dimensional ball with radius $R$ centered at $\bmu_j$ to that with radius $D\sqrt{d}$. Using the ball volume equation, we obtain  
\[
  P(\bigcup_{j=1}^k(||\bx_i-\bmu_j||\leq R))\leq\sum_{j=1}^{k}P(||\bx_i-\bmu_j||\leq R)=k(\frac{\pi^{d/2}R^d/\Gamma(d/2+1)}{\pi^{d/2}D^d/\Gamma(d/2+1)})=\frac{kR^d}{(D\sqrt{d})^d}
\]

 For any $i\in \{1,...,n\}$ denote $E_i$ be the random event $E_{i}:\bigcap_{j=1}^k(||\bx_i-\bmu_j||>R)$, so we just proved above that $P(\overline E_{i})\leq k(R/(D\sqrt{d}))^d$.
Then
$$P(\bigcap_{i=1}^nE_{i})=1-P(\bigcup_{i=1}^n \overline E_{i})\ge 1-\sum_{i=1}^n P(\overline E_{i})\geq 1-nk (R/(D\sqrt{d}))^d.$$

Letting $R=2\sigma_{\max}\sqrt{dG}$, with  probability at least $1-nk(2\sigma_{\max}\sqrt{G}/D)^d$,  
$$||\bx_i-\bmu_j||>2\sigma_{\max}\sqrt{dG}, \forall i, j.$$ 
\end{proof}

\begin{lemma}\label{coverage}
Given  n observations from a GMM of  $k$ isotropic Gaussians with true means 
$\bmu_1,\cdots,\bmu_k$, variances $\sigma_{1}^2,\cdots,\sigma_{k}^2$ respectively and $G>1$
 then

\[
P(||\bx_i-\bmu_j||<\sigma_j\sqrt{dG}, \forall  j \in\{1,...,k\}, \forall \bx_i\in S_j)>1-2ne^{-d(G-1)^2/8}.
\]

\end{lemma}

\begin{proof}  From Corollary \ref{cor:probgauss} we have that for $ \forall j\in \{1,\cdot,k\}$ and any one $\bx_i \in S_j$, 
\[
P(||\bx_i-\bmu_j||<\sigma_j\sqrt{dG})>1- 2e^{-d(G-1)^2/8}.
\]
For $\bx_i \in S_j$ denote by $E_{ij}$ the event $E_{ij}:||\bx_i-\bmu_j||< \sqrt{dG} \sigma_{j}$. Then using the union bound we get
$$P(\bigcap_{\bx_i \in S_j}E_{ij})=1-P(\bigcup_{\bx_i \in S_j} \overline E_{ij}
)\ge 1-\sum_{\bx_i \in S_j}P(\overline E_{ij}
)>1-2|S_j|e^{-d(G-1)^2/8}$$

Similarly, since $\sum_{j}|S_j|=n$
\begin{eqnarray*}
P(\bigcap_{}( E_{ij},\forall \bx_i \in S_j, \forall j
))&&=1-P(\bigcup_{}(\overline{E}_{ij})
, \forall \bx_i \in S_j, \forall j)\ge 1-\sum_{\forall \bx_i \in S_j, \forall j
}P(\overline E_{ij})\\
&&>1-2\sum_{j}|S_j|e^{-d(G-1)^2/8}=1-2ne^{-d(G-1)^2/8}.
\end{eqnarray*}
\end{proof}
\begin{lemma}\label{original condition}
Given  $n$ observations from a GMMUB of  $k$ isotropic Gaussians with true means $\bmu_1,\cdots,\bmu_k$, variances $\sigma_{1}^2I_d,\cdots,\sigma_{k}^2I_d$ respectively, and weights $\pi_1,\cdots,\pi_k$ together with the uniform distribution within radius $D\sqrt{d}$, with weight $\pi_{k+1}$. If A1-A3 and C2 are satisfied, then the following two statements hold:

\noindent{\bf Statement 1}:
For any two positive clusters $S_j$ and $S_l$ with true means $\bmu_j$, $\bmu_l$, covariance matrix $\sigma_j^2I_d$, $\sigma_l^2I_d$ respectively, there is no point from $S_j$ at a distance less than $R_{\sigma_{\max}}=\sigma_{\max}\sqrt{dG}$ from $\bmu_l$ and no point from $S_l$ at a distance less than $R_{\sigma_{\max}}$ from $\bmu_j$.

\noindent{\bf Statement 2}:  A $d$-dimensional ball of radius $R_{\sigma_{\max}}=\sigma_{\max}\sqrt{dG}$ centered at any point $\bx\in S_j$ from a cluster $S_j$ will cover all the points of $S_j$.

\end{lemma}
\begin{proof}
 If A1-A3 and C2 hold, 
 for any positive point $\bx_i \in S_j$ and any $l\not =j$ we have :
\[
||\bx_i-\bmu_l||\geq ||\bmu_j-\bmu_l||-||\bx_i-\bmu_j||>2\sqrt{dG}\sigma_{\max}-\sqrt{dG}\sigma_{\max}=\sqrt{dG}\sigma_{\max}
\]

Hence, for any two positive clusters $S_j$, $S_l$ and any one point $\bx_i \in S_j$, then $ ||\bx_i-\bmu_l||>\sigma_{\max}\sqrt{dG}$. A similar result also holds when we select any one point from $S_l$ and measure the distance between $\bmu_j$ and this point. Hence,  Statement 1 holds.

If A1-A3 and C2 hold,  then for any $ \bx_i,\bx\in S_j$, we have:
$||\bx_i-\bmu_j||<\sigma_j \sqrt{dG}$ and $||\bx-\bmu_j||< \sigma_j\sqrt{dG}$, therefore
\[
||\bx-\bx_i||\leq||\bx-\bmu_j||+||\bmu_j-\bx_i||\leq 2\sigma_l\sqrt{dG}<\sigma_{\max}\sqrt{dG}
\]
Hence, Statement 2 holds.
\end{proof}
From condition C1 we come up with the following lemma:

\begin{lemma}\label{separation condition}
Suppose $\bx\sim\N(0, \sigma_1^2I_d)$ and $\bx_j$ is a sample from a uniform distribution such that C1 is satisfied. If  $||\bx||\leq \sigma_1\sqrt{dG}$, then $\ell(\bx-\bx_j,\sigma_{\max})=0$.
\end{lemma}

\begin{proof}
From C1, selecting $\bmu_1=0$, we obtain
\[
||\bx_j||> 2\sigma_{\max}\sqrt{dG}>(\sigma_{\max}+\sigma_1)\sqrt{dG}
\]
\[
||\bx_j-\bx||\geq||\bx_j||-||\bx||> (\sigma_{\max}+\sigma_1)\sqrt{dG}-\sigma_1\sqrt{dG}=\sigma_{\max}\sqrt{dG}
\]
Then,
\[
||\bx_j-\bx||^2> dG\sigma_{\max}^2.
\]

But when $||\bx_j-\bx||^2> dG\sigma_{\max}^2$, we have that $\ell(\bx-\bx_j,\sigma_{\max})=0.$
\end{proof}
The following lemma will show that with high probability the robust loss $L(\bx,\sigma_{max})$ computed at any positive point $\bx$ is smaller than the robust loss computed at any negative point. 
Therefore, by minimizing the robust loss, the algorithm will easily find the positive points since the loss is smaller.


\begin{lemma} \label{lem:loss1neg}
Suppose there are $n$ observations from a mixture of one Gaussian $\N(\bmu_1, \sigma_1^2I_d)$ with mixture weight $\pi_1$ and a uniform distribution inside the sphere of radius $D\sqrt{d}$. C1-C2 are satisfied. Let $\bx_l$ be any positive point and $\bx_j$ any negative point . If $\sigma_{\max}>2\sigma_1$ and 
\[\pi_1>\frac{({\sigma_{\max}}\sqrt{G}/D)^dG/(d/2+1)}{(G-\frac{(1+G)\sigma_1^2}{\sigma_{\max}^2})+({\sigma_{\max}}\sqrt{G}/D)^dG/(d/2+1)},\] 
 then with probability at least $1-2\exp(-nW^2/2G^2)$  we have $L(\bx_l,\sigma_{\max})<L(\bx_j,\sigma_{\max})$, where $$W=\pi_1(G-\frac{(1+G)\sigma_1^2}{\sigma_{\max}^2})-(1-\pi_1)({\sigma_{\max}}\sqrt{G}/D)^d\frac{G}{d/2+1}.$$ 

\end{lemma}
\begin{proof} Denote $B_i, C_i, i=1,...,n$ be Bernoulli indicator variables, 
$$
B_i=
 \begin{cases}
 1& \text{$\bx_i$ is a positive point}\\
 0& \text{else}
 \end{cases}
,$$

$$
C_i=
 \begin{cases}
 1& \text{ if }\|\bx_i-\bx_j\|<R_{\sigma_{\max}}\\
 0& \text{else}
 \end{cases}
,$$
 thus $E[B_i]=\pi_1$ and $E[C_i]\leq ({\sigma_{\max}}\sqrt{G}/D)^d< 1$.
 Furthermore, consider the random variables $P_i=\ell(\bx_i-\bx_l,\sigma_{\max})$, $
Q_i=\ell(\bx_i-\bx_j,\sigma_{\max})$. 
Finally, the random variable  $T_i=B_iP_i-(1-B_i)C_iQ_i$ representing the loss value $\ell(\bx_i-\bx_l,\sigma_{\max})-\ell(\bx_i-\bx_j,\sigma_{\max})$, since due to Lemma \ref{separation condition}, the balls of radius $R_{\sigma_{\max}}$ centered at $\bx_l$ and $\bx_j$ are disjoint.

Denote  $u_l=\|\bx-\bx_l\|^2$, $\bx-\bx_l \sim \N(\bmu-\bx_l, \sigma_1^2I_d)$, therefore $u_l$ follows the non-central chi-square distribution.
$u_l\sim \chi^2(u;d,\|\bmu-\bx_l\|^2)$, where $\|\bmu-\bx_l\|^2<dG\sigma_1^2$.

The the random variable for the total loss difference is \[
T=L(\bx_l,
\sigma_{\max})-L(\bx_j,\sigma_{\max})=\sum_{i=1}^{n}(B_iP_i-(1-B_i)C_iQ_i)=\sum_{i=1}^{n}T_i
\]  Then from Corollary  \ref{cor:meanncchi} 
\[
 E(u_l)=d\sigma_1^2+\|\bmu-\bx_l\|^2
\]
\[
T_i\in [-G,G], \forall i ,
\]
\[
\forall i, E(P_i)<\frac{E(u_l)}{d\sigma_{\max}^2}-G=\frac{d\sigma_1^2+\|\bmu-\bx_l\|^2}{d\sigma_{\max}^2}-G\leq\frac{(1+G)\sigma_1^2}{\sigma_{\max}^2}-G,
\]

From Corollary \ref{cor:pdfuniform}

\[
E(Q_i)=-\frac{G}{d/2+1}, \forall i
\]
Therefore due to the independence of $B_i,C_i,P_i,Q_i$ we have

\begin{eqnarray}
 E(T)&&=nE(B_i)E(P_i)-nE(1-B_i)E(C_i)E(Q_i)\\
&&<n(\pi_1(\frac{(1+G)\sigma_1^2}{\sigma_{\max}^2}-G)+(1-\pi_1)({\sigma_{\max}}\sqrt{G}/D)^d\frac{G}{d/2+1})
\end{eqnarray}

By Hoeffding's inequality:
$$
P(|T-E(T)|>t)<2\exp(-2t^{2}/4nG^2),
$$
 Let $t=-E(T)=n(\pi_1(G-\frac{(1+G)\sigma_1^2}{\sigma_{\max}^2})-(1-\pi_1)({\sigma_{\max}}\sqrt{G}/D)^d\frac{G}{d/2+1})=nW$, if $W>0$
we have 
\[
P(T>0)<2\exp(-nW^2/2G^2)
\]
\[
P(T<0)>1-2\exp(-nW^2/2G^2)
\]

$W>0$ is equivalent to $$\pi_1>\frac{({\sigma_{\max}}\sqrt{G}/D)^d(G/(d/2+1))}{(G-\frac{(1+G)\sigma_1^2}{\sigma_{\max}^2})+({\sigma_{\max}}\sqrt{G}/D)^d(G/(d/2+1))}.$$
 Hence, if  $\pi_1>\frac{({\sigma_{\max}}\sqrt{G}/D)^d(G/(d/2+1))}{(G-\frac{(1+G)\sigma_1^2}{\sigma_{\max}^2})+({\sigma_{\max}}\sqrt{G}/D)^d(G/(d/2+1))}$,  then:
 
$
L(\bx_l,\sigma_{\max})< L(\bx_j,\sigma_{\max})
$  with probability at least $1-2\exp(-nW^2/2G^2).\; $
\end{proof}
\begin{corollary}\label{Loss comparison}
 Under the notation and conditions of Lemma \ref{lem:loss1neg}, if $\bx_i$ is any positive point, then with probability at least $1-2n\exp(-nW^2/2G^2)$  we have that $L(\bx_i,\sigma_{\max})<L(\bx_j,\sigma_{\max})$, for all negative points $\bx_j$.
\end{corollary}
\begin{proof}

Let $K=\{j, \bx_j \text{ is a negative point}\}$. Then $|K|\leq n$. For any $j\in K$ denote $E_{j}:L(\bx_i,\sigma_{\max})-L(\bx_{j},\sigma_{\max})<0$.  Then, from Lemma \ref{lem:loss1neg}, if $||\bx_i||\leq \sqrt{dG}\sigma_1$
$$P(\bigcap_{j\in K}E_{j})=1-P(\bigcup_{j\in K} E_{j}^{c})\ge 1-\sum_{j\in K}P(E_{j}^{c})>1-2n\exp(-nW^2/2G^2)$$
Therefore, $L(\bx_i,\sigma_{\max})<L(\bx_j,\sigma_{\max})$ with probability at least \[
1-2n\exp(-nW^2/2G^2).\]
\end{proof}
\begin{proof}\textbf{of Proposition \ref{prop:oneclust}.}
Based on Corollary \ref{Loss comparison},  all the positive points have smaller cost than that of all negative points with probability at least
 $1-2n\exp(-nW^2/2G^2).$

Since $i=\argmin_{i}L(\bx_{i},\sigma_{\max}),$ then $\bx_i$ is a positive point with probability at least $1-2n\exp(-nW^2/2G^2)$.

If  $\bx_i$ is a positive point, based on C1-C2 and Lemma \ref{original condition},  $C=\{\bx\in S, \|\bx-\bx_i\|<\sigma_{\max}\sqrt{dG}\}$ covers all the positive points without any negative point. 

Hence, the algorithm correctly finds all the positives with probability at least $$1-2n\exp(-nW^2/2G^2).$$
\end{proof}
We generalize Lemma \ref{lem:loss1neg} and Corollary \ref{Loss comparison} to the cases of multiple Gaussians to prove Prop \ref{prop:multiclust}.

\begin{lemma} \label{lem:loss1negmult}
Let $\bx_{i}\in \RR^d, i=1,...,n$ be $n$ observations sampled  from a mixture of $k$  isotropic Gaussians with means $\bmu_1,\cdots,\bmu_k$, covariance  matrices $\sigma_1^2 I_{d},\cdots,\sigma_k^2I_{d}$, weights $\pi_1,\cdots,\pi_k$ and the uniform distribution within a ball of radius $D\sqrt{d}$ centered at the origin, with weight $\pi_{k+1}$, so that  $\pi_1+\cdots\pi_k+\pi_{k+1}=1$. 
Assume that C1 and C2 hold and that
$${\pi_j}>\frac{(G/(d/2+1))({\sigma_{\max}}\sqrt{G}/D)^d}{(G-\frac{(1+G)\sigma_j^2}{\sigma_{\max}^2})+(G/(d/2+1))({\sigma_{\max}}\sqrt{G}/D)^d}, \forall j \in \{1,\cdots k\}.$$ 
Let $\bx_l$ be any positive point, $\bx_l \in S_j$, for a certain $j$  and $\bx_m$  be any negative point,
 then with probability at least $1-2\exp(-nW_j^2/2G^2)$  we have $L(\bx_l,\sigma_{\max})<L(\bx_m,\sigma_{\max})$, where $$W_j=\pi_j(G-\frac{(1+G)\sigma_j^2}{\sigma_{\max}^2})-\pi_{k+1}({\sigma_{\max}}\sqrt{G}/D)^d\frac{G}{d/2+1}.$$ 

\end{lemma}
\begin{proof}
The proof of Lemma \ref{lem:loss1negmult} is similar to proof of Lemma \ref{lem:loss1neg}.

Denote $B_i, C_i, i=1,...,n$ be Bernoulli indicator variables, 
$$
B_i=
 \begin{cases}
 1& \text{$\bx_i$ is a positive point from $S_j$}\\
 0& \text{else}
 \end{cases}
,$$

$$
C_i=
 \begin{cases}
 1& \text{ if }\|\bx_i-\bx_m\|<R_{\sigma_{\max}}\\
 0& \text{else}
 \end{cases}
,$$

$$
E_i=
 \begin{cases}
 1& \text{ $\bx_i$ is a negative point}\\
 0& \text{else}
 \end{cases}
.$$

Hence, $E[B_i]=\pi_j$, $E[E_i]=\pi_{k+1}$, $E[C_i]\leq ({\sigma_{\max}}\sqrt{G}/D)^d< 1$.
Let $P_i=\ell(\bx_i-\bx_l,\sigma_{\max})$, $
Q_i=\ell(\bx_i-\bx_m,\sigma_{\max})$. 
Then $T_i=B_iP_i-E_iC_iQ_i$ represents the loss value $\ell(\bx_i-\bx_l,\sigma_{\max})-\ell(\bx_i-\bx_m,\sigma_{\max})$.

Denote  $u_l=\|\bx-\bx_l\|^2$, $\bx-\bx_l \sim \N(\bmu_j-\bx_l, \sigma_1^2I_d)$, therefore $u_l$ follows the non-central chi-square distribution.
$u_l\sim \chi^2(u;d,\|\bmu-\bx_l\|^2)$, where $\|\bmu_j-\bx_l\|^2<dG\sigma_j^2$.

The the random variable for the total loss difference is \[
T=L(\bx_l,
\sigma_{\max})-L(\bx_m,\sigma_{\max})=\sum_{i=1}^{n}(B_iP_i-E_iC_iQ_i)=\sum_{i=1}^{n}T_i
\] 

Similar to proof of Lemma \ref{lem:loss1neg}, we have that:
\begin{eqnarray}
 E(T)&&<n(\pi_j(\frac{(1+G)\sigma_j^2}{\sigma_{\max}^2}-G)+\pi_{k+1}({\sigma_{\max}}\sqrt{G}/D)^d\frac{G}{d/2+1})
\end{eqnarray}

By Hoeffding's inequality:
$$
P(|T-E(T)|>t)<2\exp(-2t^{2}/4nG^2),
$$
 Let $t=-E(T)=n(\pi_j(G-\frac{(1+G)\sigma_j^2}{\sigma_{\max}^2})-\pi_{k+1}({\sigma_{\max}}\sqrt{G}/D)^d\frac{G}{d/2+1})=nW_j$, if $W_j>0$
we have 
\[
P(T>0)<2\exp(-nW_j^2/2G^2)
\]
\[
P(T<0)>1-2\exp(-nW_j^2/2G^2)
\]

But $W_j>0$ is equivalent to $$\pi_j>\frac{({\sigma_{\max}}\sqrt{G}/D)^d(G/(d/2+1))}{(G-\frac{(1+G)\sigma_j^2}{\sigma_{\max}^2})+({\sigma_{\max}}\sqrt{G}/D)^d(G/(d/2+1))}.$$
 Hence, if  $\pi_j>\frac{({\sigma_{\max}}\sqrt{G}/D)^d(G/(d/2+1))}{(G-\frac{(1+G)\sigma_j^2}{\sigma_{\max}^2})+({\sigma_{\max}}\sqrt{G}/D)^d(G/(d/2+1))}$,  then:
 
$
L(\bx_l,\sigma_{\max})< L(\bx_m,\sigma_{\max})
$  with probability at least $1-2\exp(-nW_j^2/2G^2).\; $
\end{proof}

\begin{corollary}\label{Loss comparisonmult}
 Under the notation and conditions of Lemma \ref{lem:loss1negmult}, if $\bx_l$ is any positive point from a positive cluster, then with probability at least $1-2ne^{-n\min_{j} W_j^2/2G^2}$, we have that $L(\bx_l,\sigma_{\max})<L(\bx_m,\sigma_{\max})$, for all negative points $\bx_m$.
Here, 
\[
 W_j=\pi_j(G-\frac{(1+G)\sigma_j^2}{\sigma_{\max}^2})-\pi_{k+1}({\sigma_{\max}}\sqrt{G}/D)^d\frac{G}{d/2+1}.
\]
\end{corollary}
\begin{proof}

Let $K=\{m, \bx_m \text{ is a negative point}\}$. Then $|K|\leq n$. For any $m \in K$ denote $E_{m}:L(\bx_l,\sigma_{\max})-L(\bx_{m},\sigma_{\max})<0$.  Then, from Lemma \ref{lem:loss1negmult}, suppose $\bx_l \in S_j$, $\forall j$.
$$P(\bigcap_{m\in K}E_{j})=1-P(\bigcup_{m\in K} E_{j}^{c})\ge 1-\sum_{m\in K}P(E_{j}^{c})>1-2n\exp(-n\min_{j} W_j^2/2G^2)$$

Therefore, $L(\bx_l,\sigma_{\max})<L(\bx_m,\sigma_{\max})$ with probability at least $$1-2n\exp(-n\min_{j} W_j^2/2G^2).$$
\end{proof}
\begin{proof}\textbf{of Proposition \ref{prop:multiclust}.}

The proof of Proposition \ref{prop:multiclust} is similar to proof of Proposition \ref{prop:oneclust}.

In CRLM, OCRLM is run $k$ times, each time finding an observation $\bx$ with minimum loss $L(\bx,\sigma_max)$.
For each iteration, based on Corollary \ref{Loss comparisonmult},  all the positive points have smaller cost than that of all negative points with probability at least
 $$1-2n\exp(-n\min_{j} W_j^2/2G^2).$$

Since $i=\argmin_{i}L(\bx_{i},\sigma_{\max})$, then $\bx_i$ is a positive point with probability at least $1-2n\exp^{\frac{-n\min_{j} W_j^2}{2G^2}}$.

If  $\bx_i$ is a true positive point, without loss of generality, suppose $\bx_i \in S_j$. Then based on C1-C2 and Lemma \ref{original condition},  the set $C=\{\bx\in S, \|\bx-\bx_i\|<\sigma_{\max}\sqrt{dG}\}$ covers all the points from $S_j$ without any negative point or  other positive points from other positive clusters. 

Denote denote $E_j$ be the random event $E_{j}$: all the points of $S_j$ are perfectly clustered by CRLM and $E:$ all the points are perfectly clustered.
Therefore, \[
P(E)=(\bigcap_{j=1}^{k}E_{j})=1-P(\bigcup_{j=1}^{k} E_{j}^{c})\ge 1-\sum_{j=1}^{k}P(E_{j}^{c})>1-2nk\exp(-n\min_{j} W_j^2/2G^2)
\]
Hence, CRLM correctly clusters all the points with probability at least $$1-2nk\exp(-n\min_{j} W_j^2/2G^2).$$

\end{proof}

\bibliographystyle{plain}
\bibliography{references}

\end{document}